\documentclass{siamart171218}

\usepackage{graphicx}
\usepackage{amsmath}
\usepackage{graphics}
\usepackage{amsfonts}
\usepackage{amssymb}%
\usepackage{amssymb}
\usepackage{mathrsfs}
\usepackage{amsfonts}
\usepackage{amsmath}
\usepackage{bm}
\usepackage{graphicx}
\usepackage{multirow}
\usepackage{caption2}
\usepackage{float}
\usepackage{booktabs}
\usepackage{algorithm}
\usepackage{algorithmic}
\usepackage{subfigure}

\usepackage{amsmath}
\usepackage{amsfonts}
\usepackage{graphicx}
\usepackage{epstopdf}
\usepackage{comment}
\usepackage{subfigure}
\usepackage{algorithmic}
\ifpdf
  \DeclareGraphicsExtensions{.eps,.pdf,.png,.jpg}
\else
  \DeclareGraphicsExtensions{.eps}
\fi

\newsiamremark{remark}{Remark}
\newsiamremark{hypothesis}{Hypothesis}
\crefname{hypothesis}{Hypothesis}{Hypotheses}
\newsiamthm{claim}{Claim}

\headers{Sketching with Spherical Designs  on Spheres}{S.-B. Lin, D. Wang, and D.-X. Zhou}

\title{Sketching with Spherical Designs  for   Noisy Data Fitting on Spheres\thanks{The corresponding author is Di Wang.
\funding{S. B. Lin was partially
	supported by the  National Key R\&D Program of China (No.2020YFA0713900) and the
	National Natural Science Foundation of China [Grant Nos. 62276209]. D. Wang was partially
	supported by the
	National Natural Science Foundation of China [Grant Nos. 61772374]. D. X. Zhou was partially  supported in part by the Research Grants Council of Hong Kong [Project Nos. CityU 11308020, N\_CityU 102/20, C1013-21GF], Hong Kong Institute for Data Science,  Germany/Hong Kong Joint Research Scheme  [Project No. G-CityU101/20],
  Laboratory for AI-Powered Financial Technologies, and National Science Foundation of China [Project No. 12061160462] when he worked at City University of Hong Kong 
and the first version of the paper was written.}}}

\author{Shao-Bo Lin\thanks{Center for Intelligent Decision-Making and Machine Learning, School of Management, Xi'an Jiaotong University, Xi'an 710049, China
  (\email{sblin1983@gmail.com}).}
\and Di Wang\thanks{Center for Intelligent Decision-Making and Machine Learning, School of Management, Xi'an Jiaotong University, Xi'an 710049, China
  (\email{wang.di@xjtu.edu.cn}).}
\and Ding-Xuan Zhou \thanks{School of Mathematics and Statistics, University of Sydney, Sydney NSW 2006, Australia  (\email{dingxuan.zhou@sydney.edu.au}).}}

\usepackage{amsopn}

\newcommand{\sfgrad}[1][]{ 
	\nabla_{*}
}

\newcommand{\sfcurl}[1][]{ 
	\mathbf{L}
}

\newcommand{\Jw}[1][\alpha,\beta]{ 
w_{#1}
}

\newcommand{\imat}[1][d]{ 
    I
}

\newcommand{\Lpw}[2][\Jw]{ 
\mathbb{L}_{#2}(#1)
}

\newcommand{\InnerLGb}[2][{\Jw[r-\frac{1}{2},r-\frac{1}{2}]}]{ 
\left(#2\right)_{\Lpw[{#1}]{2}}
}

\newcommand{\Diff}[2][t]{ 
\ifthenelse{\equal{#2}{1}}{\frac{\mathrm{d}}{\mathrm{d}#1}}{
\left(\frac{\mathrm{d}}{\mathrm{d}#1}\right)^{#2}}
}

\ifpdf
\hypersetup{
  pdftitle={{ sketching} with Spherical Designs},
  pdfauthor={S.-B. Lin, D. Wang, and D.-X. Zhou}
}
\fi

\begin{document}

\maketitle

\begin{abstract}
This paper proposes a { sketching} strategy based on spherical designs, which is applied to the classical spherical basis function approach for massive spherical data fitting.
We conduct  theoretical analysis and numerical verifications to demonstrate the feasibility of the proposed { sketching}  strategy. From the theoretical side, we prove that { sketching} based on spherical designs can reduce the computational burden of the  spherical basis function approach without sacrificing its approximation capability. In particular, we provide  upper and lower bounds for the proposed { sketching} strategy to fit noisy data on spheres. From the experimental  side, we numerically illustrate the feasibility  of the { sketching} strategy by showing its comparable fitting performance  with the   spherical basis function approach.  
 These interesting findings show  that the proposed  { sketching} strategy is capable of fitting massive and noisy data on spheres.
\end{abstract}




\pagestyle{myheadings}
\thispagestyle{plain}

\section{Introduction}
Spherical data abound in our lives: from geophysics \cite{king2012lower}, quantum chemistry \cite{choi1999rapid}, planetary science \cite{wieczorek1998potential}
and astrophysics \cite{jarosik2011seven} to computer graphic science \cite{tsai2006all}, image processing \cite{mcewen2013sparse} and signal recovery \cite{mcewen2011novel}.  For example, thousands of directional data are collected  for image rendering \cite{tsai2006all};  millions of cosmic microwave background (CMB) observations  \cite{jarosik2011seven} are gained to analyze the evolution of the Universe;
billions of Gravity Recovery and Climate Experiment (GRACE) data \cite{king2012lower} are sampled to study gravity and Earth's natural systems.   { Different from classical machine learning problems \cite{gyorfi2002distribution} that assume data to be drawn randomly according to some unknown   probability distribution,    spherical data fitting frequently requires delicate sampling mechanisms to generate deterministic samples.}




Sampling according to sampling theorems \cite{mcewen2011novel} and   quadrature rules \cite{brauchart2007numerical} are two popular sampling mechanisms { to generate spherical data}. { Since sampling theorems  are only valid under some sparseness assumptions,    spherical quadrature rules   have received considerable attention  and can generate most of the existing scattered in the literature \cite{mhaskar2001spherical,brown2005approximation,leopardi2006partition,le2008localized}}. 
The problem is, however, that  the    quadrature weights depend  heavily on the location of the  corresponding quadrature points, resulting in the existence of inactive samples for which the weights are small.
Spherical designs are  special spherical quadrature rules with equal weights.
The equal-weight nature implies the equal-contribution of the quadrature points in spherical designs, which excludes  quadrature points associated with small quadrature weights and then reduces the number of sampling  {points} \cite{bondarenko2013optimal}. Due to this,
spherical designs  have been widely used to construct spherical hyper-interpolations  \cite{lin2021distributed} to learn noisy data, produce spherical regularized least squares estimators \cite{an2012regularized} to fit scattered data, and design non-convex minimization algorithms \cite{chen2018spherical} for signal recovery on spheres.



In this paper, we consider problems of noisy spherical data fitting   sampled  at spherical $t$-designs. 
Let $D:=\{(x_{i},y_{i})\}_{i=1}^{|D|}$ be the set of data  
with $\{x_i\}_{i=1}^{|D|}$   a spherical $t$-design 
and  
 \begin{equation}\label{Model1:fixed}
        y_{i}=f^*(x_{i})+\varepsilon_{i},  \qquad\forall\
        i=1,\dots,|D|,
\end{equation}
where $|D|$ denotes the cardinality of $D$,   {  the noise terms  $\{\varepsilon_i\}_{i=1}^{|D|}$ are a set of independent  and identical random variables    satisfying $E[\varepsilon_{i}]=0$ and $|\varepsilon_{i}|\leq M$ for some $M>0$}, and $f^*$ is a function to model the relation between the input $x_i$ and output $y_i$.  
We are interested
in developing an effective algorithm to find an approximation of $f^*$ based on the given data $D$. { It should be mentioned that our approach is suitable for applications that the user  can determine the sampling mechanism on the sphere since the data inputs are assumed to be spherical designs in \eqref{Model1:fixed}}.

If $\varepsilon_i=0$, $i=1,\dots,|D|$, finding an approximation of $f^*$ in model (\ref{Model1:fixed}) is the classical interpolation problem, which can be successfully settled by     spherical polynomials \cite{womersley2001good}, splines \cite{wahba1981spline} and
spherical basis functions (SBF) \cite{narcowich2002scattered}.   In particular, it can be found in \cite{narcowich2002scattered} that the approximation error of the SBF
interpolant can be estimated  in terms of the mesh norm, provided $f^*$ is smooth.  { If $\varepsilon_i\neq 0$, the frequently large condition number} of the SBF-based interpolation matrix \cite{narcowich1998stability} makes the SBF interpolant sensitive to the noise. In this case, a regularization term is needed to guarantee the well-conditionedness  of the interpolation matrix, just as \cite{le2006continuous,hesse2017radial} did, 
{the analysis of the fitting  performance of the SBF approach with regularization is carried out when the noise is extremely small}. { Since the kernel matrix of the classical SBF approach is positive definite and  usually full, which is different from the sparse matrix  generated for Wendland SBF and corresponding multi-level schemes \cite{le2010multiscale}, it requires  $\mathcal O(|D|^2)$ and $\mathcal O(|D|^3)$ complexities for storage and computation to solve the corresponding regularized least squares problem for any fixed regularization parameter  \cite{hesse2017radial}.
These above two phenomena make  the SBF approach difficult to apply when the noise is not small  and the size of the data is large.}

In our previous work \cite{lin2021distributed,Feng2021radial}, we proved that the widely used  distributed learning  equipped with a divide-and-conquer strategy in machine learning \cite{zhang2015divide}
is feasible for noisy and massive spherical data fitting. It should be mentioned that the distributed learning approach proposed in \cite{lin2021distributed,Feng2021radial} requires multiple computational resources, which may dampen the users' spirits and force them to turn to other scalable and stable fitting algorithms.
Our purpose in this paper  is to propose a { sketching} scheme to   reduce the computational burden of the SBF approach  
while maintaining the approximation accuracy.
Different from the  classical { sketching} approach in \cite{rudi2015less,lu2019analysis} that randomly selects part of columns of the interpolation matrix, we focus on generating another spherical $s^*$-design with $s^*\leq t$ to be the centers of SBF.
We provide both theoretical analysis and numerical verifications for the proposed { sketching} strategy.  

 {Our contributions can be stated as follows.} First, we present upper and lower error estimates for the proposed { sketching} strategy to show that it succeeds in reducing the computational burden of the SBF approach without sacrificing its fitting accuracy.
Second, we present Sobolev-type error estimates for the proposed { sketching} strategy, which are different from \cite{hesse2017radial,Feng2021radial} where the analysis is carried out in the $L^2$ space. Third, we employ a novel integral operator approach in our proofs, which avoids the well known ``native space barrier'' \cite{narcowich2007direct} that requires $f^*$ to belong to the native space of the SBF. Finally,  we conduct two toy simulations to illustrate that the proposed { sketching} strategy is more suitable than the classical schemes  in \cite{rudi2015less,lu2019analysis} for noisy data fitting on spheres. All these  demonstrate the power of the proposed { sketching} approach and show its efficiency in fitting massive and noisy spherical data.

The rest of the paper is organized as follows. In the next section, we present a novel { sketching} strategy based on spherical designs. In Section \ref{Sec.theory}, we analyze the theoretical behaviors of the proposed  { sketching} strategy, whose proofs are postponed to Section \ref{Sec.Proof}.
In Section \ref{Sec.Numerical}, we conduct two numerical simulations to verify our assertions.

\section{{ Sketching} with Spherical Designs}\label{Sec.Algorithm}
In this section, after introducing some basic properties
of spherical designs, we propose the { sketching}
strategy based on spherical designs. 
\subsection{Spherical designs} 
Let   $\mathbb S^d$ be the unit sphere embedded in the $(d+1)$-dimensional Euclidean space $\mathbb R^ {d+1}$.
  For $t\in\mathbb N$, denote by $\Pi_t^{d}$
the class  of all algebraic polynomials of degree at most $t$  defined on $\mathbb S^d$.
A spherical $t$-design, denoted by $\mathcal T_t:=\{x_i\}_{i=1}^{|\mathcal T_t|}$, is a finite subset of $\mathbb S^d$ satisfying
\begin{equation}\label{Spherical-design}
     \frac1{|\mathcal T_t|}\sum_{i=1}^{|\mathcal T_t|} { \pi}(x_i)=\frac1{\Omega_d}\int_{\mathbb S^d}{ \pi}(x)d\omega(x),\qquad\forall {\pi}\in\Pi_t^d,
\end{equation}
where $\Omega_d=\frac{2\pi^{\frac{d+1}{2}}}{\Gamma(\frac{d+1}{2})}$ is the volume of $\mathbb S^d$,  and $d\omega$ denotes the Lebesgue measure on the sphere. Spherical $t$-designs were introduced in  \cite{Delsarte1977} which also provided a lower bound on the number of points, i.e. $|\mathcal T_t|\geq c't^d$ for some $c'>0$.  For a given $t$, it is meaningless to present an upper bound of $|\mathcal T_t|$  since the union of two spherical $t$-designs is also a spherical $t$-design. Instead,  {many studies \cite{yudin1995coverings,bannai2009survey,bondarenko2013optimal,bondarenko2015well,brauchart2015distributing,womersley2018efficient} have been done to establish lower bounds}
of $|\mathcal T_t|$ and { find $c_0$-tight} spherical $t$-designs defined as follows.

\begin{definition}\label{Def:design}
A spherical $t$-design is said to be $c_0$-tight for a positive constant $c_0$ { depending only on $d$} if $|\mathcal T_t|=c_0t^d$.
\end{definition}
 
{ It should be mentioned that the above definition only requires the existence of spherical $t$-designs with the optimal order of number of points rather than the Delsarte-Goethals-Seidel lower bound $ct^d$ with an absolute constant $c$ \cite{bannai2009survey}. In fact, the Delsarte-Goethals-Seidel lower bound of points of $t$-designs is known to exist for most $t$ and $d$.} 
The  { $c_0$-tight    spherical $t$-designs  in Definition \ref{Def:design} require the achievability of the lower bound $|\mathcal T_t|\geq c't^d$ with $c'$ depending on $d$,} which has been verified in the seminal paper \cite{bondarenko2013optimal} as the following lemma.

\begin{lemma}\label{Lemma:Spherica-d-1}
 There exists a constant $c_d>0$, depending only on $d$, such that for every $N\geq c_dt^d$ and $t\geq 1$, there exists an $N$-point spherical $t$-design on $\mathbb S^d$.
\end{lemma}

The above lemma 
demonstrates the existence  of $c_0$-tight spherical $t$-designs for any $t\in\mathbb N$ and   shows that one can design a spherical $t$-design  $\mathcal T_t$ satisfying  $|\mathcal T_t|\sim t^d$.  In this way, the classical covering result established 
in \cite{yudin1995coverings} showed that $c_0$-tight spherical $t$-designs have a covering radius (or mesh norm) 
$ 
     h_{_{\mathcal T_t}}:= 
                \max_{x\in\mathbb S^d}\min_{x_i\in\mathcal T_t}\arccos(x\cdot x_i)\leq c_1|\mathcal T_t|^{-\frac1d}
$
for some $c_1$ depending only on $d$ and $c_0$, implying that these designs cannot have large holes.  However, as the union of two spherical $t$-designs is a spherical design, $c_0$-tight spherical $t$-designs may have arbitrarily poor separation, i.e., the separation radius $
   q_{_{\mathcal T_t}}:=\frac12\min_{i\neq i'}\arccos(x_i\cdot x_{i'})
$
can be arbitrarily small. The following lemma     derived  in \cite{bondarenko2015well} shows  the existence of  well separated $c_0$-tight spherical $t$-designs.

\begin{lemma}\label{Lemma:Spherica-d-2}
For $t\in\mathbb N$ and $d\geq 2$, there exist   $c_0$-tight spherical $t$-designs satisfying 
  $q_{_{\mathcal T_t}}\geq c_1|\mathcal T_t|^{-1/d}$,
where $c_0$ and $c_1$ are constants depending only on $d$.
\end{lemma}

Lemma \ref{Lemma:Spherica-d-2} shows that there exist
$c_0$-tight spherical $t$-designs whose points are almost evenly distributed on the sphere, i.e.,  $q_{_{\mathcal T_t}}\sim  h_{_{\mathcal T_t}}\sim |\mathcal  T_t|^{-1/d}$.
 Throughout the paper,   $a\sim b$ for $a,b\in\mathbb R_+$ means that there exists a    constant  $\hat{c}\geq1$  depending only on $d$ such that $\hat{c}^{-1}  a\leq b\leq \hat{c} a$. Compared with  theoretical studies on the existence of $c_0$-tight spherical designs, practitioners are also interested in computing lists of spherical $t$-designs \cite{hardin1993new,chen2006existence,graf2011computation,womersley2018efficient}, most of which were investigated on $\mathbb S^2$. In particular, a  list of near  $1$-tight spherical $t$-designs for $t\leq 21$, $t\leq 100$, and $t\leq 325$ was provided  in \cite{hardin1993new,chen2006existence,womersley2018efficient}, respectively. We refer interested readers to \cite[Sec.2.5]{brauchart2015distributing} for more information on the constructions of $c_0$-tight spherical $t$-designs. Table \ref{RelationNt} exhibits the detailed  $|\mathcal T_{t}|=:N$ for the spherical $t$-design developed in \cite{womersley2018efficient}.
Furthermore, Figure \ref{S2points}  presents a geometrical distribution of different spherical $t$-designs listed in 
  \cite{womersley2018efficient} and shows that the constructed spherical $t$-design is  almost equally spaced on $\mathbb S^2$. 
 
\begin{table}[t]
\centering
\caption{The number of points for some symmetric spherical $t$-designs.}\label{RelationNt}
 \begin{tabular}{|| c || c | c | c | c | c | c | c | c | c ||} 
 \toprule[1.5pt]
 Degree $t$ & 1 & 5 & 9 & 13 & 17 & 21 & 25 & 29 & 33\\
 \hline
\#points $N$ & 2 & 12 & 48 & 94 & 156 & 234 & 328 & 438 & 564 \\ 
  \hline\hline
 Degree $t$ & 39 & 45 & 51 & 57 & 63 & 69 & 75 & 81 & 87 \\
 \hline
\#points $N$ & 782  & 1038 & 1328 & 1656 & 2018 & 2418 & 2852 & 3324 & 3830\\ 
  \hline\hline
 Degree $t$ & 93 & 99 & 105 & 111 & 117 & 123  & 129 &  135 & 141 \\ 
 \hline
\#points $N$ & 4374  & 4952 & 5568 & 6218 & 6906 & 7628 & 8388 & 9182 & 10014  \\ 
 \toprule[1.5pt]
 \end{tabular}
\end{table}

\begin{figure*}[t]
    \centering
    \subfigcapskip=-2pt
    \subfigure[Spherical $t$-designs]{\includegraphics[width=5cm,height=5cm]{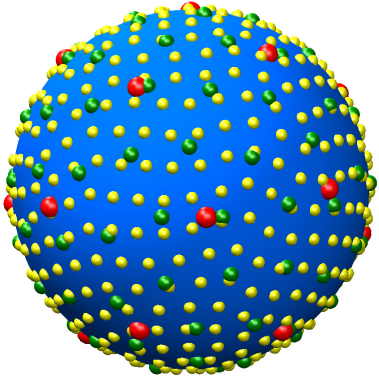}}\hspace{0.5in}
    \subfigure[Symmetric spherical $t$-designs]{\includegraphics[width=5cm,height=5cm]{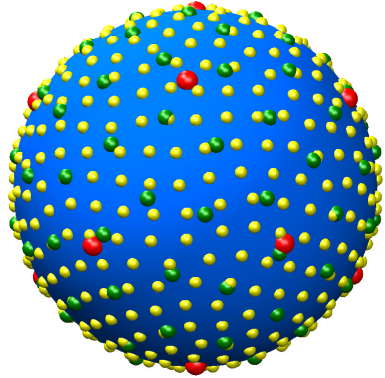}}
    \caption{Spherical $t$-designs with $t=5$, $13$, and $29$ denoted as red, green, and yellow points, respectively. }\label{S2points}
\end{figure*}

%
%
%
%
%
%

\subsection{{ Sketching} with spherical designs}
We adopt the widely used SBF approach \cite{hesse2017radial} on data $D$  whose inputs are selected from the 
$c_0$-tight spherical $t$-design and outputs satisfying \eqref{Model1:fixed}.   Our goal is to 
develop a novel { sketching} strategy for fitting massive and noisy spherical data. 
 For $k\in\mathbb N$, 
let $P_k^{d+1}$ be the normalized 
generalized-Legendre  polynomial,
 i.e., $P_k^{d+1}(1)=1$ and
$$
       \int_{-1}^1P_k^{d+1}(u)P^{d+1}_j(u)(1-u^2)^{\frac{d-2}2}du
       =\frac{\Omega_{d}}{\Omega_{d-1}Z(d,k)}\delta_{k,j},
$$
where $\delta_{k,j}$ is the usual Kronecker symbol and  $Z(d,k)\sim k^{d-1}$ denotes the dimension of $\mathbb{H}^{d}_k$, the space of spherical harmonics \cite{muller2006spherical} of degree $k$. 

We say that a univariate  function $\phi\in L^2[-1,1]$ is an SBF \cite{narcowich2007direct}, if
its Legendre-expansion
$
           \phi(u)=\sum_{k=0}^\infty
            \hat{\phi}_k\frac{Z(d,k)}{\Omega_d} P_k^{d+1}(u)
$ satisfies 
$$
     \hat{\phi}_k:= \Omega_{d-1} \int_{-1}^1P_k^{d+1}(u)\phi(u)(1-u^2)^{\frac{d-2}2}du>0.
$$
{ In addition, if $\sum_{k=0}^\infty\hat{\phi}_k\frac{Z(d,k)}{\Omega_d}<\infty$, then $\phi$ is said to be a positive definite function.}
It is well known that each SBF $\phi$ corresponds to a  native space  
$$
               \mathcal N_\phi:=\left\{f(x)=\sum_{k=0}^\infty\sum_{j=1}^{Z(d,k)}\hat{f}_{k,j}Y_{k,j}(x): \|f\|_\phi^2:=
               \sum_{k=0}^\infty
                  \hat{\phi}_k^{-1}\sum_{j=1}^{Z(d,k)}\hat{f}_{k,j}^2<\infty\right\},
$$
endowed with the inner product $
             \left\langle f,g\right\rangle_{\phi}\hspace{-0.02in}:=\hspace{-0.02in}\sum_{k=0}^\infty
              \hspace{-0.03in}\hat{\phi}_k^{-1}\hspace{-0.03in}\sum_{j=1}^{Z(d,k)}\hspace{-0.03in}\hat{f}_{k,j}\hat{g}_{k,j}
$, where $\{Y_{k,j}\}_{j=1}^{Z(d,k)}$ is an arbitrary orthonormal basis of
$\mathbb H_k^d$ and
$
                 \hat{f}_{k,j}:=\int_{\mathbb
                 S^d}f(x)Y_{k,j}(x)d\omega(x)
$
is the Fourier coefficient of $f$. Without loss of generality, we assume 
\begin{equation}\label{assumption-on-phi}
    0<\hat{\phi}_k\leq 1,\qquad k=0,1,\dots
\end{equation}
  throughout the paper. 
If $\phi$ is a positive definite function, then $\mathcal N_\phi$ is a reproducing kernel Hilbert space with
the
 reproducing kernel
 $ K(x, x')= \phi (x \cdot x')$.

Regularized least squares based on a positive definite function $\phi$ is a popular scheme to approximate $f^*$ in the noisy data model \eqref{Model1:fixed}.
Given
a regularization parameter $\lambda\geq0$, regularized least squares is defined by
\begin{equation}\label{KRR}
    f_{D,\lambda} =\arg\min_{f\in \mathcal N_\phi}
    \left\{\frac{1}{|D|}\sum_{i=1}^{|D|}(f(x_i)-y_i)^2+\lambda\|f\|^2_{\phi}\right\}.
\end{equation}
As an inversion of a $|D|\times |D|$ matrix is involved, 
the storage and computational complexities for solving (\ref{KRR}) are  $\mathcal{O}(|D|^2)$ and $\mathcal{O}(|D|^3)$, respectively.

For $s^*\leq t$, let $\mathcal T_{s^*}=\{x_j^*\}_{j=1}^{m}$ be a { $c_0$-tight spherical $s^*$-design} with $m=|\mathcal T_{s^*}|=c_0(s^*)^d$.
Define
\begin{equation}\label{hypothesis-space-sub}
      \mathcal H_{\mathcal T_{s^*},\phi}:=\left\{\sum_{j=1}^ma_j\phi_{{x}_j^*}:a_j\in\mathbb R\right\},
\end{equation}
where  $\phi_x(\cdot):=\phi(x,\cdot)$.  { Sketching} with spherical  $s^*$-design  is   defined  by
\begin{equation}\label{Nystrom}
       f_{D,s^*,\lambda}:=\arg\min_{f\in\mathcal
       H_{\mathcal T_{s^*},\phi}}\frac1{|D|}\sum_{i=1}^{|D|}(f(x_i)-y_i)^2+\lambda\|f\|_\phi^2.
\end{equation}
Direct computation \cite{rudi2015less} yields
\begin{equation}\label{Analytic-solution}
   f_{D,s^*,\lambda}(\cdot)=\sum_{j=1}^m \alpha_j\phi_{ {x}_j^*}(\cdot),
\end{equation}
where
\begin{equation}\label{Analytic-solution1}
   (\alpha_1,\dots,\alpha_m)^T= \left(\mathbb K_{|D|,m}^T\mathbb K_{|D|,m}+\lambda|D|\mathbb K_{m,m}\right)^\dagger \mathbb K_{|D|,m}^Ty_D,
\end{equation}
$\mathbb A^\dagger$ and $\mathbb A^T$ respectively denote the Moore-Penrose pseudo-inverse and transpose  of a matrix $\mathbb A$,  $(\mathbb K_{|D|,m})_{i,j}=\phi(x_i\cdot {x}^*_j)$, $(\mathbb K_{m,m})_{j,j'}=\phi( {x}^*_j\cdot{x}^*_{j'})$, and $y_D=(y_1,\dots,y_{|D|})^T$. Therefore, to solve \eqref{Nystrom}, it only requires $\mathcal O(|D|m)$ and $\mathcal O(|D|m^2)$ complexities in storage and computation, respectively, which significantly reduces the storage and computational complexities for solving the regularized least squares (\ref{KRR}), especially when $m$ is much less than $|D|$. 
The aim of this paper is to show that such a reduction { in computation}  does not sacrifice the fitting performance of regularized least squares. 

In the classical scattered data fitting community, the separation radius is crucial in reflecting the stability of kernel interpolation \cite{narcowich1998stability} while the mesh norm is key in the quality of approximation \cite{narcowich2002scattered,narcowich2007direct}. Thus, a quasi-uniform point set for which the ratio of mesh norm to separation radius is an absolute constant is  preferable for analysis. To finalize this section, we present our reasons why we focus on { $c_0$-tight spherical $t$-designs} rather than quasi-uniform points on the sphere. On one hand, since there is a regularization parameter $\lambda$  involved in our algorithm \eqref{Nystrom}, it is not necessary to assume the separation radius to be comparable with the mesh norm. In particular, our theoretical analysis only requires the equal-weight quadrature rule and a tight number of quadrature points,    which is satisfied by  the { $c_0$-tight spherical $t$-designs.} On the other hand, if  the quasi-uniform points set is adopted, then the classical { sketching} algorithm \eqref{Nystrom} \cite{rudi2015less,lu2019analysis}     should be modified to accommodate the different-weights  quadrature rules \cite{brown2005approximation} 
to guarantee the theoretically optimal approximation error.  It would be interesting to design a similar { sketching} strategy as \eqref{Nystrom} for quasi-uniform point sets and derive the same optimal approximation error estimate as the derived results for spherical designs in the next section.

\section{Theoretical Behaviors}\label{Sec.theory}
There are totally three SBFs involved in our theoretical  analysis. The first one is the positive
definite function $\phi$, which  is used  to produce the estimator in (\ref{Nystrom}). The second one is 
 $\psi(\cdot)=\sum_{k=0}^\infty
            \hat{\psi}_k\frac{Z(d,k)}{\Omega_d} P_k^{d+1}(\cdot)$   satisfying
\begin{equation}\label{kernel-relation}
     \hat{\psi}_k=\hat{\phi}_k^{\beta}, \qquad  0\leq \beta\leq 1,
\end{equation}
and it corresponds to a native space $\mathcal N_\psi$ where our analysis is carried out. 
Due to (\ref{assumption-on-phi}), (\ref{kernel-relation}) yields $\hat{\phi}_k \leq \hat{\psi}_k $  and consequently $\mathcal N_\phi\subseteq \mathcal N_\psi$. It should be mentioned that
$\beta=0$ in  (\ref{kernel-relation}) implies $\mathcal N_\psi=L^2(\mathbb S^d)$, while $\beta=1$ yields $\mathcal N_\psi=\mathcal N_\phi$. The last one is  $\varphi(\cdot)=\sum_{k=0}^\infty
            \hat{\varphi}_k\frac{Z(d,k)}{\Omega_d} P_k^{d+1}(\cdot)$   satisfying 
\begin{equation}\label{kernel-relation-1}
     \hat{\varphi}_k=\hat{\phi}_k^{\alpha}, \qquad  \alpha\geq \beta,
\end{equation}
and it is used to quantify the regularity of $f^*$ in terms of $f^*\in \mathcal N_\varphi$.

If $\alpha\geq 1$, then (\ref{kernel-relation-1}) together with (\ref{assumption-on-phi})
implies  $\mathcal N_\varphi\subseteq \mathcal N_\phi$, showing that $f^*$ is in the native space of $\mathcal N_\phi$. 
For $\beta\leq  \alpha< 1$, we have $\mathcal N_\phi\subset  \mathcal N_\varphi\subseteq \mathcal N_\psi$. 
We refer to the former as the in-native-space setting, while the latter as the out-of-native-space setting.

It should be mentioned that the analysis of the fitting performances of SBF interpolation or approximation is  different \cite{narcowich2007direct} for different settings. In particular, the fitting  error   for algorithm (\ref{KRR}) with $\lambda\geq0$ has been established in \cite{narcowich2002scattered,le2006continuous,hesse2017radial}  for  the in-native-space setting.  
However, for the out-of-native-spacing setting, it  usually requires more technical skills to conquer the native-space barrier, just as \cite{narcowich2007direct} did via a Bernstein-type inequality for SBF.
In this paper, we show that the analysis can be unified by using a recently developed integral operator approach on the sphere \cite{Feng2021radial}. 
 The following theorem presents an upper bound estimate for the { sketching} strategy (\ref{Nystrom}).
 
 \begin{theorem}\label{Theorem:native-out}
Let $0<\delta<1$, $D$ be the data set satisfying \eqref{Model1:fixed}  with $\{\varepsilon_i\}_{i=1}^{|D|}$ being i.i.d. zero-mean random variables satisfying $|\varepsilon|\leq M$,
$\mathcal T_{s^*}$ be a $c_0$-tight spherical $s^*$-design, $\phi$ be a positive
definite function, and $\psi$ and $\varphi$ be SBFs satisfying (\ref{kernel-relation}) and (\ref{kernel-relation-1}) with
  $0\leq \beta\leq 1$  and  $0\leq\alpha-\beta\leq 1$.
If $\hat{\phi}_k\sim k^{-2\gamma}$ with $\gamma>d/2$, $f^*\in\mathcal N_{\varphi}$,  $\lambda\sim|D|^{-\frac{2\gamma}{2\gamma(\alpha-\beta)+d}}$,
and
\begin{equation}\label{restriction-on-s*}
      c|D|^{\frac{2}{2\gamma(\alpha-\beta)+d}}\leq s^*\leq t,
\end{equation}
 then with confidence $1-\delta$, there holds
\begin{equation}\label{approxiamtion-error}
    \|J_{\phi,\psi}f_{D,s^*,\lambda}-f^*\|_\psi \leq C|D|^{-\frac{\gamma(\alpha-\beta)}{2\gamma(\alpha-\beta)+d}}\log\frac3\delta,
\end{equation}
where  $J_{\phi,\psi}:\mathcal N_\phi\rightarrow \mathcal N_\psi$ is the Canonical inclusion, and
$c$ and $C$ are constants independent of $|D|$, $t$, $\delta$, $\lambda$ and $s^*$.
\end{theorem}

Instead of assuming $f^*\in\mathcal N_\phi$ \cite{hesse2017radial,Feng2021radial},  our analysis in Theorem \ref{Theorem:native-out}  is carried out  for $f^*\in\mathcal N_\varphi$ with $0\leq \alpha-\beta\leq 1$. Such an analysis avoids the classical native-space barrier \cite{narcowich2007direct} via showing the same format of error bounds in (\ref{approxiamtion-error}) for all $\alpha$ satisfying $0\leq \alpha-\beta\leq 1$.  
In particular, if $\beta=0$ and $0<\alpha<1$, then (\ref{approxiamtion-error})  reduces to  
\begin{equation}\label{Appro-error-native-out}
    \|J_{\phi,\psi}f_{D,s^*,\lambda}-f^*\|_{L^2(\mathbb S^d)} \leq C|D|^{-\frac{\alpha\gamma}{2\alpha\gamma+d}}\log\frac3\delta,
\end{equation}
showing that the proposed { sketching} strategy  (\ref{Nystrom}) is also effective  even when $f^*\notin \mathcal N_\phi$. If we set $\beta=1$ and $1<\alpha\leq 2$, then (\ref{approxiamtion-error}) reduces to 
\begin{equation}\label{Appro-error-native-in}
    \|f_{D,s^*,\lambda}-f^*\|_{\phi} \leq C|D|^{-\frac{(\alpha-1)\gamma}{2(\alpha-1)\gamma+d}}\log\frac3\delta,
\end{equation}
presenting   Sobolev-type error estimates for noisy data fitting. Since we only impose   $0\leq \beta\leq 1$ and $0\leq \alpha-\beta\leq 1$ in our analysis,   Theorem \ref{Theorem:native-out} covers numerous $(\psi,\varphi)$-pairs adapting to different error analysis frameworks \cite{narcowich2002scattered,levesley2005approximation,le2006continuous,narcowich2007direct,hangelbroek2010kernel,hangelbroek2011kernel,hesse2017radial,zhou2018deep,gao2020multivariate,lin2021distributed,an2021lasso}.

If we set $s^*=t$, then the { sketching} strategy (\ref{Nystrom}) becomes the classical regularized least squares \eqref{KRR}. Theorem \ref{Theorem:native-out} shows that with confidence $1-\delta$, there holds
\begin{equation}\label{approxiamtion-error-RLS}
    \|J_{\phi,\psi}f_{D,\lambda}-f^*\|_\psi \leq C|D|^{-\frac{\gamma(\alpha-\beta)}{2\gamma(\alpha-\beta)+d}}\log\frac3\delta.
\end{equation} 
Compared with  \cite{Feng2021radial}, the derived error estimate is novel in the sense that we provide a larger range of $\alpha$ and $\beta$  than $\alpha=1$ and $\beta=0$. This is  non-trivial since our result is available in the out-of-native-space setting.   Noting (\ref{restriction-on-s*}) and $m=|\mathcal T_{s^*}|=c_0(s^*)^d$, we have $m\geq c_1|D|^{\frac{2d}{2\gamma(\alpha-\beta)+d}}$ for $c_1=c_0c^d$. This implies that the size of { sketching} depends on $\gamma$, $\alpha$, and $\beta$. If $\alpha-\beta=1$, then $\gamma>d/2$ naturally yields a computation-reduction of the { sketching} strategy. In general, $\gamma$ should satisfy $\gamma>d/(2(\alpha-\beta))$.

 Finally, comparing the derived estimate (\ref{approxiamtion-error-RLS})   with the results in \cite{hesse2017radial} with $\alpha=1$ and $\beta=0$, the approximation rate of an order $|D|^{-\frac{\gamma}{2\gamma+d}}$ is 
worse than $|D|^{-\frac{\gamma}{d}}$. This seems that the computation-reduction is   built upon a sacrifice of fitting performance at first glance. In the following theorem, we show that the derived error bounds in Theorem \ref{Theorem:native-out} cannot be essentially improved.

\begin{theorem}\label{Theorem:lower-bound}
 Let $\mathcal M$ be  the set of  zero-mean distributions with uniform bound $M$,   $D$ be the data set satisfying \eqref{Model1:fixed}  with $\{\varepsilon_i\}_{i=1}^{|D|}$ drawn i.i.d. according to a distribution in $\mathcal M$, $f_D$ be an arbitrary function derived from $D$, $\phi$, $\psi$, and $\varphi$ be SBFs satisfying (\ref{kernel-relation}) and (\ref{kernel-relation-1}) with $0\leq \beta\leq 1$ and $\alpha-\beta\geq 0$ {  
and $\hat{\phi}_k \tilde k^{-2\gamma}$ with $\gamma >d/2$. Then there exists a $\rho^*\in \mathcal M$ and an $f^*_{bad}\in\mathcal N_{\varphi}$  with $\|f_{bad}^*\|_\varphi\leq U$ for some     $U>0$ such that
\begin{eqnarray*}
   \mathbf  P_{\rho^*}\left[\|f_D-f^*_{bad}\|_{\psi}\geq C_1|D|^{-\frac{\gamma(\alpha-\beta)}{2\gamma(\alpha-\beta)+d}}\right]
    \geq
    \frac{1}{4},
\end{eqnarray*}
where $\mathbf P_{\rho^*}$ denotes the probability with respect to the distribution $\rho^*$, and $C_1$ is a constant  independent of $|D|$, $\delta$ and $t$}.
\end{theorem}

Theorem \ref{Theorem:lower-bound} shows that the fitting error derived  in \eqref{approxiamtion-error} is optimal in a probability sense.
Combining Theorem \ref{Theorem:native-out} with Theorem \ref{Theorem:lower-bound}, 
we obtain that the proposed { sketching} strategy (\ref{Nystrom}) maintains the approximation capability of regularized least squares (\ref{KRR})
while significantly reducing the computational burden.
This makes (\ref{Nystrom}) available for massive and noisy data fitting problems on spheres. Furthermore, 
the reason for the degradation in fitting performance
is the large amount of noise in the model (\ref{Model1:fixed})
when compared with \cite{hesse2017radial}.
In fact, the magnitude of noise  $M$ in (\ref{Model1:fixed}) can be comparable with $\|f^*\|_\varphi$, which is far beyond the scope of \cite{hesse2017radial}. As shown in our proof, the accommodation of large noise leaves a large room for   { sketching}, showing that larger noise admits smaller $s^*$ for { sketching}  with spherical $s^*$-designs. The restriction of $s^*$ in (\ref{restriction-on-s*}) is presented for the worst-case analysis since we do not give any lower bound on the magnitude of noise.

\section{Numerical Verifications}\label{Sec.Numerical}

In this section, some numerical results are reported to verify our theoretical statements. Three { sketching} methods are employed for comparisons. The first method chooses the first $m$ samples from the training set for the { sketching} set (denoted by First). The second method randomly chooses $m$ samples from the training set for the { sketching} set (denoted by Random). These two methods have been widely used for kernel learning with random samples \cite{rudi2015less,lu2019analysis}  in Euclidean space  and provide  baselines for our analysis.
The third method uses Womersley's symmetric spherical $s^*$-designs with $m=(s^*)^2/2+s^*/2+O(1)$ points \cite{womersley2018efficient}
\footnote{https://web.maths.unsw.edu.au/\%7Ersw/Sphere/EffSphDes/} for { sketching} set (denoted by $s^*$-designs). 

We start by introducing two testing functions. The first function is the Franke function modified by Renka \cite[p. 146]{renka1988multivariate},
\begin{eqnarray}
f_1(x) &=& 0.75\exp(-(9x^{(1)}-2)^2/4-(9x^{(2)}-2)^2/4-(9x^{(3)}-2)^2/4) \nonumber\\
            &+&   0.75\exp(-(9x^{(1)}+1)^2/49-(9x^{(2)}+1)/10-(9x^{(3)}+1)/10) \nonumber\\
            &+&   0.5\exp(-(9x^{(1)}-7)^2/4-(9x^{(2)}-3)^2/4-(9x^{(3)}-5)^2/4) \nonumber\\
            &-&   0.2\exp(-(9x^{(1)}-4)^2-(9x^{(2)}-7)^2-(9x^{(3)}-5)^2),
\end{eqnarray}
where ${x}=(x^{(1)}, x^{(2)}, x^{(3)})^T$.
The second function is constructed via the well known Wendland function \cite{chernih2014wendland}
\begin{equation}
\tilde{\psi}(u) = (1-u)_{+}^8(32u^3+25u^2+8u+1),
\end{equation}
where $u_+ = \max\{u, 0\}$, and it is defined by
\begin{equation}
f_2({x}) = \sum\limits_{i=1}^{20} \tilde{\psi}(\|{x}-{z}_i\|_2),
\end{equation}
where ${z}_i$ ($i=1,\cdots,20$) are the center points of the regions of an equal area partitioned by Leopardi's recursive zonal sphere partitioning
procedure \cite{leopardi2006partition} \footnote{http://eqsp.sourceforge.net}. It should be noted that $f_1$ is an extremely smooth  spherical function,  while $f_2$  is in the Sobolev space $W^r(\mathbb S^d)$ with $r=4.5$ \cite{lin2021distributed}. 

In the simulations, the inputs $\{{x}_i\}_{i=1}^N$ of training samples are generated by Womersley's symmetric spherical $141$-designs, which includes $10014$ points on the unit sphere. The corresponding outputs $\{y_i\}_{i=1}^N$
are generated by the function $f_j$ ($j=1,2$) plus { truncated Gaussian} noise, i.e., for each point $x_i$,
\begin{equation}\label{Model1:fixed1}
y_i = f_j(x_i) + \varepsilon_i,
\end{equation}
 { where $\varepsilon_i$ is the independent   truncated Gaussian  noise $\mathcal{N}(0, \delta^2)$, i.e., $\varepsilon_i$ is initially generated by the Gaussian noise $\mathcal{N}(0, \delta^2)$ and then truncated to $[-10, 10]$}. The inputs $\{x_i'\}_{i=1}^{N'}$  of testing samples are  { $N'=10000$ generalized
spiral points on the unit sphere}, and the corresponding outputs $\{y_i'\}_{i=1}^{N'}$ are generated by $y_i'=f_j(x_i')$. For the approximation of testing function
$f_1$, we use the positive
definite function
$\phi_1(x_1, x_2)=\exp(-\frac{\|x_1 - x_2\|_2^2}{2\sigma^2})$ with the regularization parameter $\lambda$ being chosen from the set
$\{\frac{1}{2^q} | \frac{1}{2^q}>10^{-10}, q=0,1,2,\cdots\}$  and the  width $\sigma$ being chosen from 10 values that are drawn in a logarithmic, equally spaced interval
$[0.1, 1]$ \footnote{For the noise-free training data, i.e., $\delta=0$, the  width $\sigma$ is chosen from 10 values that are drawn in logarithmic,
equally spaced interval $[0.028, 0.28]$.}.  For the approximation of testing function $f_2$, the positive
definite function is defined as $\phi_2(x_1, x_2)=\tilde{\psi}(\|x_1 - x_2\|_2)$ with the regularization parameter $\lambda$ being  chosen from the
set $\{\frac{1}{1.5^q} | \frac{1}{1.5^q}>10^{-10}, q=0,1,2,\cdots\}$.
All parameters in the simulations are selected by grid search.

\begin{figure*}[t]
    \centering
    \subfigcapskip=-2pt
    \subfigure{\includegraphics[width=3.15cm,height=2.45cm]{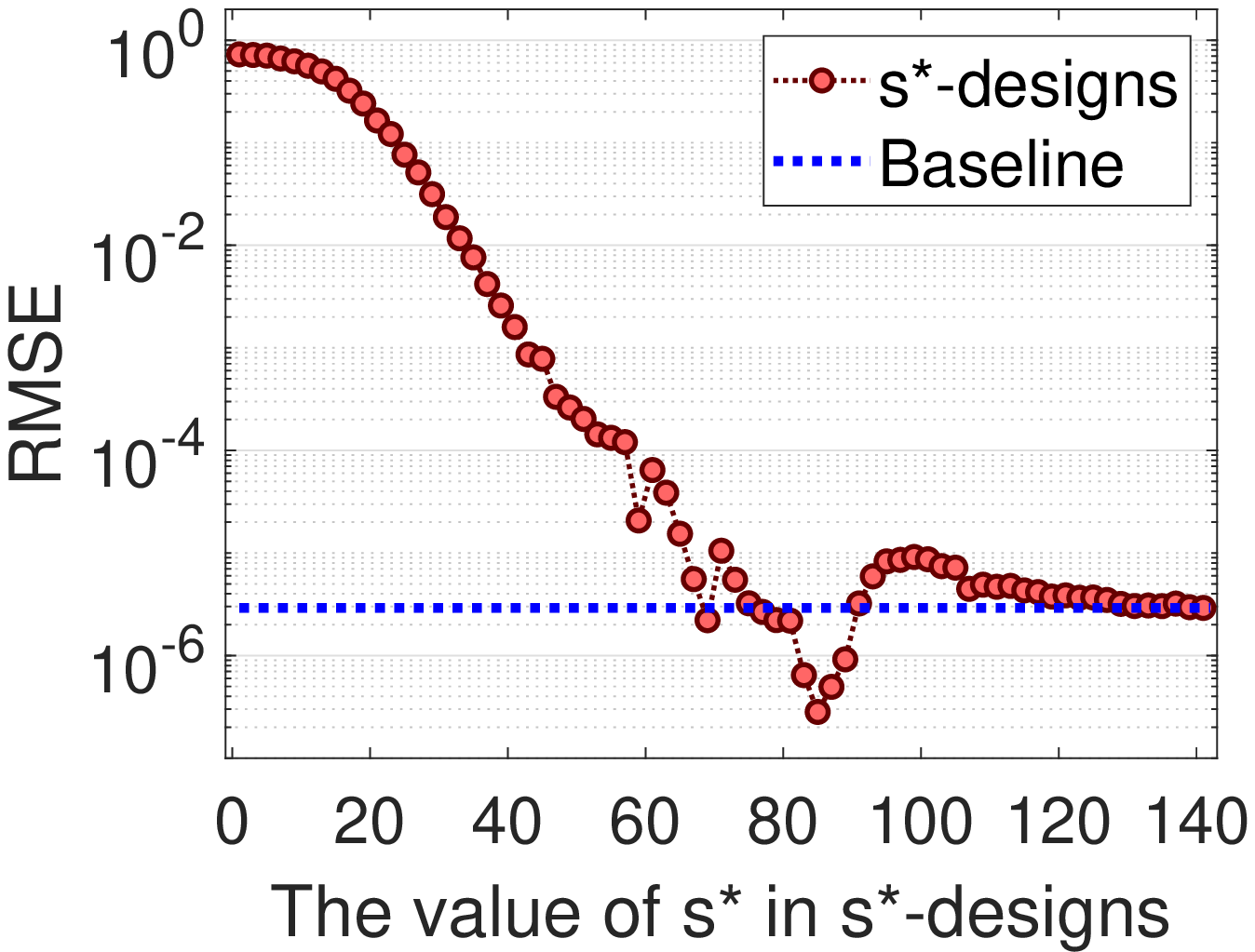}}
    \subfigure{\includegraphics[width=3.15cm,height=2.45cm]{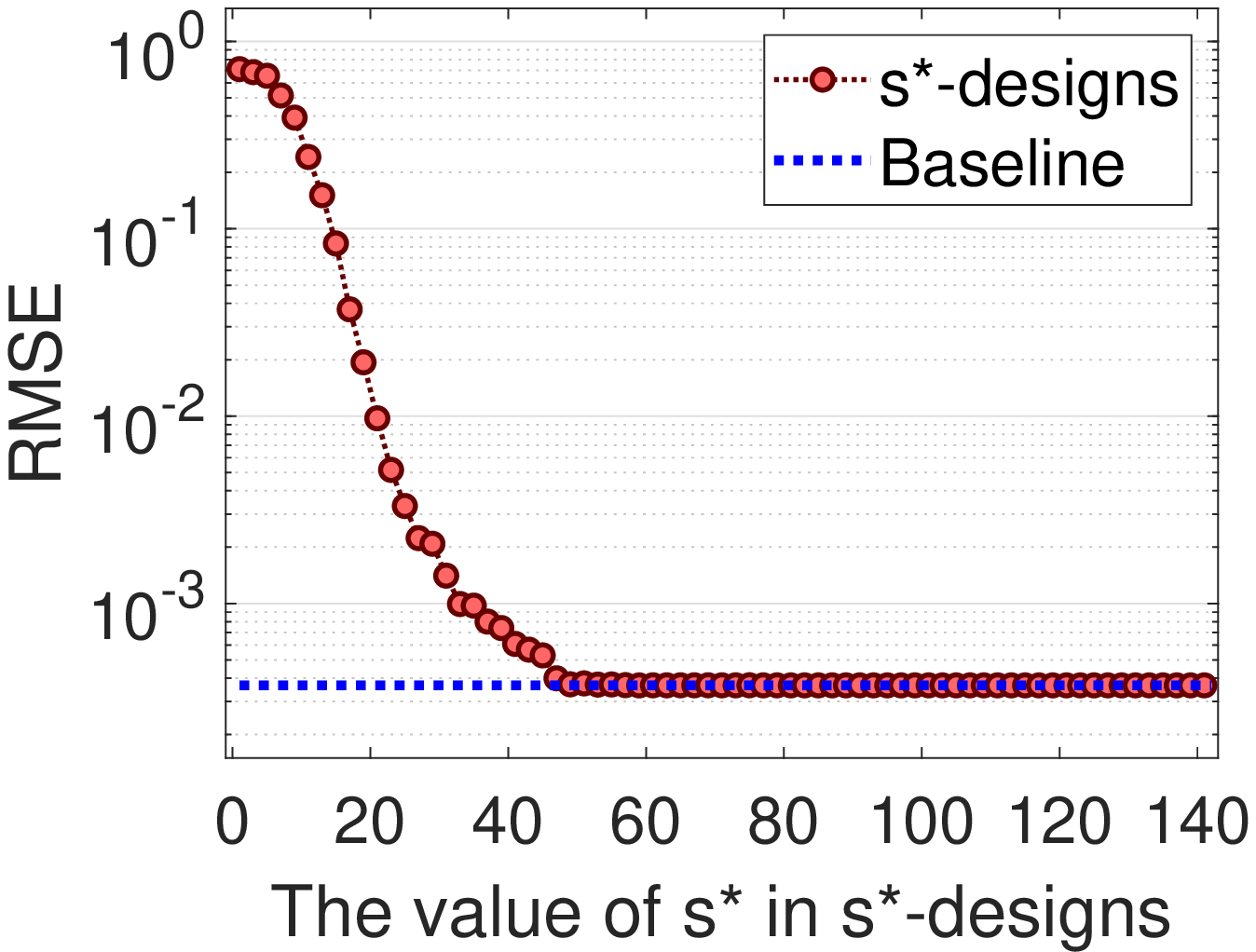}}
    \subfigure{\includegraphics[width=3.15cm,height=2.45cm]{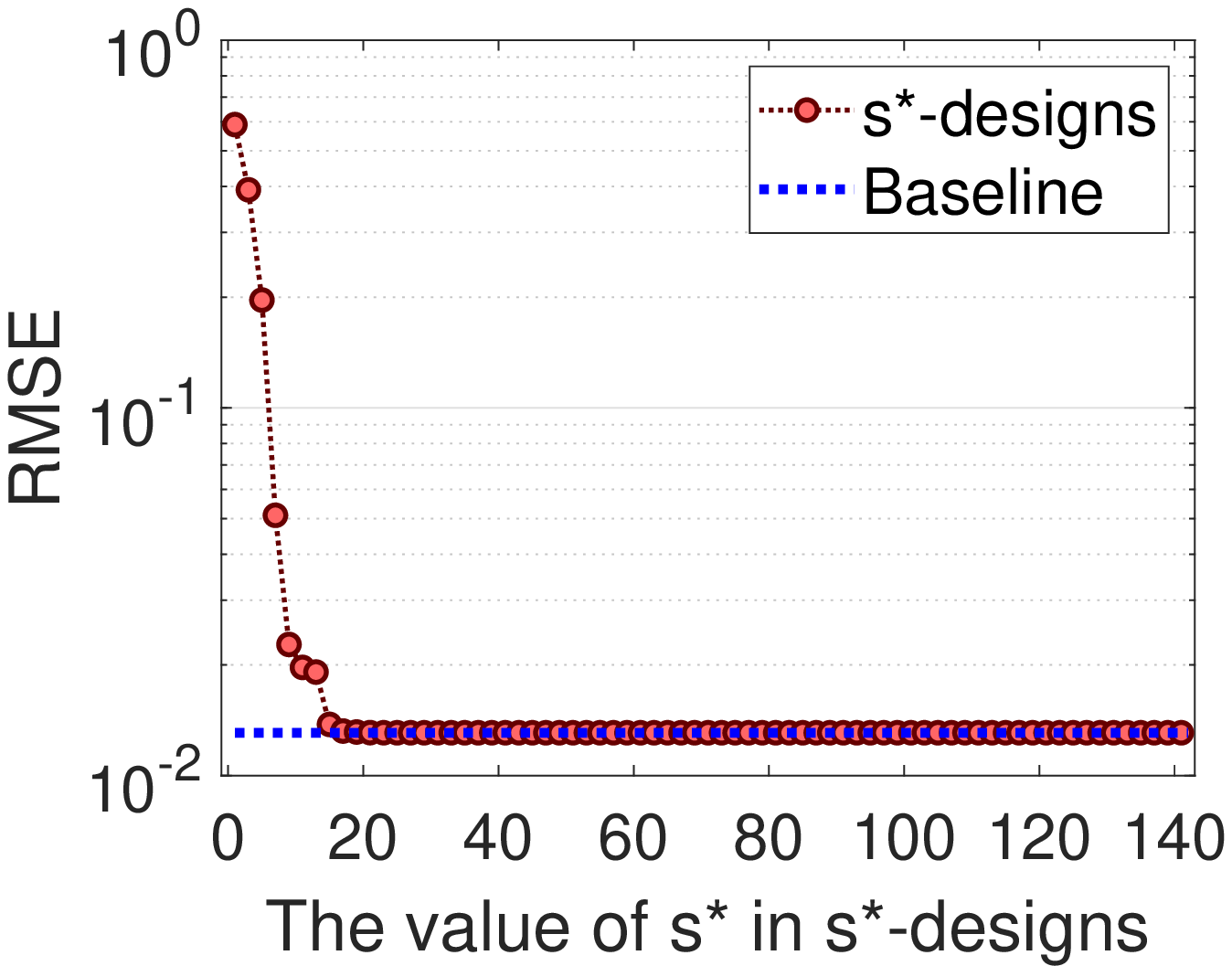}}
    \subfigure{\includegraphics[width=3.15cm,height=2.45cm]{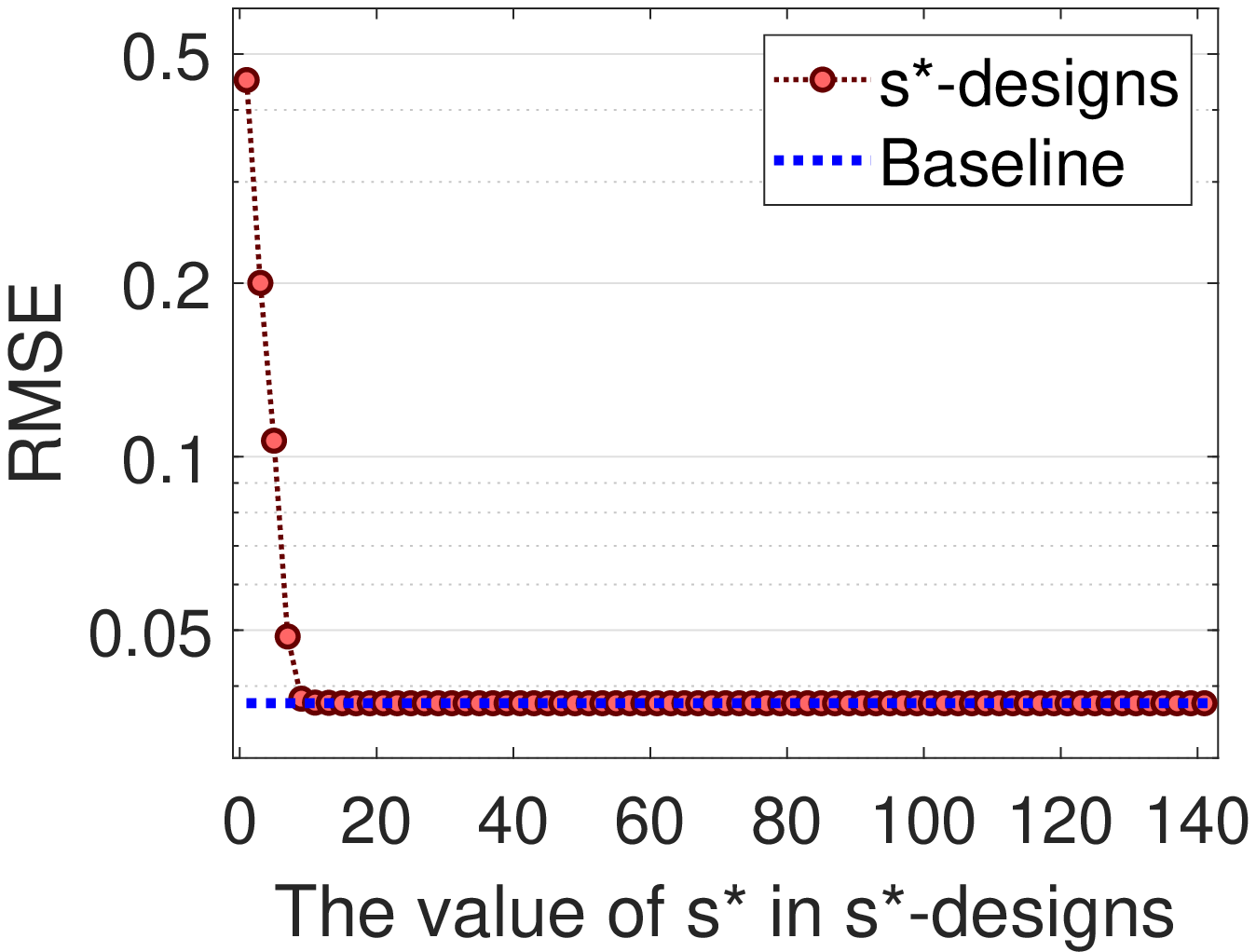}}\\
    \vspace{-0.08in}
    \setcounter{subfigure}{0}
    \subfigure[noise free]{\includegraphics[width=3.15cm,height=2.45cm]{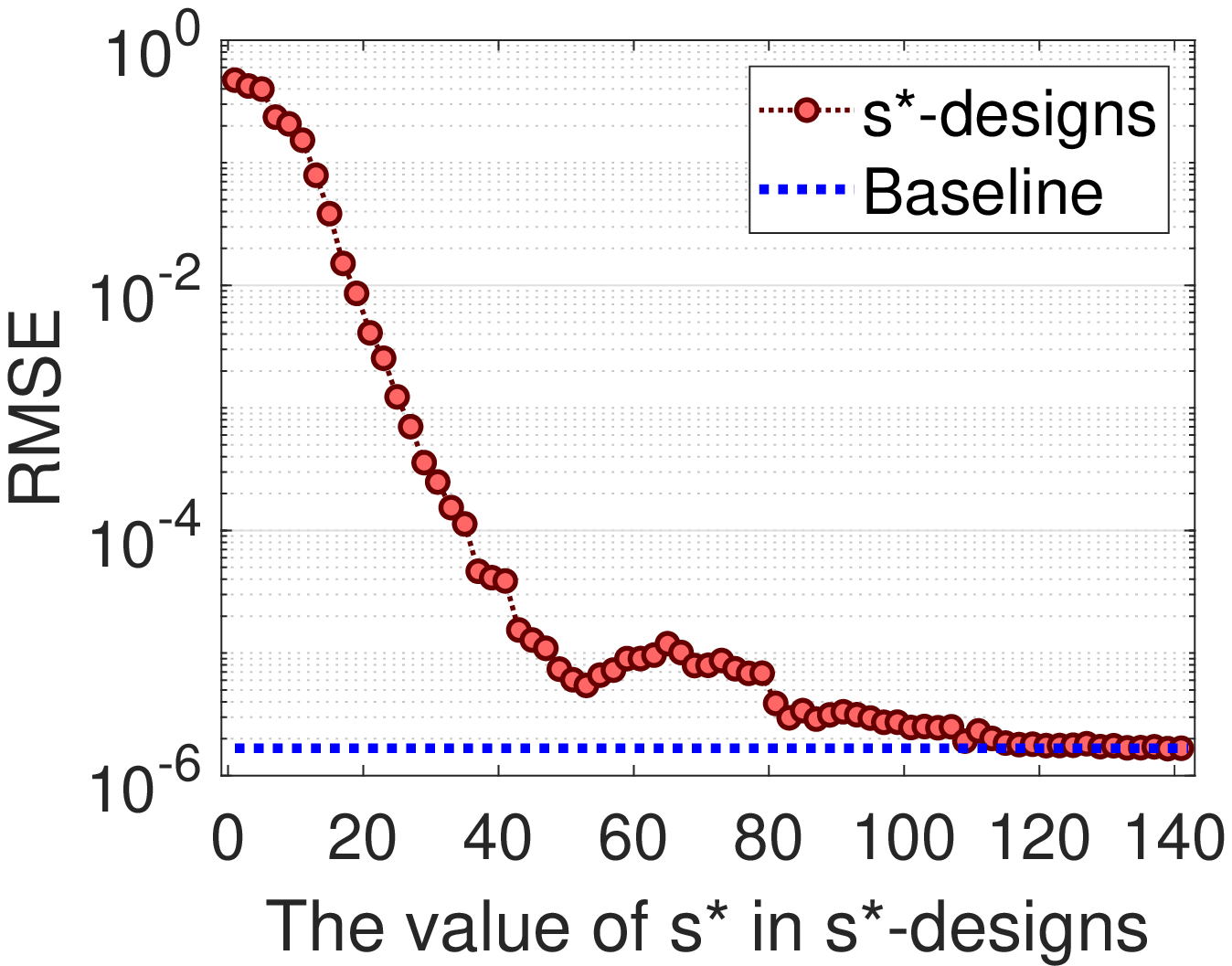}}
    \subfigure[$\delta=0.001$]{\includegraphics[width=3.15cm,height=2.45cm]{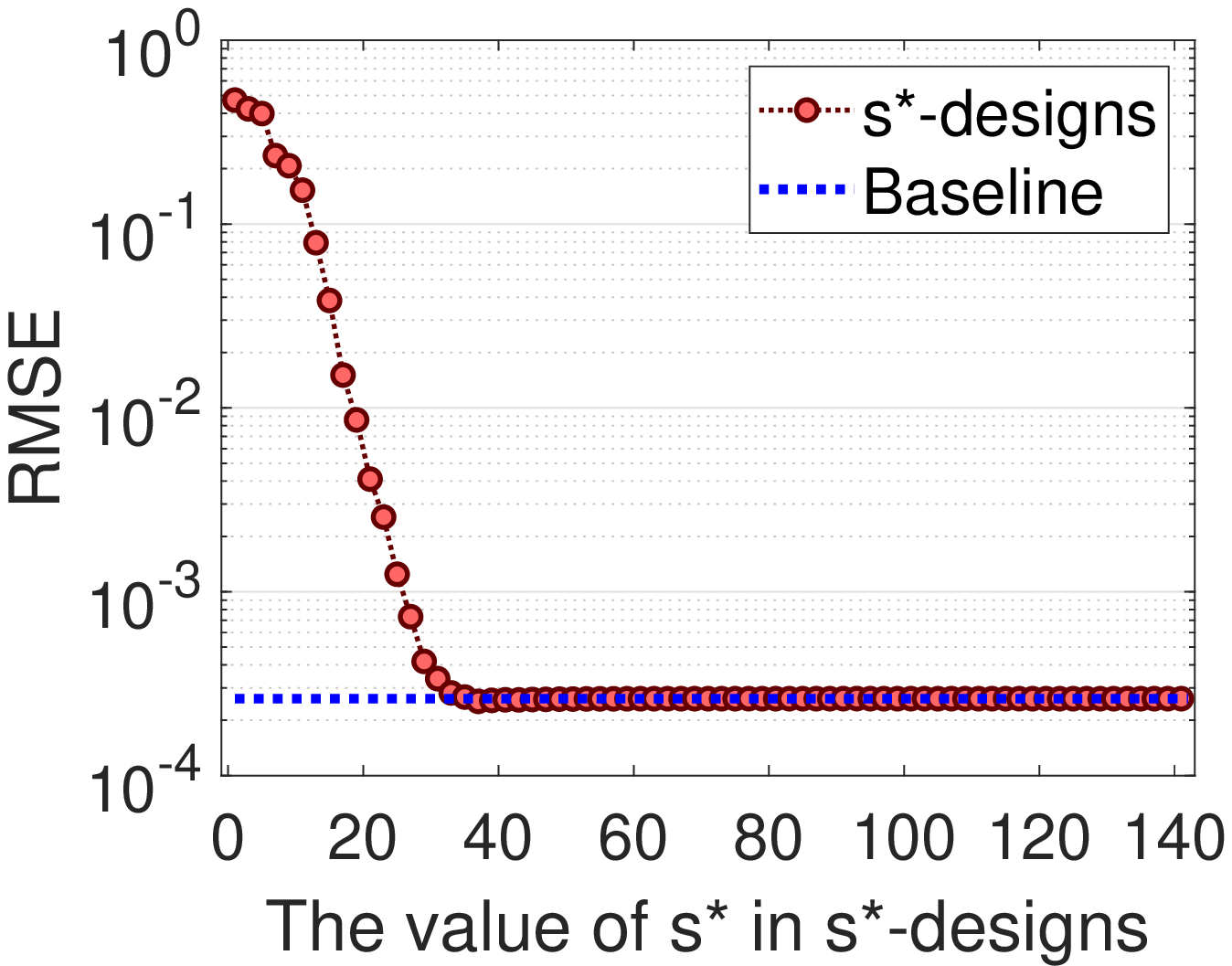}}
    \subfigure[$\delta=0.1$]{\includegraphics[width=3.15cm,height=2.45cm]{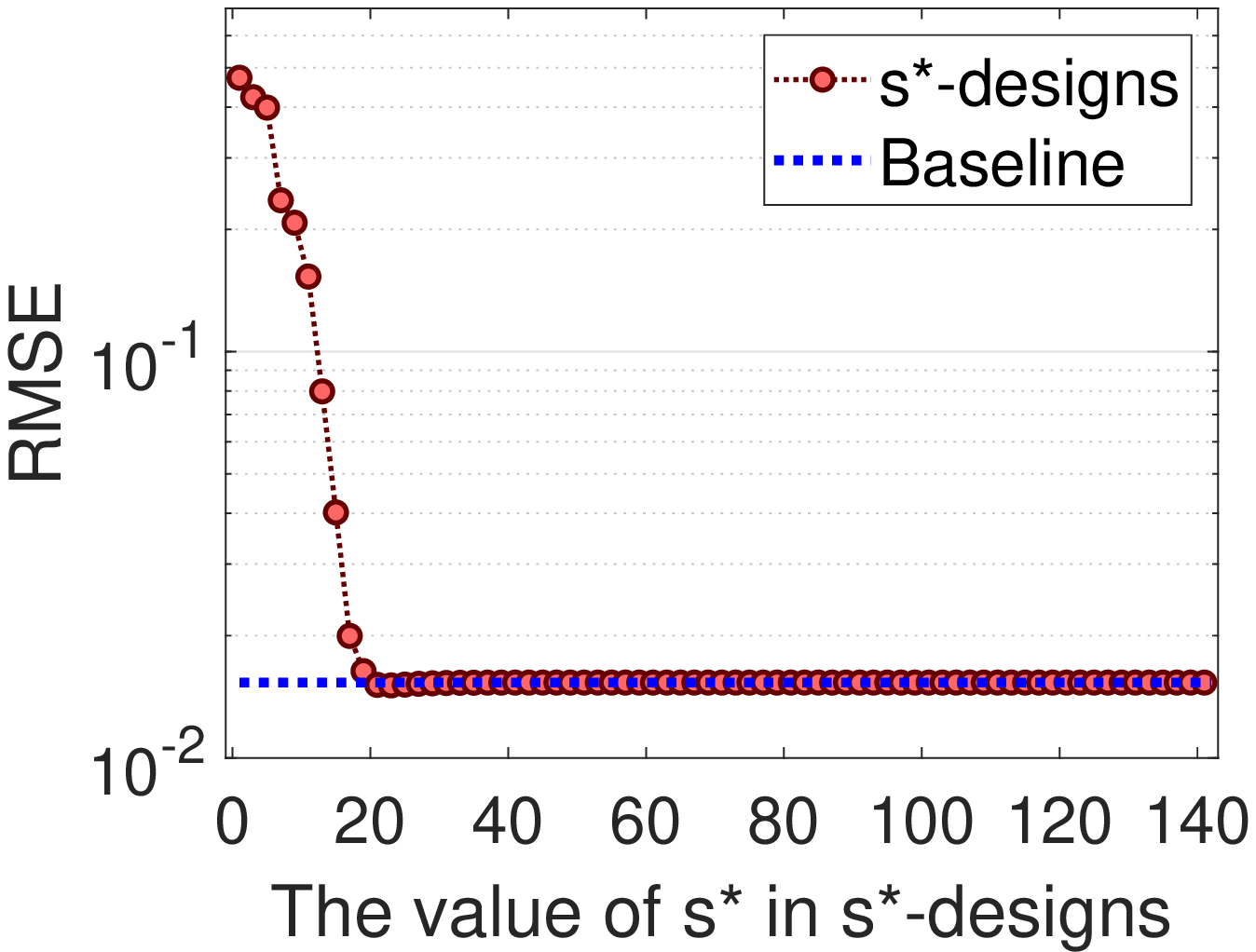}}
    \subfigure[$\delta=0.5$]{\includegraphics[width=3.15cm,height=2.45cm]{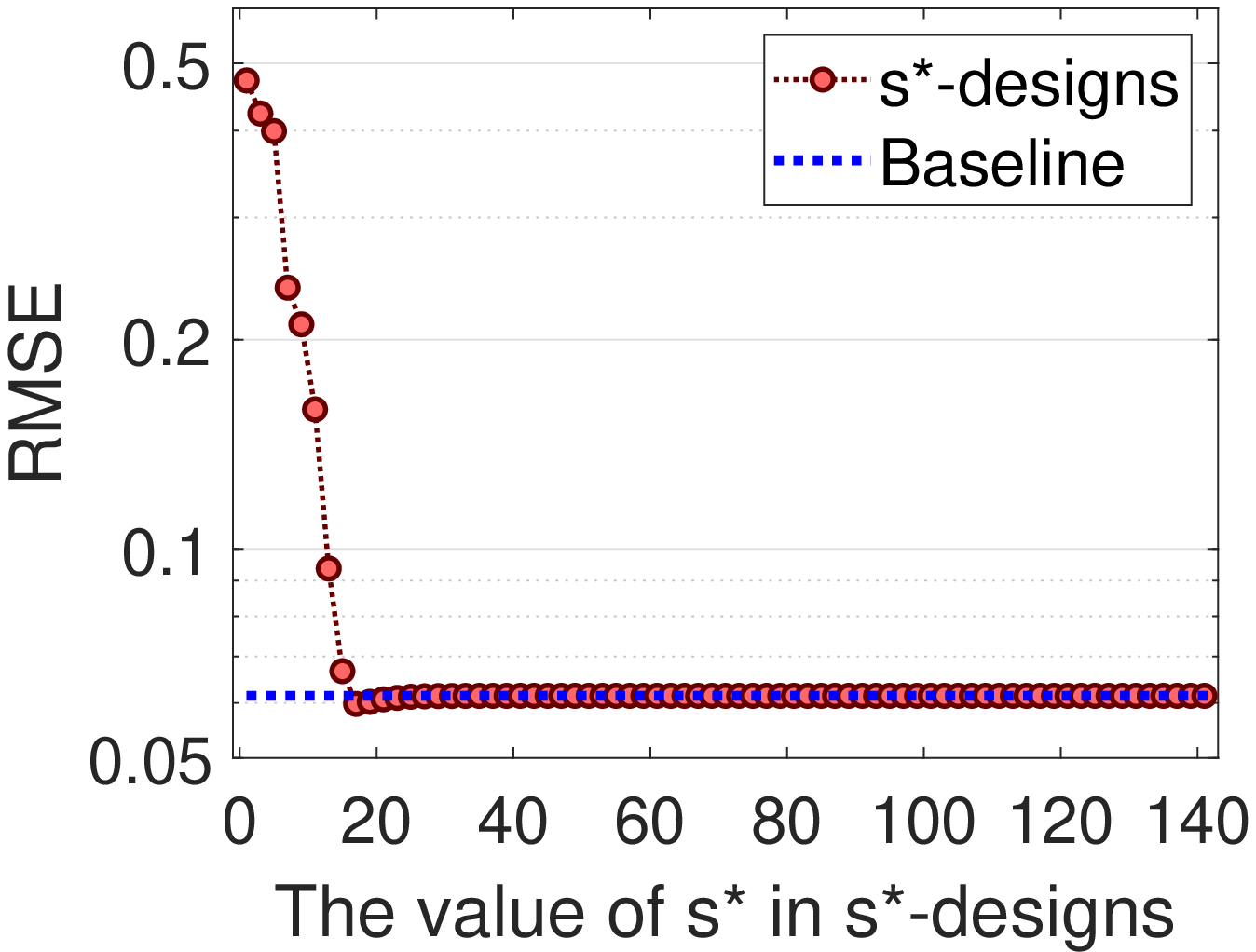}}
    \caption{The relation between RMSE and the number $s^*$ for { sketching} with $s^*$-designs under different levels of { truncated Gaussian} noise. {The sub-figures of the top and bottom rows are the results for the approximation of $f_1$ and $f_2$, respectively.} }\label{RMSERecovery_sDesign}
\end{figure*}

\begin{figure*}[t]
    \centering
    \subfigcapskip=-2pt
    \subfigure{\includegraphics[width=3.15cm,height=2.5cm]{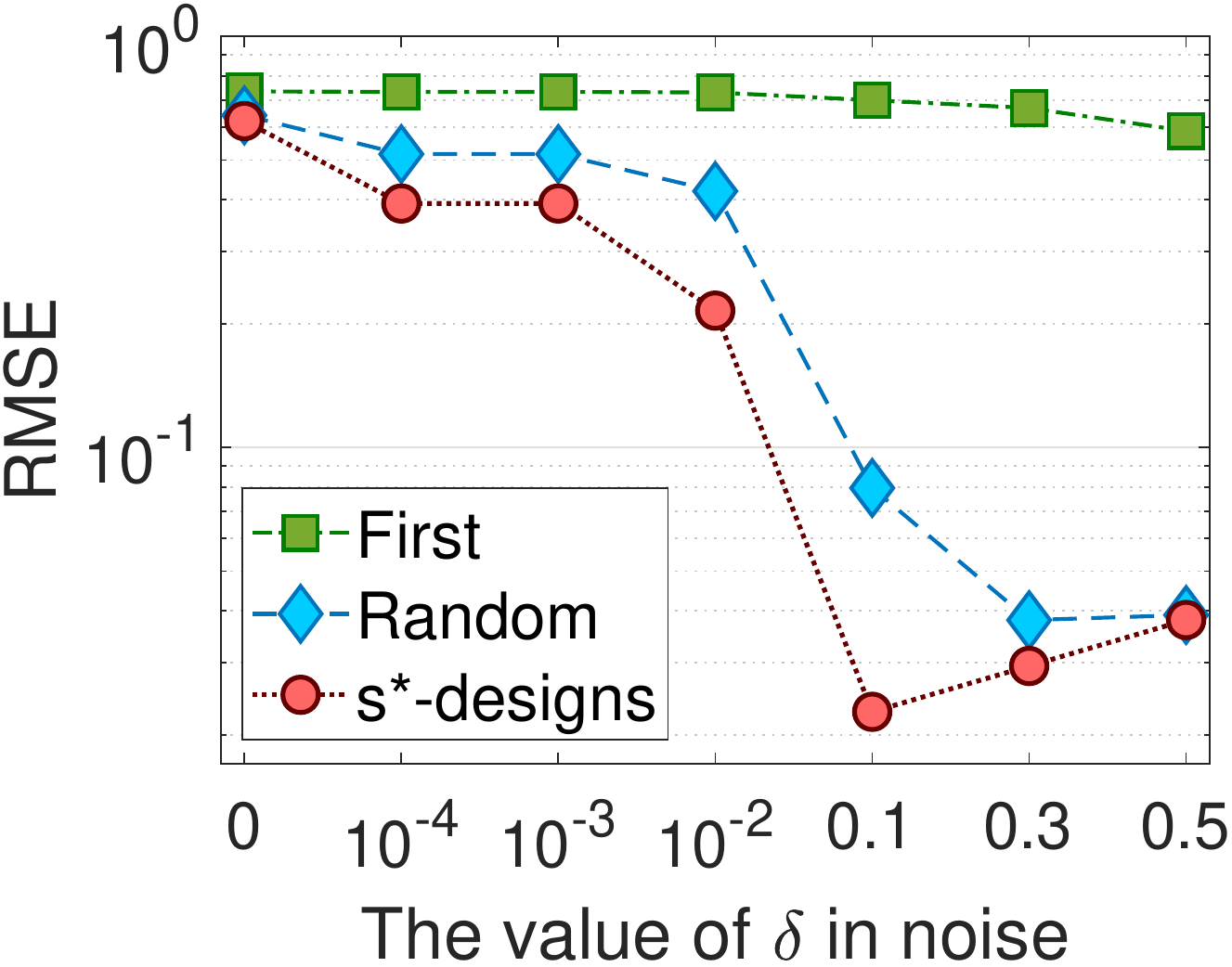}}
    \subfigure{\includegraphics[width=3.15cm,height=2.45cm]{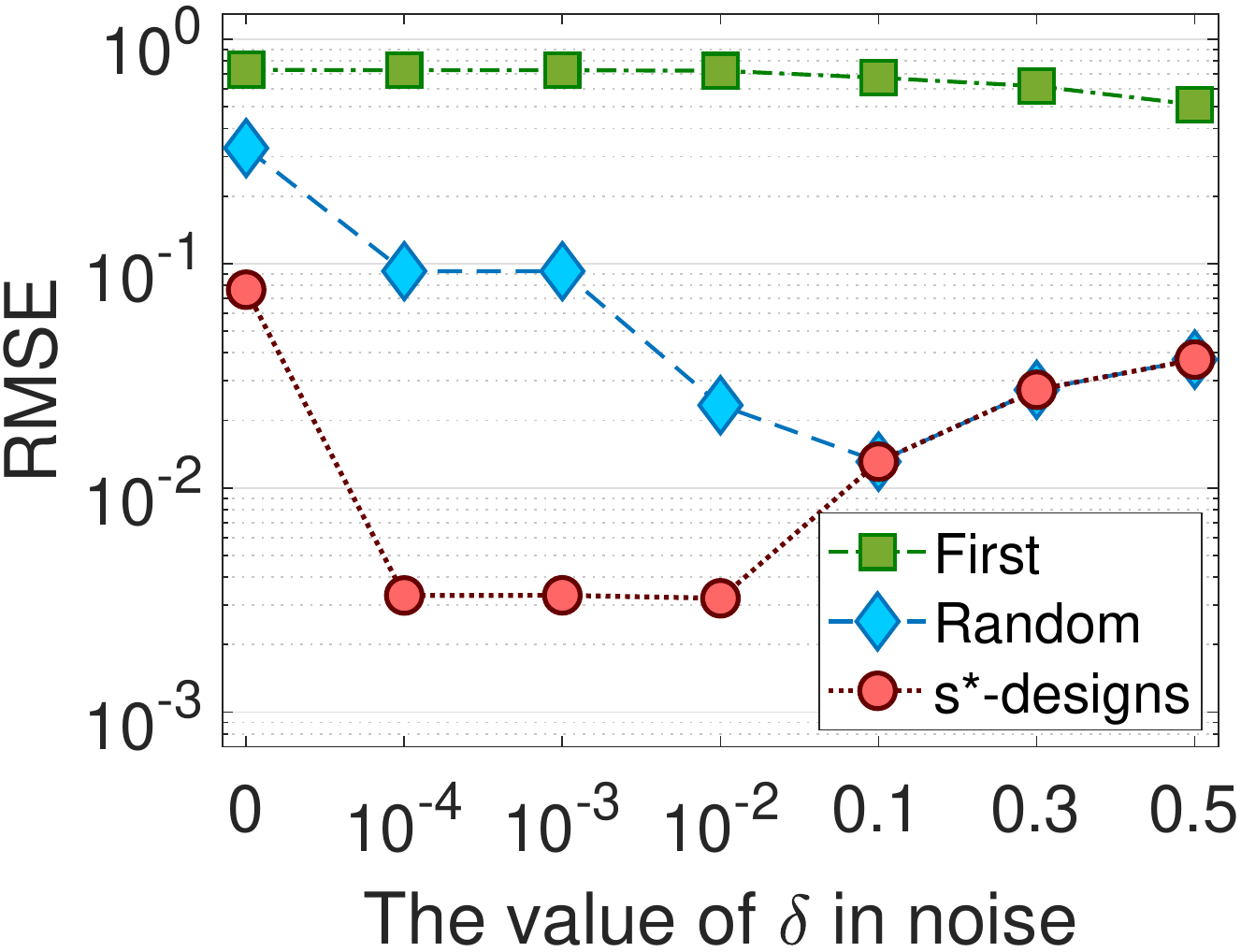}}
    \subfigure{\includegraphics[width=3.15cm,height=2.45cm]{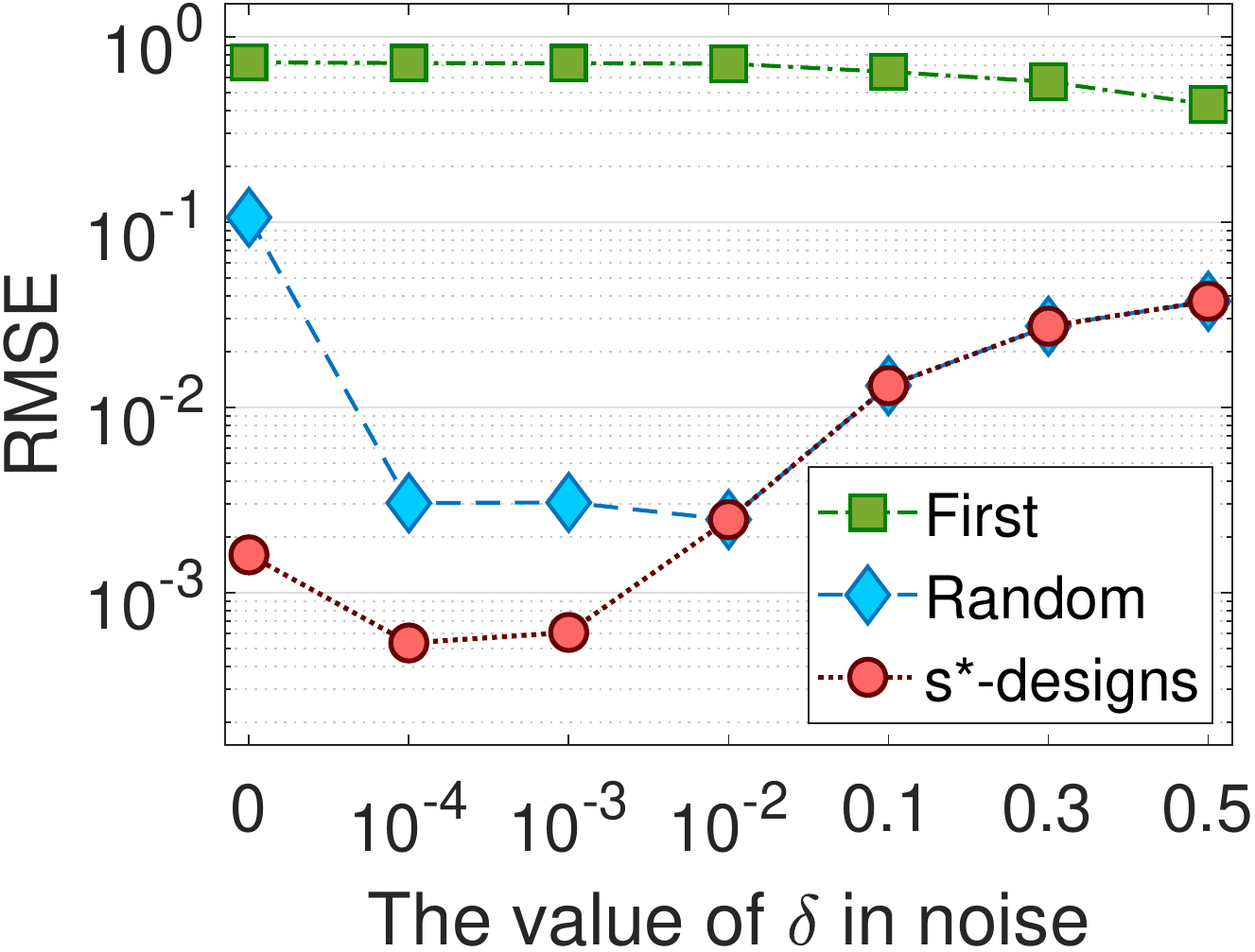}}
    \subfigure{\includegraphics[width=3.15cm,height=2.45cm]{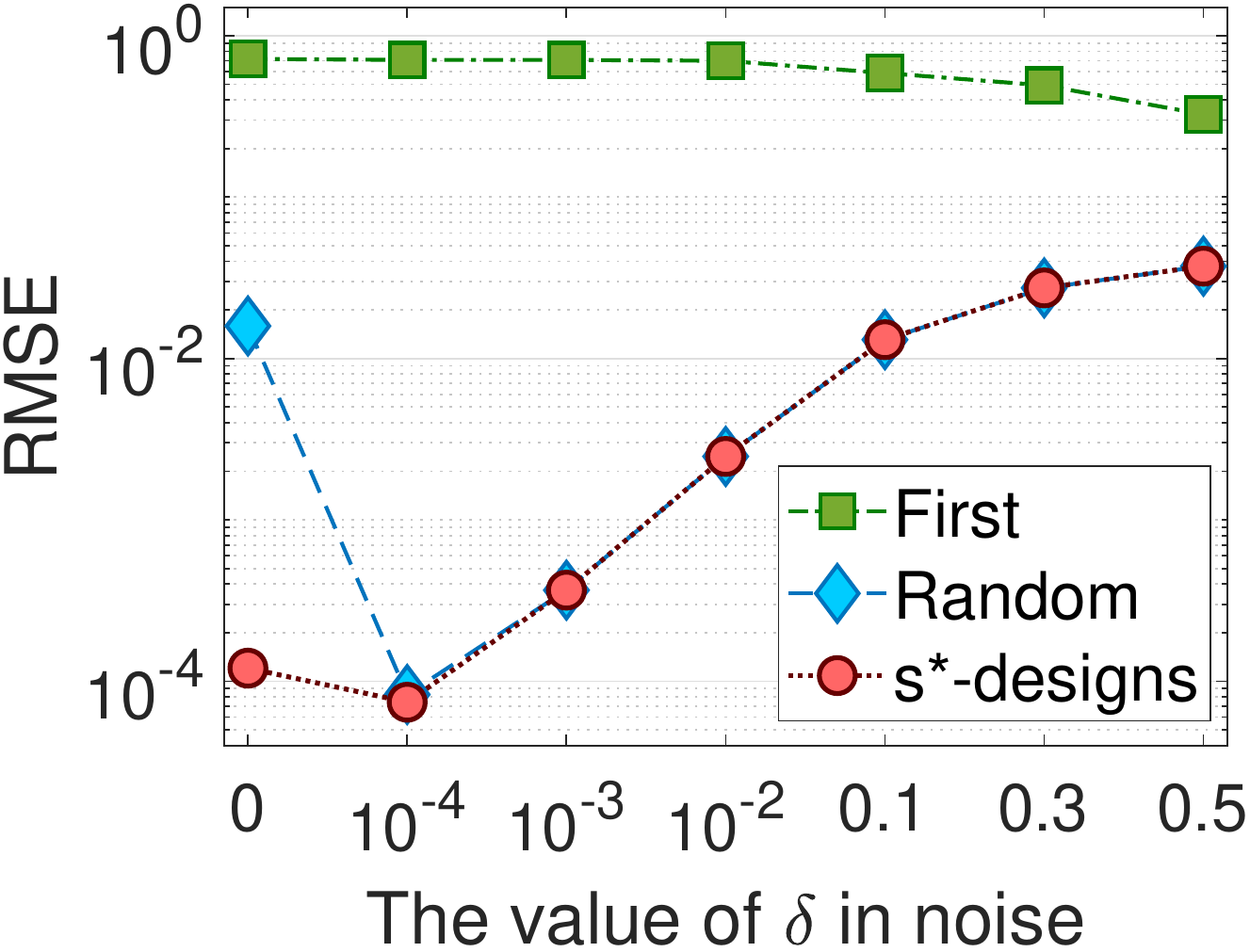}}\\
    \vspace{-0.08in}
    \setcounter{subfigure}{0}
    \subfigure[$s^*\hspace{-0.02in}=\hspace{-0.02in}9$ (SR$=\hspace{-0.02in}0.48\%$)]{\includegraphics[width=3.15cm,height=2.45cm]{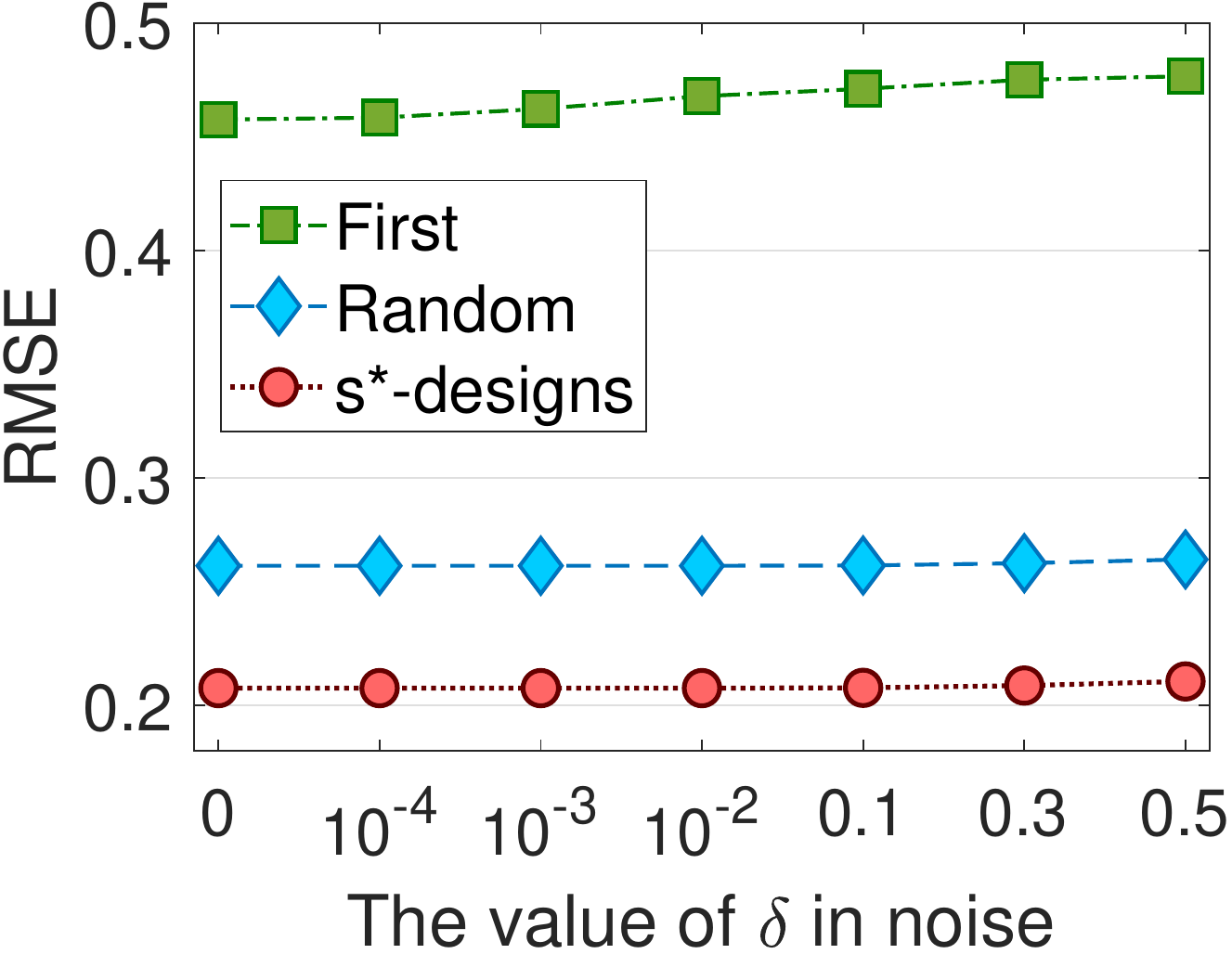}}
    \subfigure[$s^*\hspace{-0.02in}=\hspace{-0.02in}25$ (SR$=\hspace{-0.02in}3.28\%$) ]{\includegraphics[width=3.15cm,height=2.45cm]{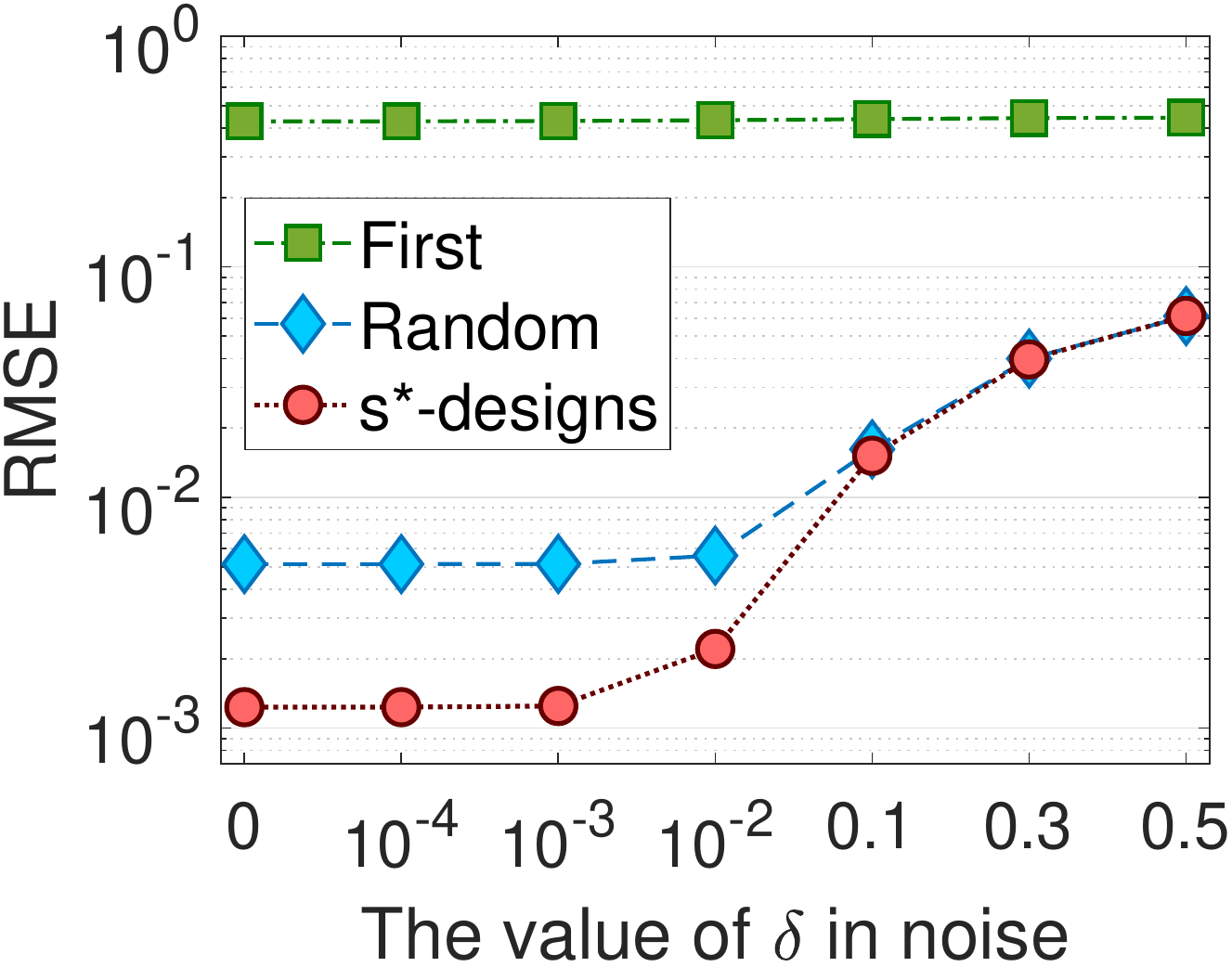}}
    \subfigure[$s^*\hspace{-0.02in}=\hspace{-0.02in}41$ (SR$=\hspace{-0.02in}8.63\%$) ]{\includegraphics[width=3.15cm,height=2.45cm]{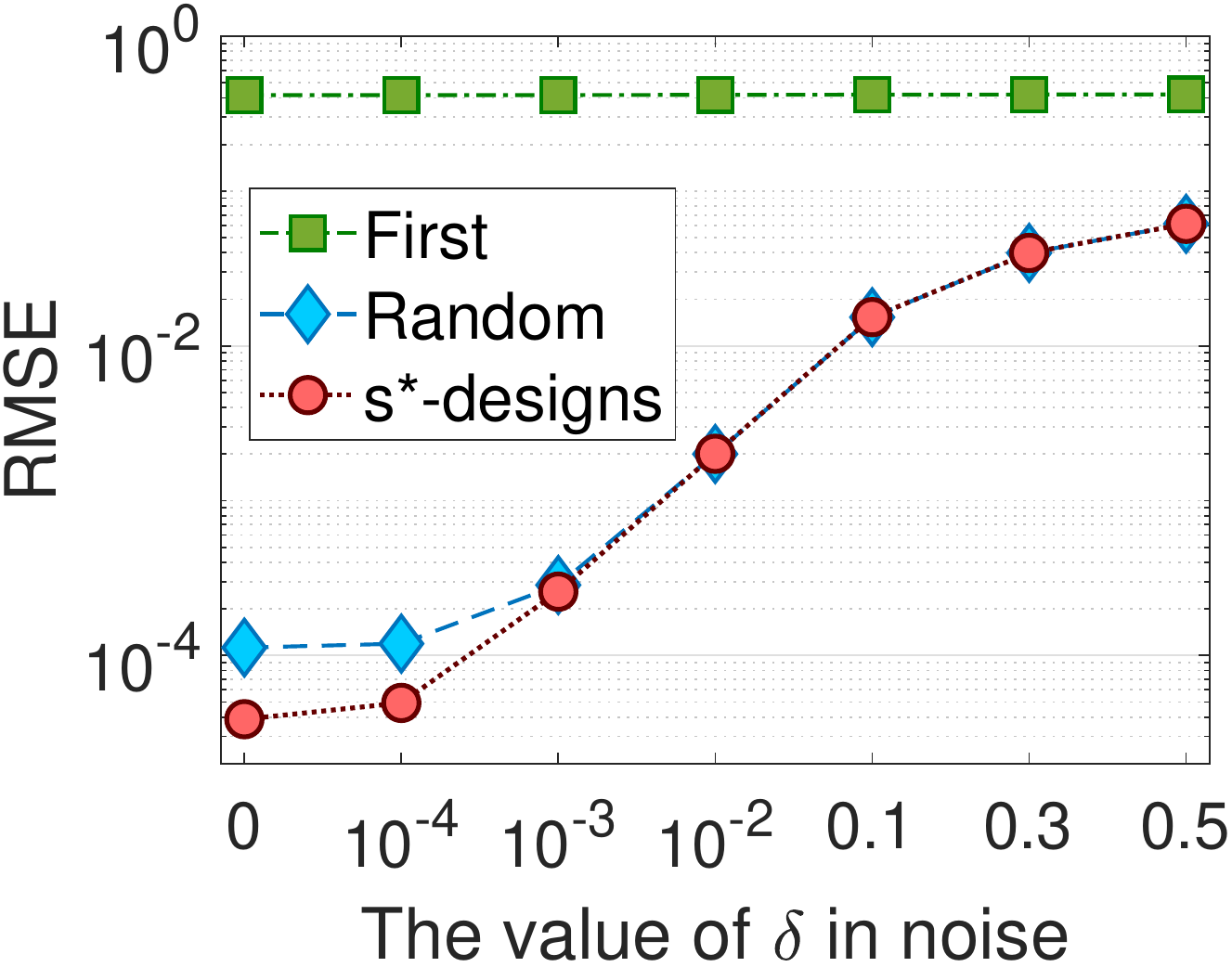}}
    \subfigure[$s^*\hspace{-0.04in}=\hspace{-0.03in}57$ (SR$=\hspace{-0.04in}16.54\%$) ]{\includegraphics[width=3.15cm,height=2.45cm]{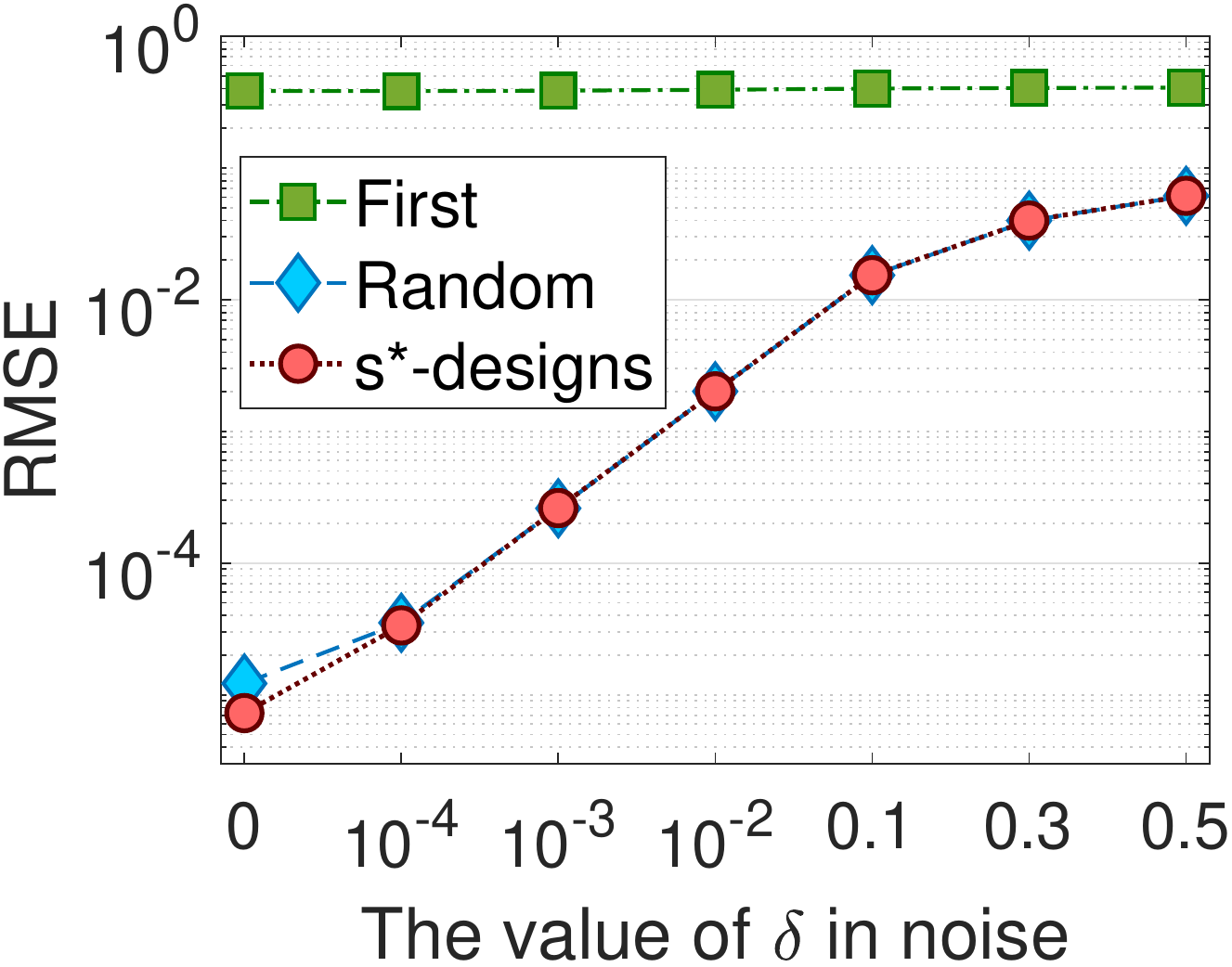}}
	\caption{The comparison of RMSE among the three { sketching} methods with increasing levels of { truncated Gaussian} noise for fixed numbers of $s^*$.  {The sub-figures of the top and bottom rows are the results for the approximation of $f_1$ and $f_2$, respectively.}}\label{RMSEComparison}
\end{figure*}

The simulations are done for three purposes. The first one is devoted to investigating the approximation performance of { sketching} with $s^*$-designs under different levels of { truncated Gaussian} noise. The second simulation focuses on comparisons for the three { sketching} methods with different sampling ratios (SRs). The last one aims at giving an intuitive visualization for the recovery results of { sketching} with $s^*$-designs.

\textbf{Simulation 1: } In this simulation, we generate the training samples according to (\ref{Model1:fixed1}) for { truncated Gaussian} noise with
standard deviations $\delta\in\{0, 0.001, 0.1, 0.5\}$.  For each $\delta$, we record RMSEs of { sketching} with $s^*$-designs by varying 
the number $s^*$ in the set $\{1, 3, 5, \cdots, 141\}$. 
It is notable that the case of $s^*=141$ is the standard regularized least squares (\ref{KRR}) with training all samples, and then the corresponding RMSE provides a baseline to assess the performance of { sketching} with $s^*$-designs. The results of RMSE on the  testing set as a function of the number $s^*$, which  corresponds to the sampling ratio (SR), for different levels of { truncated Gaussian} noise are shown in Figure \ref{RMSERecovery_sDesign}. From the results, we can conclude the following assertions: 1) $s^*$-designs do not perform very well for noise-free training data ($\delta=0$), and their accuracy is stable and comparable with regularized least squares only when $s^*$ is larger than $119$, i.e., the SR reaches more than $71.32\%$ (the number of samples of $119$-designs is $7142$, and $\frac{7142}{10014}\approx 0.7132$).
2) For noisy training data, RMSE decreases rapidly as the number $s^*$ increases, reaching the accuracy of regularized least squares under a relatively small $s^*$. In addition, { sketching} with $s^*$-designs is more effective on data with higher levels of { truncated Gaussian} noise.  This verifies our theoretical assertions after Theorem \ref{Theorem:lower-bound} that  the accommodation of large noise in (\ref{Model1:fixed}) leaves a large room for  { sketching}, implying  that larger noise admits smaller $s^*$ for { sketching}  with spherical $s^*$-designs.

\begin{figure*}[t]
    \centering
    \subfigcapskip=-2pt
    \subfigure{\includegraphics[width=3.1cm, height=3.7cm]{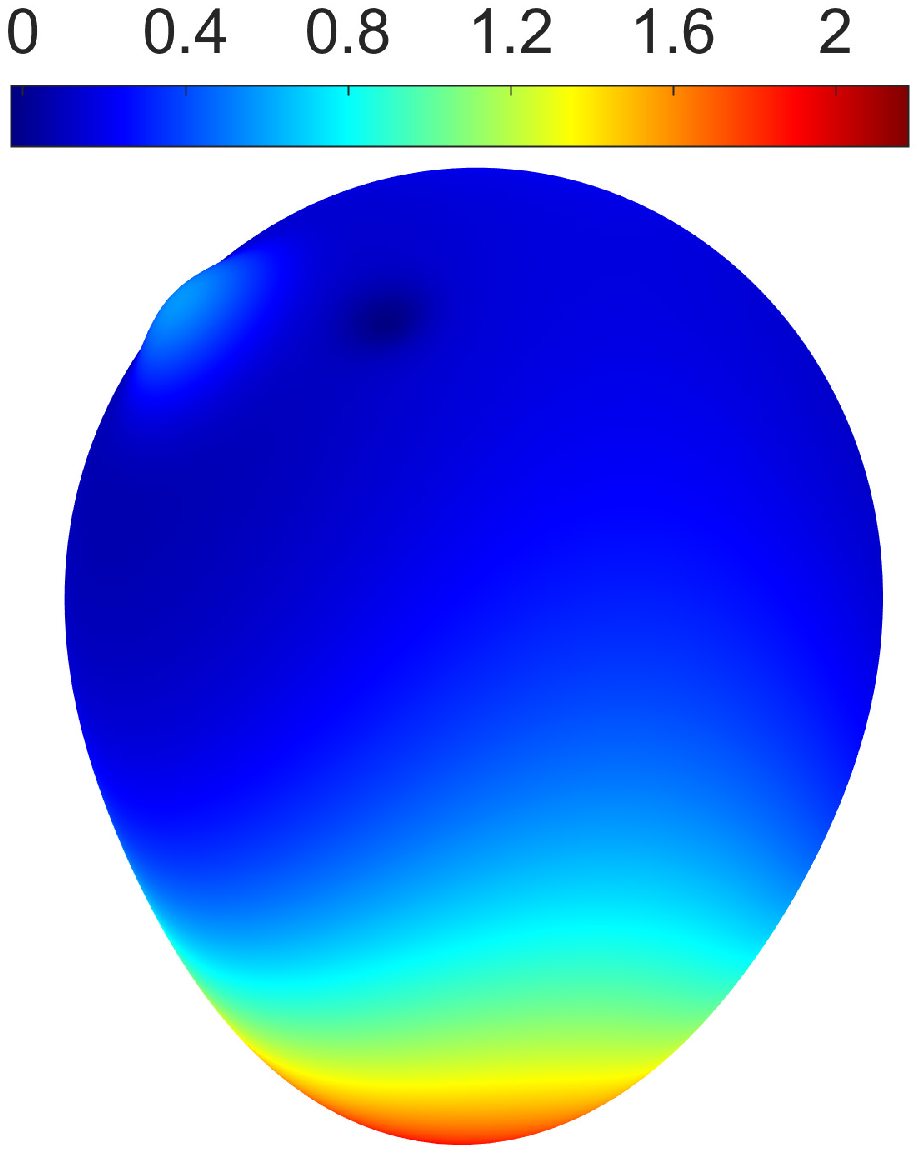}}
    \subfigure{\includegraphics[width=3.1cm, height=3.7cm]{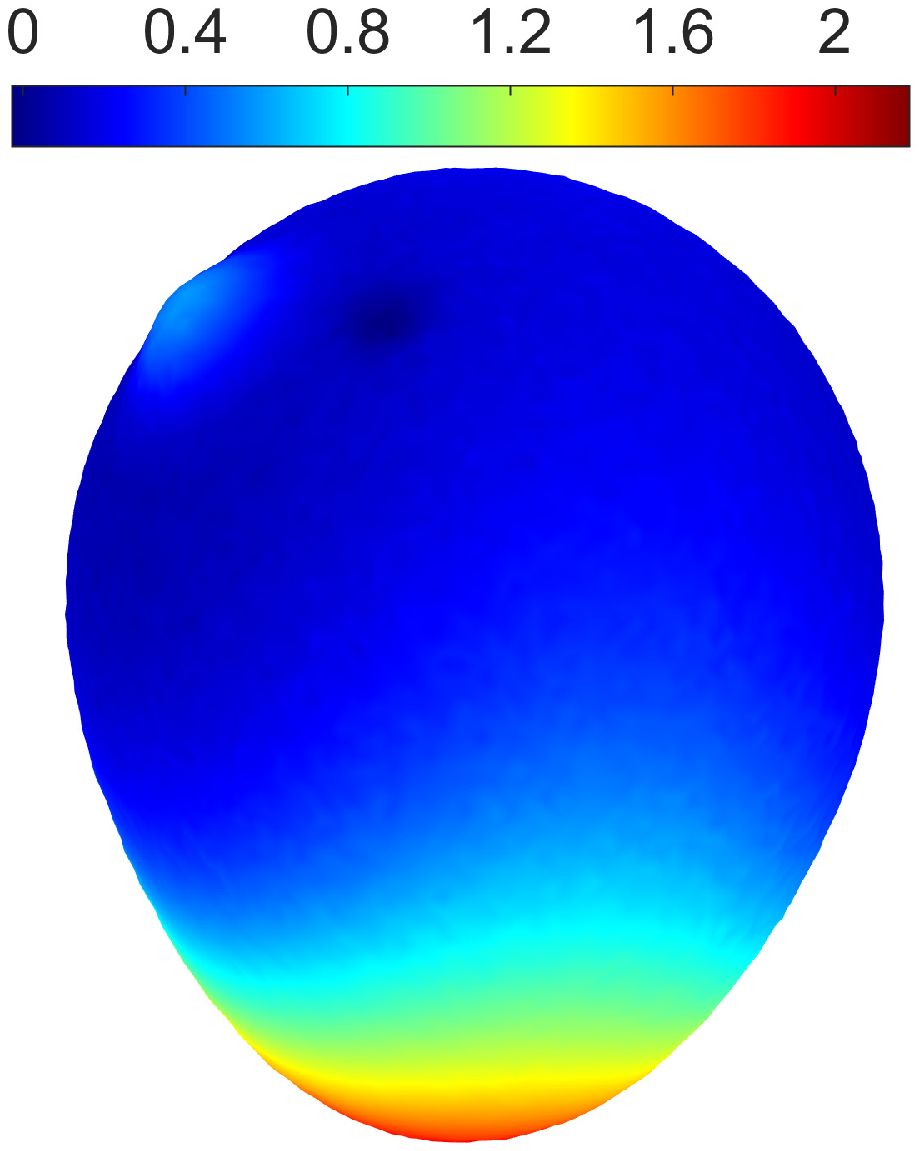}}
    \subfigure{\includegraphics[width=3.1cm, height=3.7cm]{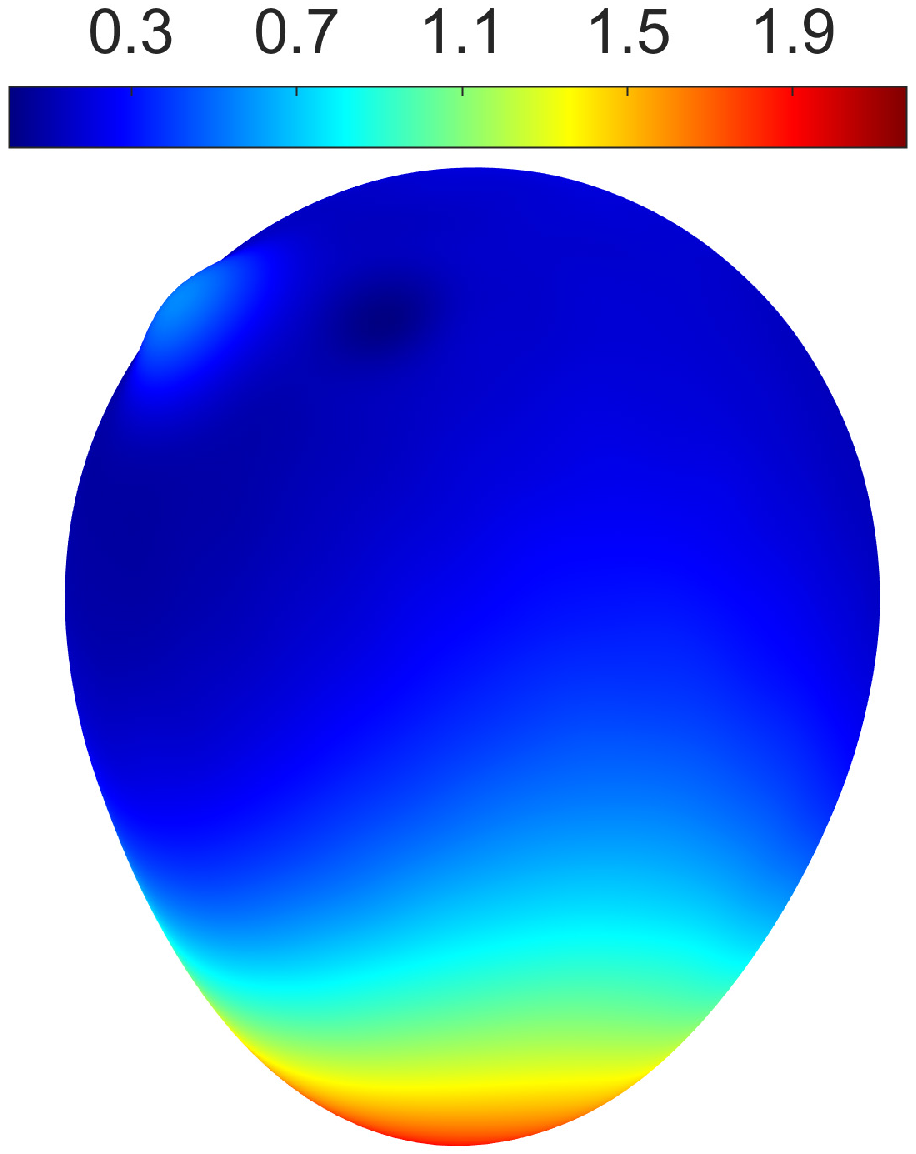}}
    \subfigure{\raisebox{0.05\height}{\includegraphics[width=3.1cm, height=3.5cm]{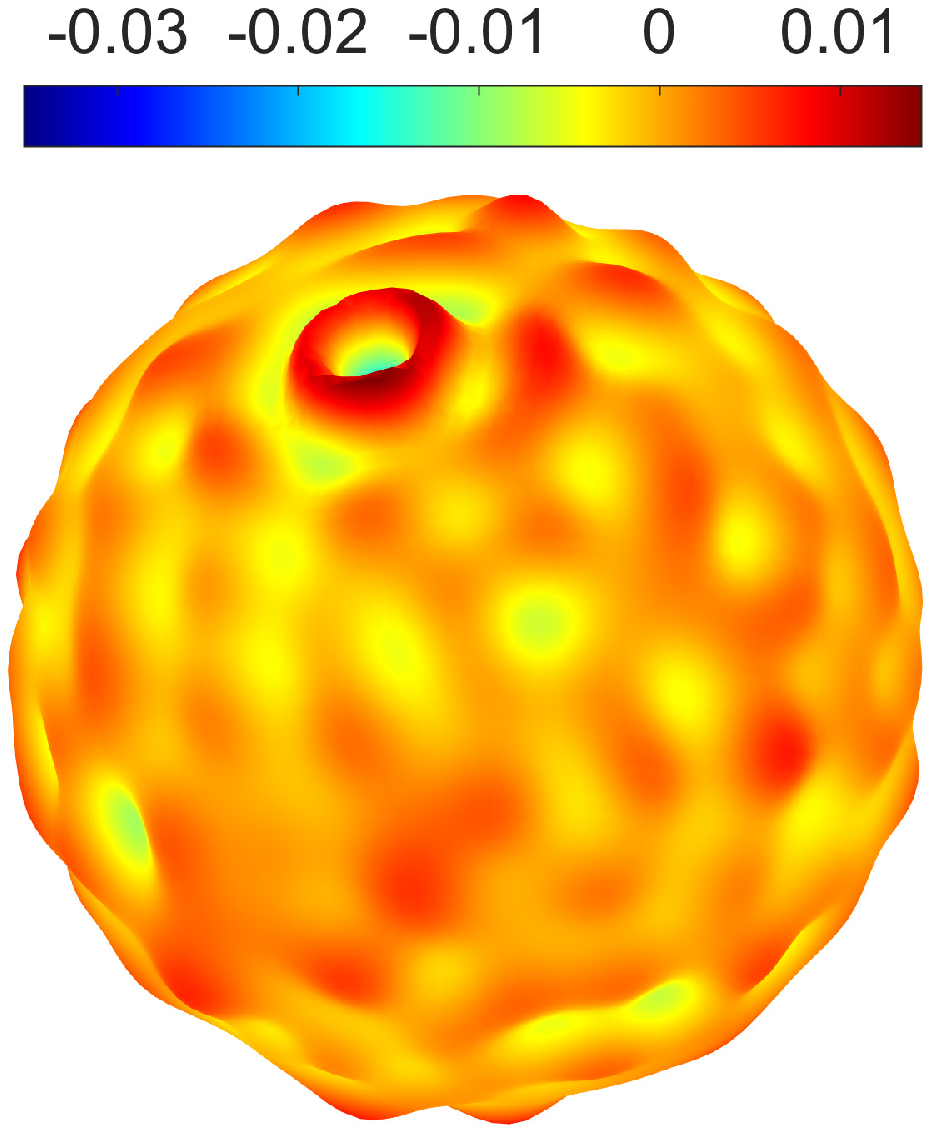}}}\\
    \vspace{-0.08in}
    \setcounter{subfigure}{0}
    \subfigure[exact function]{\includegraphics[width=3.1cm, height=3.7cm]{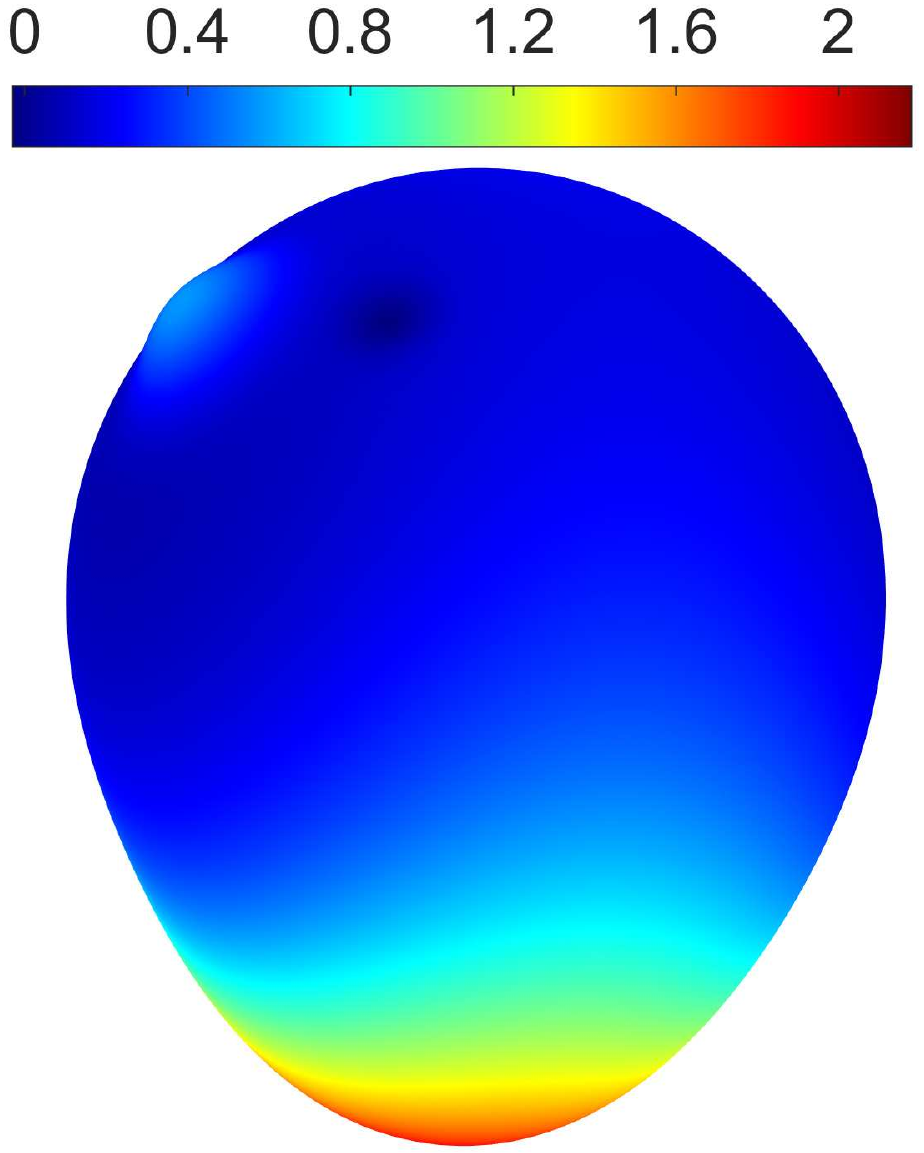}}
    \subfigure[noisy function]{\includegraphics[width=3.1cm, height=3.7cm]{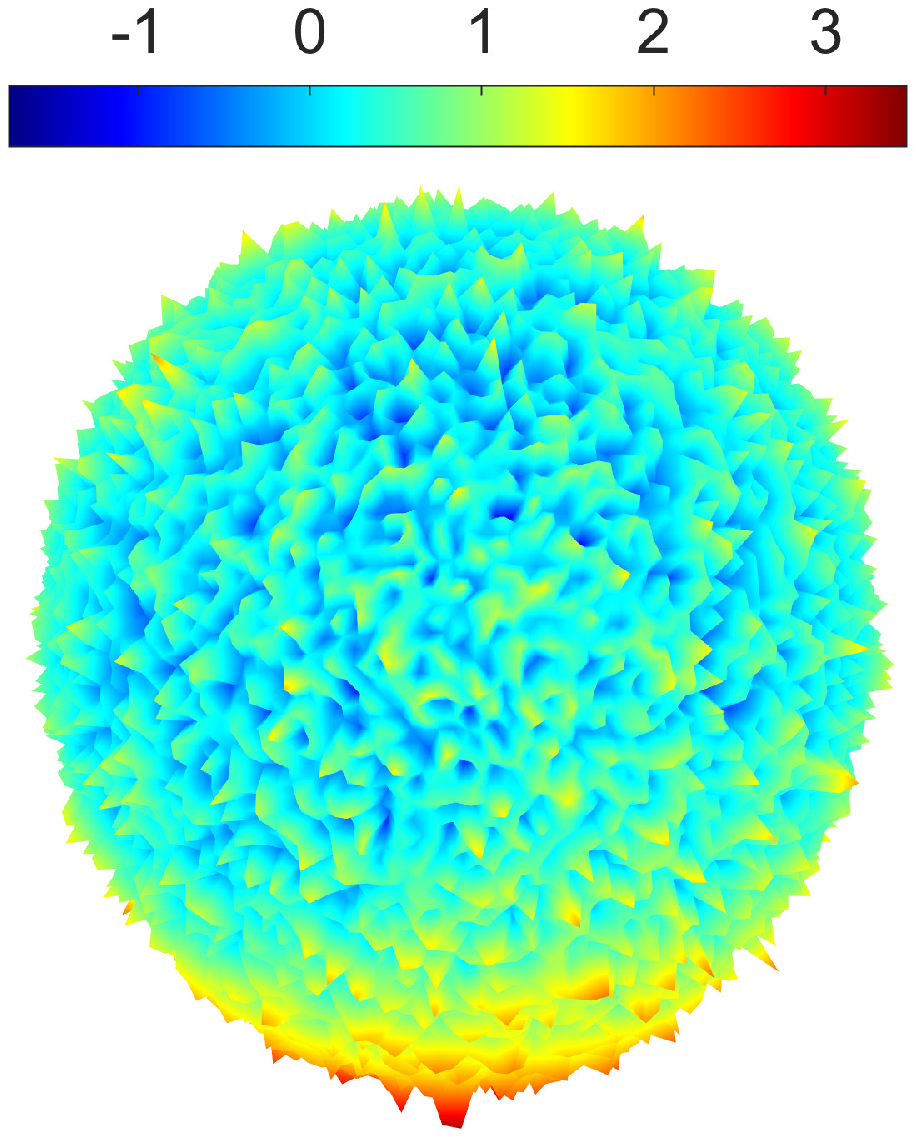}}
    \subfigure[recovery]{\includegraphics[width=3.1cm, height=3.68cm]{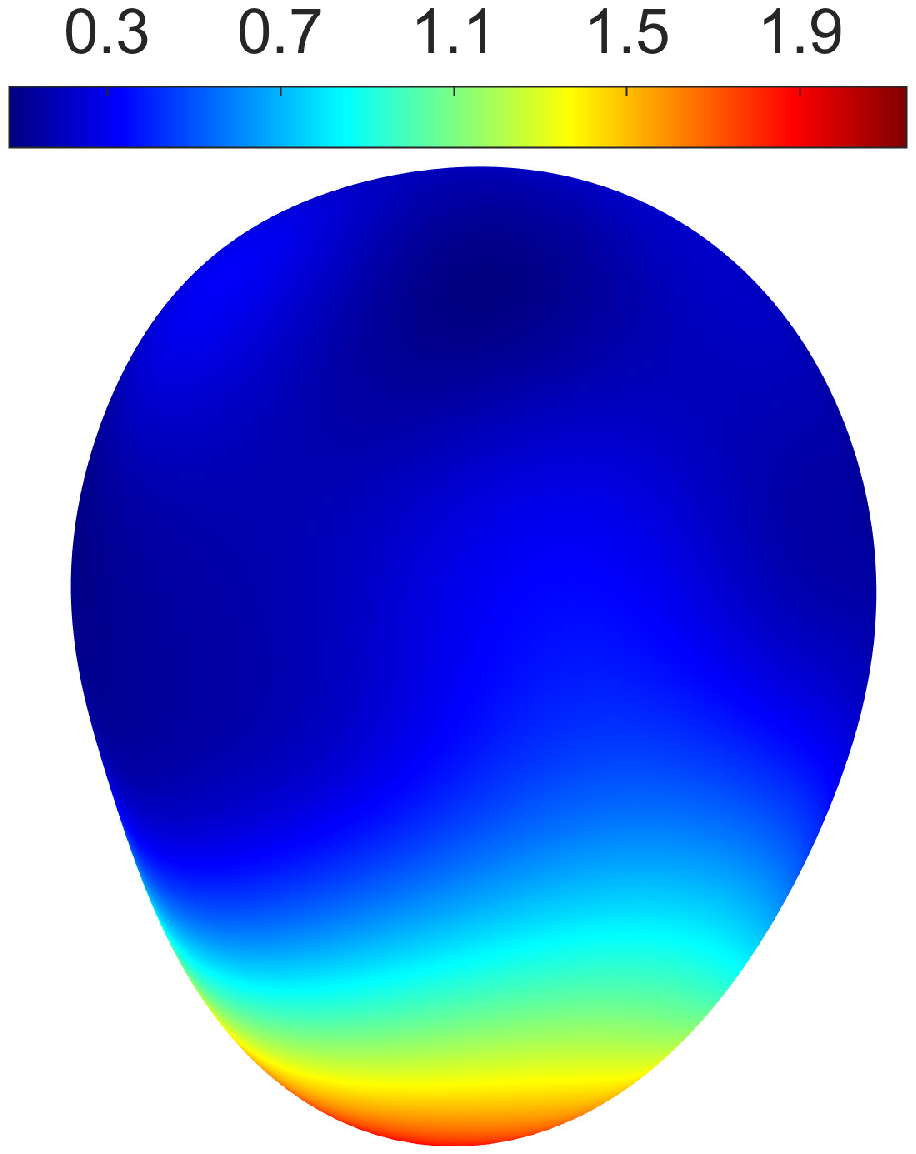}}
    \subfigure[error]{\raisebox{0.05\height}{\includegraphics[width=3.1cm, height=3.5cm]{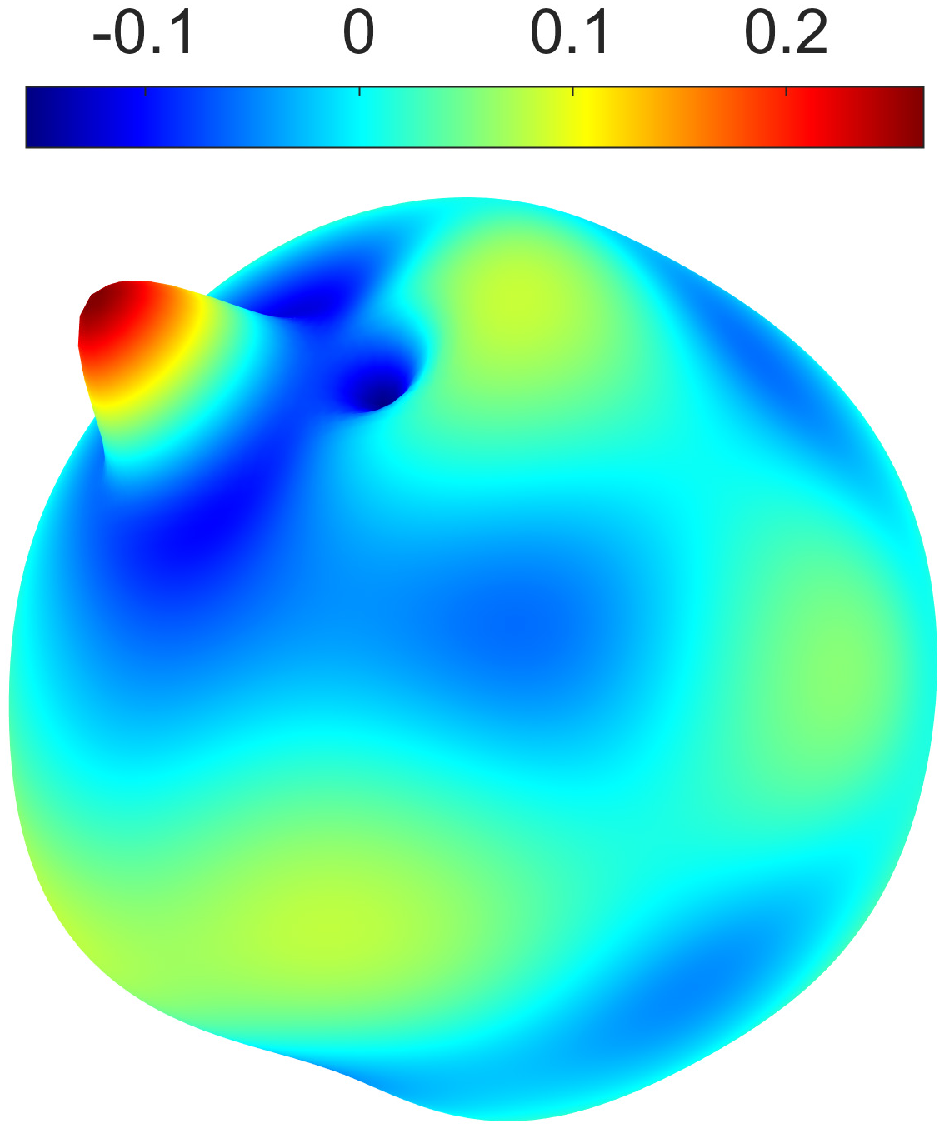}}}
	\caption{Approximation results of $f_1$ over the unit sphere (with noise standard deviation $\delta=0.01$ and $\delta=0.5$) via { sketching} with $s^*$-designs.}\label{FrankeNoiseRecovery}
\end{figure*}

\begin{figure*}[t]
    \centering
    \subfigcapskip=-2pt
    \subfigure{\includegraphics[width=3.1cm, height=3.5cm]{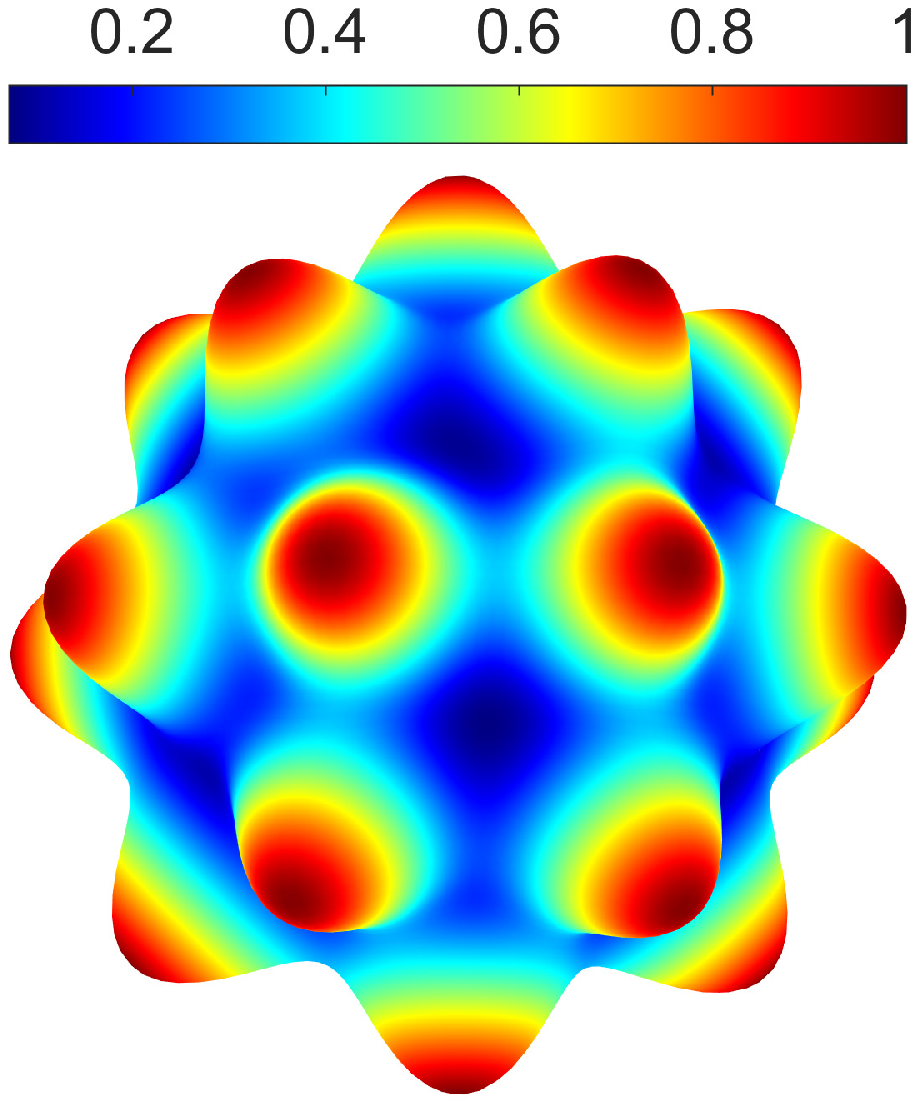}}
    \subfigure{\includegraphics[width=3.1cm, height=3.5cm]{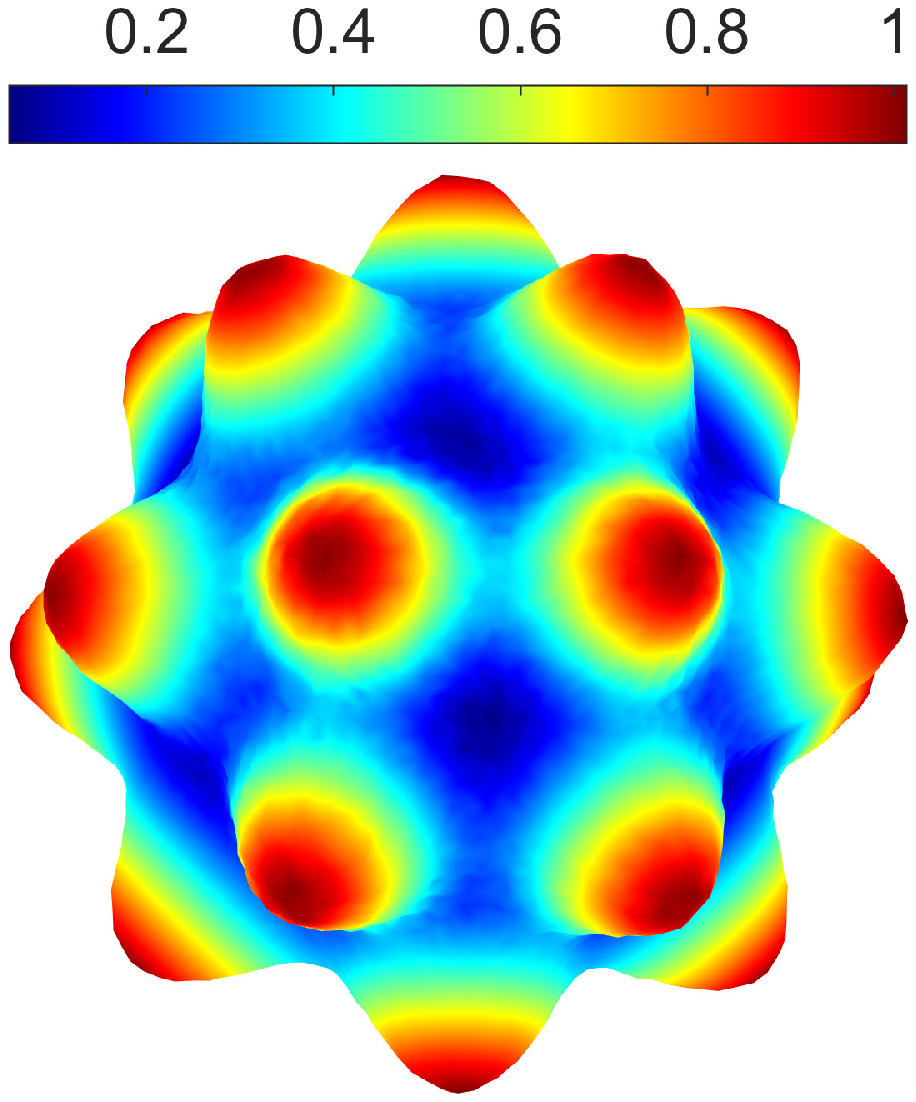}}
    \subfigure{\includegraphics[width=3.1cm, height=3.5cm]{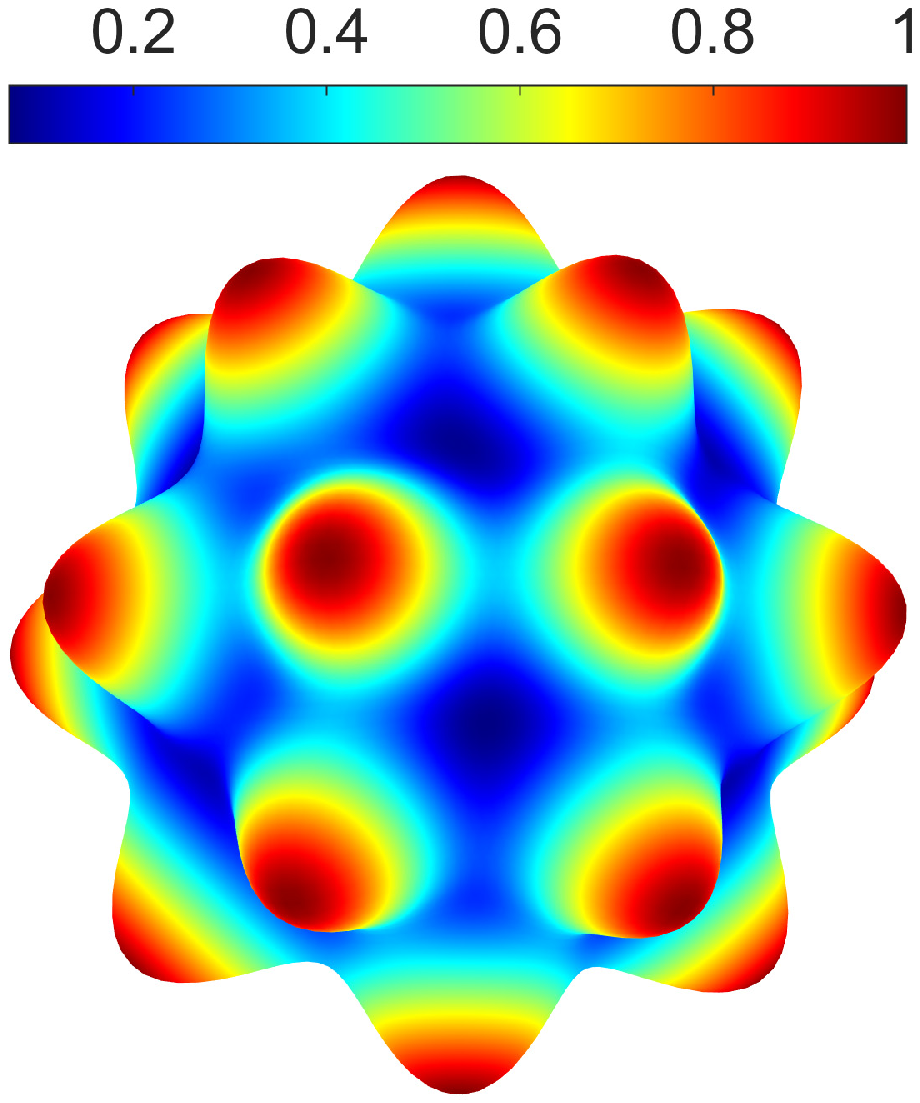}}
    \subfigure{\raisebox{0.05\height}{\includegraphics[width=3.1cm, height=3.31cm]{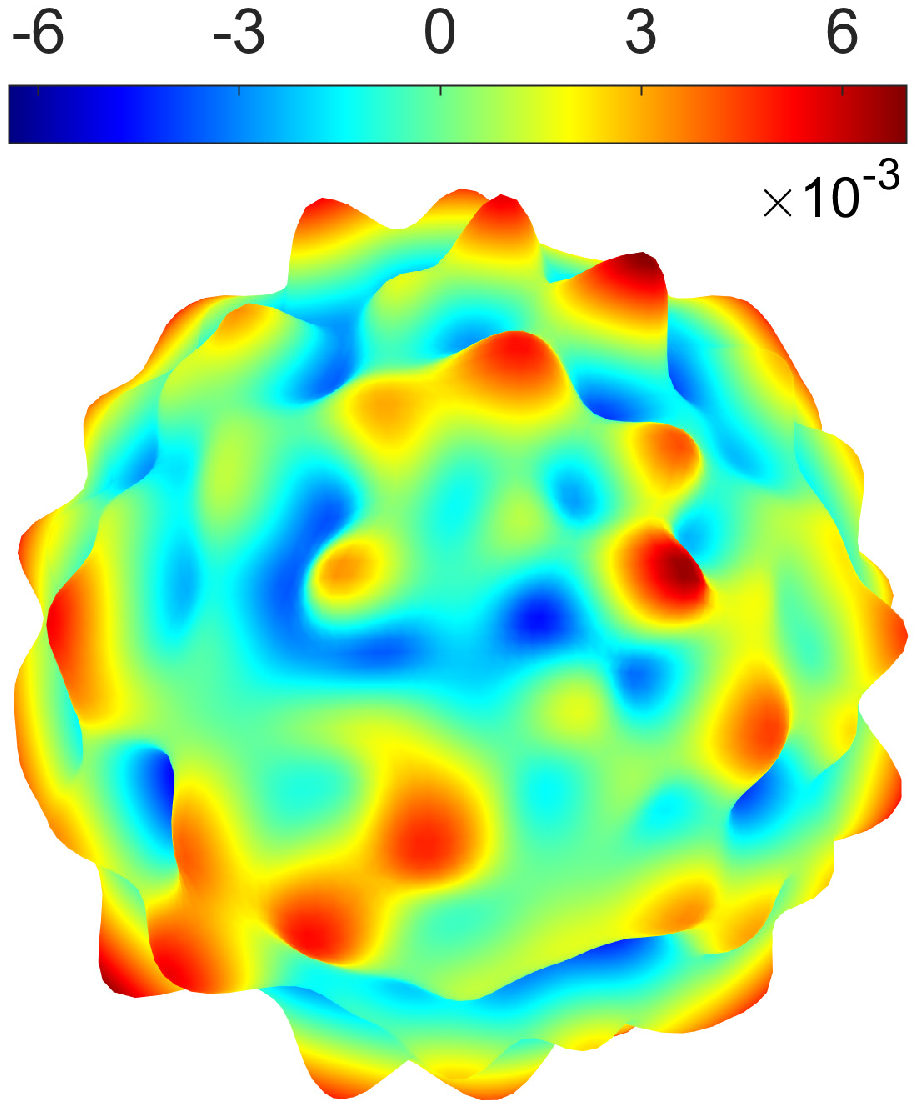}}}\\
    \vspace{-0.08in}
    \setcounter{subfigure}{0}
    \subfigure[exact function]{\includegraphics[width=3.1cm, height=3.5cm]{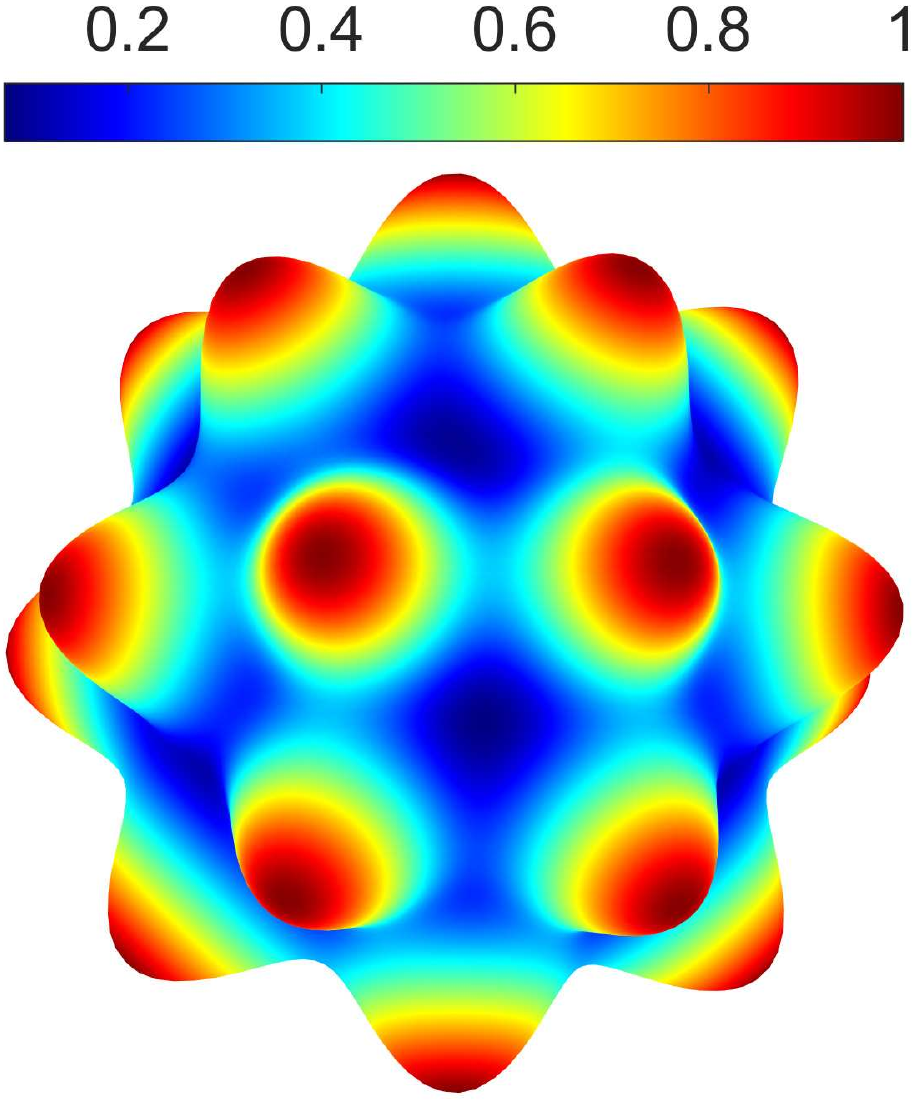}}
    \subfigure[noisy function]{\includegraphics[width=3.1cm, height=3.5cm]{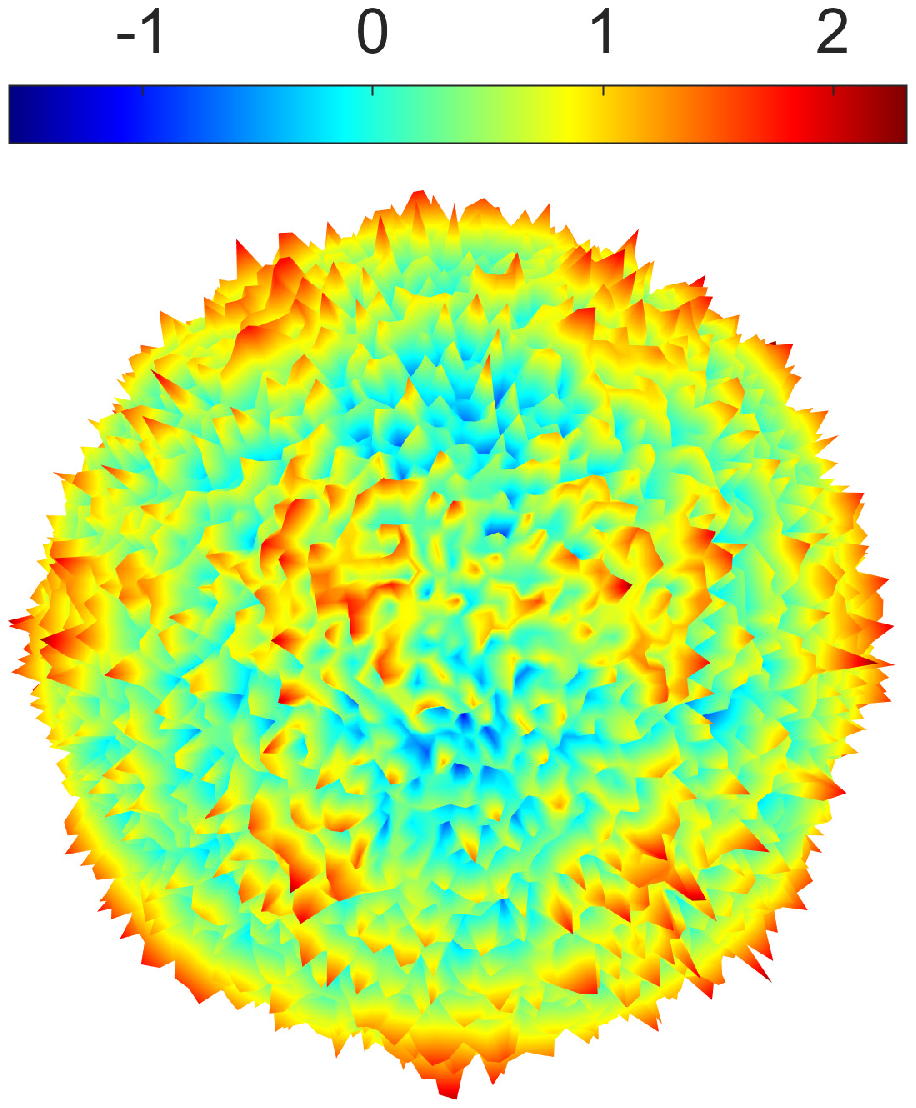}}
    \subfigure[recovery]{\includegraphics[width=3.1cm, height=3.5cm]{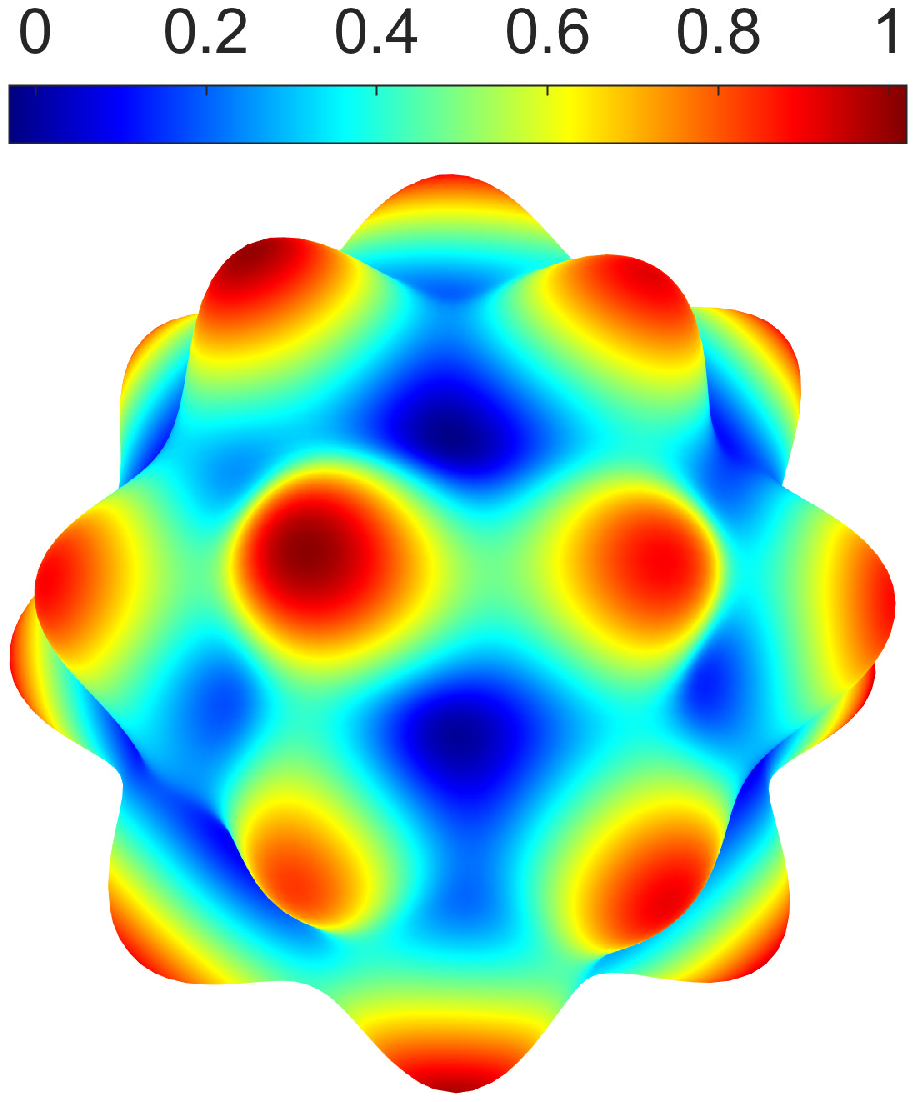}}
    \subfigure[error]{\raisebox{0.05\height}{\includegraphics[width=3.1cm, height=3.31cm]{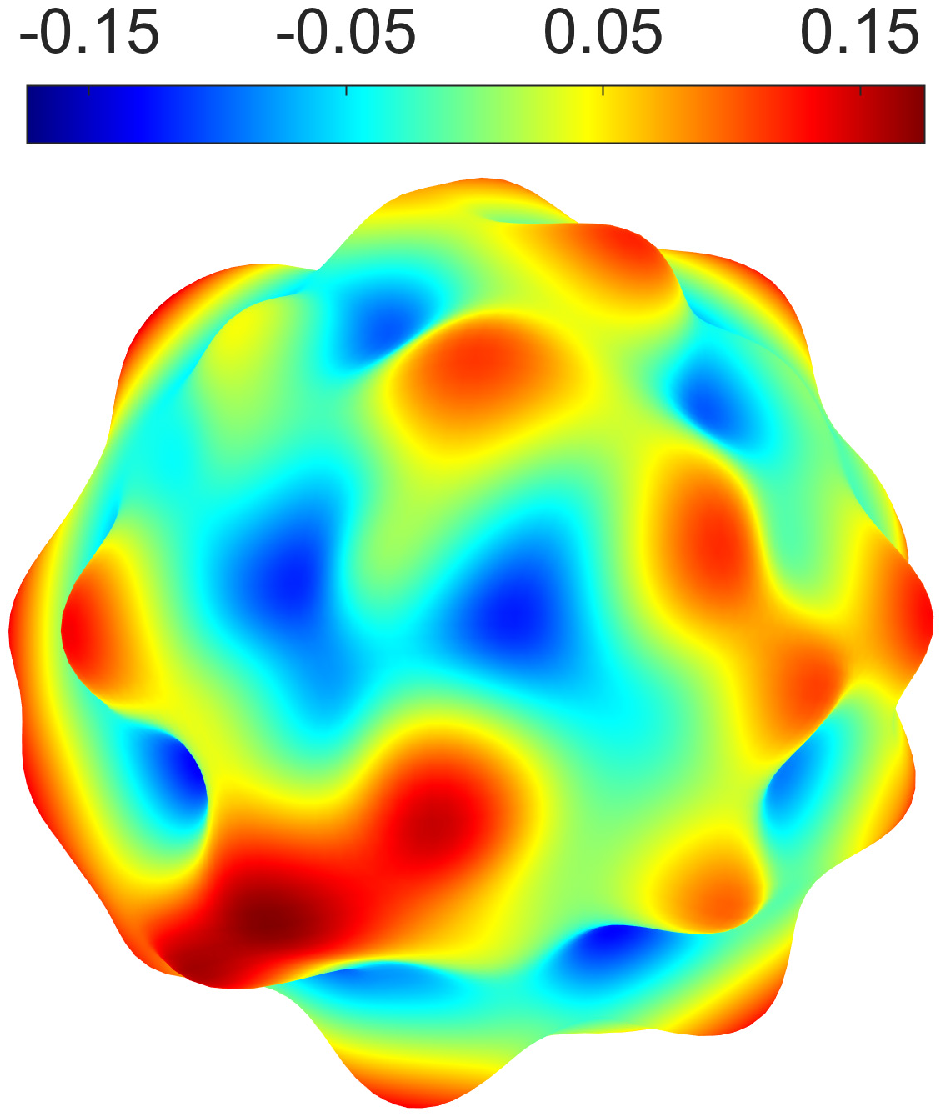}}}
	\caption{Approximation results of $f_2$ over the unit sphere (with noise standard deviation $\delta=0.01$ and $\delta=0.5$) via { sketching} with $s^*$-designs. }\label{FNoiseRecovery}
\end{figure*}

\textbf{Simulation 2: } This simulation compares $s^*$-designs with the other two aforementioned { sketching} methods by varying standard deviations $\delta\in\{0, 10^{-4}, 10^{-3},$ $10^{-2}, 0.1, 0.3, 0.5\}$ of { truncated Gaussian} noise for fixed numbers $s^*\in\{9, 25, 41, 57\}$. To make a fair comparison, the sizes of { sketching} sets for the three methods are the same. For example, if we select $25$-designs for { sketching}, i.e., the number of points is $328$, then we also select $328$ samples from the training set for the other two methods.
The relation between RMSE and the standard deviation $\delta$ for each fixed number $s^*$ is shown in Figure \ref{RMSEComparison}.
From these results, we have the following observations.
1) { Sketching} with selecting the first $m$ samples from the training set has the worst performance for all values of $\delta$ and $s^*$.
2) { Sketching} with randomly selected training samples achieves a noticeable gain in the approximation of the testing functions, 
because the randomly selected $m$ training samples have more opportunities to construct the overall outline of the testing function than the first $m$ training samples when SR is small. 3) Although the performance of randomly selected training samples is excellent, it can still be significantly improved by using $s^*$-designs, especially for small SRs. This is the main focus of this paper, which provides a { sketching} method with a lower SR while maintaining a comparable performance to the standard regularized least squares (\ref{KRR}).

\textbf{Simulation 3: } In this simulation, some approximation results by using $s^*$-designs are visualized in Figures \ref{FrankeNoiseRecovery} and \ref{FNoiseRecovery}, where the top and bottom rows are the results of the training data with noise standard deviations
$\delta=0.01$ and $\delta=0.5$, respectively. The sub-figures (a) show the pictures of exact functions, and the sub-figures (b) show the noisy functions (i.e., $f+\varepsilon$). The sub-figures (c) show the approximation of { sketching} with $s^*$-designs, where the number $s^*=33$ in the case of $\delta=0.01$ and $s^*=9$ in the case of $\delta=0.5$ for the function $f_1$, and $s^*=29$ in the case of $\delta=0.01$ and $s^*=19$ in the case of $\delta=0.5$ for the function $f_2$. The sub-figures (d) are the visualizations of the corresponding approximation error. These results also verify that { sketching} with $s^*$-designs can achieve satisfactory accuracy for the approximation of testing functions when SR is relatively small.

\section{Proofs}\label{Sec.Proof}

In this section, we apply  the idea of integral operator  from \cite{rudi2015less,lin2017distributed,lu2019analysis,Feng2021radial,sun2021nystr} and \cite{tsybakov2009introduction,blanchard2018optimal,fischer2020sobolev}
to prove the upper  and lower bounds, respectively.

\subsection{Preliminaries  of positive operators}

Before presenting the error analysis,  we recall  some basic definitions and properties of
positive operators  established  in \cite{bhatia2013matrix}. Let
$\mathcal H_1$ and $\mathcal H_2$ be two Hilbert spaces and $\mathcal L(\mathcal H_1,\mathcal H_2)$ be the space of all
bounded linear operators from $\mathcal H_1$ to $\mathcal H_2$. 
For $A\in\mathcal L(\mathcal H_1,\mathcal H_2)$, the operator norm is defined  by
$$
        \|A\|_{\mathcal H_1\rightarrow\mathcal H_2}=\sup_{\|f\|_{\mathcal H_1}\leq 1}\|Af\|_{\mathcal H_2}.
$$
If $f\in\mathcal H_1$, there naturally holds
\begin{equation}\label{norm-relation-111}
      \|Af\|_{\mathcal H_2}\leq \|A\|_{\mathcal H_1\rightarrow\mathcal H_2}\|f\|_{\mathcal H_1}.
\end{equation}
Let $A^T$ be the adjoint operator of $A$. We have
\begin{equation}\label{norm-relation-222}
   \|A\|_{\mathcal H_1\rightarrow\mathcal H_2}
   =
   \|A^T\|_{\mathcal H_2\rightarrow\mathcal H_1}
   =\|(A^TA)^{1/2}\|_{\mathcal H_1\rightarrow\mathcal H_1}
   =\|A^TA\|^{1/2}_{\mathcal H_1\rightarrow\mathcal H_1}.
\end{equation}
If $A$ is compact, then it can be found in \cite[eqs(2.43)]{engl1996regularization} that for any piecewise continuous function $\nu(\cdot):[0,\infty)\rightarrow[0,\infty)$, there holds
\begin{equation}\label{exchange-adjoint}
    \nu(AA^T)A=A\nu(A^TA).
\end{equation}


For $A\in\mathcal L(\mathcal H,\mathcal H)$, if $A=A^T$,  then $A$ is called  
self-adjoint. A self-adjoint operator is said to be positive, if $\langle f,Af\rangle_{\mathcal H}\geq 0$ for all $f\in\mathcal H$.  
For positive operators $A$ and $B$, we use the notation $A\preceq B$ to mean  that $B-A$ is positive. 
For $\mu:[0,\infty)\rightarrow[0,\infty)$, define   $\mu(A)$ by spectral calculus. 
We say that $\mu$ is operator monotone if 
$$
         A\preceq B \Rightarrow \mu(A)\preceq \mu(B).
$$
Operator monotonicity plays an important role in analyzing  differences of operators. In particular,  the following lemma can be found in \cite[Theorem X.1.1]{bhatia2013matrix}.

\begin{lemma}\label{Lemma:operator-mono1}
Let $\mu(\cdot)$ be an operator monotone function on $[0,\infty)$ such that $\mu(0)=0$. Then for any positive operators $A$ and $B$ defined on $\mathcal H$, there holds
$$
    \|\mu(A)-\mu(B)\|_{\mathcal H\rightarrow\mathcal H}
    \leq \mu(\|A-B\|_{\mathcal H\rightarrow\mathcal H}).
$$
\end{lemma}
There are numerous operator monotone functions presented in \cite[Chapter V]{bhatia2013matrix}, among which the monomial $t^r$ is the most widely used, just as the following lemma \cite[Theorem V.1.9]{bhatia2013matrix} shows.

\begin{lemma}\label{Lemma:operator-mono2}
The function $\mu(t)=t^r$ is operator monotone on $[0,\infty)$ for $0\leq r\leq 1$.
\end{lemma}
Based on Lemma \ref{Lemma:operator-mono1} and Lemma \ref{Lemma:operator-mono2}, we get the following lemma directly.
\begin{lemma}\label{Lemma:operator-mono3}
Let $0\leq r\leq 1$. Then for any positive operators $A$ and $B$ defined on $\mathcal H$, there holds
$$
    \|A^r-B^r\|_{\mathcal H\rightarrow\mathcal H}
    \leq  \|A-B\|_{\mathcal H\rightarrow\mathcal H}^r.
$$
\end{lemma}

Besides the operator monotonicity, we also need two other lemmas concerning positive operators. The first one is 
the well known Cordes inequality that can  be found in 
  \cite[Lemma VII.5.5]{bhatia2013matrix}.

\begin{lemma}\label{Lemma:Codes-inequality}
Let $A$ and $B$ be positive operators on $\mathcal H$. Then for any $0\leq r\leq 1$, there holds
$$
     \|A^r B^r\|_{\mathcal H\rightarrow\mathcal H}\leq \|AB\|^r_{\mathcal H\rightarrow\mathcal H}.
$$
\end{lemma}

The second one, proved in  \cite[Proposition 6]{rudi2015less}, shows an important property of the projection operator.

\begin{lemma}\label{Lemma:Projection general}
Let $\mathcal H$, $\mathcal K$, and $\mathcal F$ be three separable
Hilbert spaces. Let $Z:\mathcal H\rightarrow\mathcal K$ be a bounded
linear operator and $P$ be a projection operator on $\mathcal H$
such that $\mbox{range}P=\overline{\mbox{range}Z^T}$. Then for any
bounded linear operator $F:\mathcal F\rightarrow\mathcal H$ and any
$\lambda>0$, we have
$$
     \|(I-P)F\|_{\mathcal F\rightarrow\mathcal H}\leq\lambda^{1/2}\|(Z^TZ+\lambda I)^{-1/2}F\|_{\mathcal F\rightarrow\mathcal H}.
$$
\end{lemma}

\subsection{Operator representation of { sketching}}
Different from the classical approach in \cite{narcowich2002scattered,narcowich2007direct,hesse2017radial}, our analysis is based on the operator theory presented in the previous subsection. For this purpose, we should at first give an operator representation for (\ref{Nystrom}).
For  any  $f\in \mathcal N_\phi$, define the sampling operator \cite{smale2004shannon,smale2005shannon} by $S_Df:=(f(x_1),\dots,f(x_{|D|}))^T$, then its  scaled adjoint is 
$
    S^T_{D}{\bf c}:= \frac{1}{|D|}\sum_{i=1}^{|D|}c_i\phi_{x_i}.
$ 
Write
\begin{equation}\label{def.empirical-operator}
     L_{\phi,D}f:=S^T_{D} S_{D}f= \frac1{|D|}\sum_{i=1}^{|D|}f(x_i)\phi_{x_i},\qquad f\in\mathcal N_\phi.
\end{equation}
Then $L_{\phi,D}:\mathcal N_\phi\rightarrow\mathcal N_\phi$ is a positive operator of  rank $|D|$.
Throughout the proof, we denote   $D_m=\mathcal T_{s^*}$  
and $\mathcal H_{D_m}=\mathcal H_{\mathcal T_{s^*},\phi}$ for the sake of brevity.
Let   $P_{D_m}$ be the projection from $\mathcal N_\phi$ to $\mathcal
H_{D_m}$. It is easy to derive that
\begin{equation}\label{projection-p-1}
        (I-P_{D_m})^v=I-P_{D_m},\qquad \forall v\in\mathbb N.
\end{equation}

The following lemma proved in \cite{rudi2015less,lu2019analysis}  presents an operator representation of $f_{D,s^*,\lambda}$.

\begin{lemma}\label{Lemma:operator-representation}
Let $f_{D,s^*,\lambda}$ be defined by (\ref{Nystrom}), then we have
\begin{equation}\label{Nystrom operator}
       f_{D,s^*,\lambda}=g_{D_m,\lambda}(L_{\phi,D})S_D^Ty_D,
\end{equation}
where for a positive operator $A$,
\begin{equation}\label{spetral definetion}
     g_{D_m,\lambda}(A)
     :=
     (P_{D_m}AP_{D_m}+\lambda I)^{-1}P_{D_m} .
\end{equation}
\end{lemma}

Since $g_{D_m,\lambda}$ plays a crucial rule in our algorithm, we need the following important   property of $g_{D_m,\lambda}$ which was proved in \cite{sun2021nystr}.

\begin{lemma}\label{Lemma:property-g}
Let $g_{D_m,\lambda}$ be defined by (\ref{spetral definetion}). We have
\begin{equation}\label{Important2}
    \|(L_{\phi,D}+\lambda
      I)^{1/2}g_{D_m,\lambda}(L_{\phi,D})(L_{\phi,D}+\lambda
      I)^{1/2}\|_{\phi\rightarrow\phi}\leq 1,
\end{equation}
where $\|\cdot\|_{\phi\rightarrow\phi}$ denotes the operator norm from $\mathcal N_\phi$ to $\mathcal N_\phi$.
\end{lemma}

We then introduce a data-free limit of $f_{D_m,s^*,\lambda}$.  It is easy to derive \cite{lu2019analysis} that the adjoint operator of the Canonical inclusion $J_{\phi,\psi}:\mathcal N_\phi\rightarrow\mathcal N_\psi$
  satisfies 
$$
     J_{\phi,\psi}^T f(x)=\int_{\mathbb S^d} \phi(x,x')f(x')d\omega(x'),\qquad f\in \mathcal N_\psi. 
$$
Define further 
$L_{\phi,\psi}=J_{\phi,\psi} J_{\phi,\psi}^T$ and $\mathcal L_\phi=J_{\phi,\psi}^TJ_{\phi,\psi}$, then $L_{\phi,\psi}$ is an integral operator from $\mathcal N_\psi\rightarrow \mathcal N_\psi$ and $\mathcal L_{\phi}$ is an integral operator from $\mathcal N_\phi\rightarrow\mathcal N_\phi$.
In this way, $\mathcal L_\phi$ and $J_{\phi,\psi}^Tf^*$ can be regarded as     data-free limits of $L_{\phi,D}$ and $S_D^Ty_D$, respectively, and therefore 
\begin{equation}\label{population-version-1}
      f^\diamond_{D_m,\lambda}=  J_{\phi,\psi} g_{D_m,\lambda}(\mathcal L_{\phi}) J_{\phi,\psi}^T f^* 
\end{equation}
can be regarded as a data-free limit of $J_{\phi,\psi} f_{D,s^*,\lambda}$.

Since $ 
    \langle f,\sqrt{\hat{\psi}_k}Y_{k,\ell}\rangle_\psi
    =\hat{f}_{k,\ell}/ \sqrt{\hat{\psi}_k},
$ 
it follows from the well known Funk-Hecke formula \cite{muller2006spherical}
\begin{equation}\label{funkhecke}
                  \int_{\mathbb{S}^{d}}\phi(x\cdot x')Y_{k,\ell}(x')d\omega(x')  =\hat{\phi}_kY_{k,\ell}(x),\quad\forall\ \ell=1,\dots,Z(d,k), ~ k=0,1,\dots
\end{equation}
that  for any $f\in\mathcal N_\psi$, there holds
\begin{eqnarray*}
   L_{\phi,\psi} f (x)
   &=& \sum_{k=0}^\infty  \sum_{\ell=1}^{Z(d,k)}\langle f,\sqrt{\hat{\psi}_k}Y_{k,\ell}\rangle_\psi \int_{\mathbb{S}^{d}}\phi(x\cdot x')\sqrt{\hat{\psi}_k}Y_{k,\ell}(x')d\omega(x') \\
   &=&
   \sum_{k=0}^\infty  \hat{\phi}_k \sum_{\ell=1}^{Z(d,k)}\langle f,\sqrt{\hat{\psi}_k}Y_{k,\ell}\rangle_\psi \sqrt{\hat{\psi}_k}Y_{k,\ell}(x)
   =
   \sum_{k=0}^\infty  \hat{\phi}_k \sum_{\ell=1}^{Z(d,k)}\hat{f}_{k,\ell} Y_{k,\ell}(x).
\end{eqnarray*}
We then present the following   operator representation of $f^*\in\mathcal N_\varphi$.

\begin{lemma}\label{Lemma:source-condition}
If $f^*\in\mathcal N_\varphi$, (\ref{kernel-relation}) and (\ref{kernel-relation-1}) hold, then there exists an $h^*\in\mathcal N_\psi$ such that 
\begin{equation}\label{source-condition}
    f^*=L_{\phi,\psi}^{\frac{\alpha-\beta}2}h^*,\qquad  \mbox{and} \quad
    \|f^*\|_\varphi=\|h^*\|_\psi,
\end{equation}
where $L_{\phi,\psi}^r$ for $r\geq 0$ is defined by spectral calculus, i.e.,
$$
    L_{\phi,\psi}^rf=\sum_{k=0}^\infty  \hat{\phi}_k^r \sum_{\ell=1}^{Z(d,k)}\langle f,\sqrt{\hat{\psi}_k}Y_{k,\ell}\rangle_\psi \sqrt{\hat{\psi}_k}Y_{k,\ell}.
$$    
\end{lemma}

\begin{proof} Due to (\ref{assumption-on-phi}), we can define 
$$
    h^*= \sum_{k=0}^\infty  \hat{\phi}_k^{\frac{\beta-\alpha}2} \sum_{\ell=1}^{Z(d,k)}\hat{(f^*)}_{k,\ell}Y_{k,\ell},
$$
then \eqref{source-condition} obviously holds. The only thing left is to prove $h^*\in \mathcal N_\psi$. Due to the definition of $\|\cdot\|_\psi$ and
$$
   \hat{(h^*)}_{k,\ell}=\int_{\mathbb S^d}
   \sum_{k'=0}^\infty  \hat{\phi}_{k'}^{\frac{\beta-\alpha}2} \sum_{\ell'=1}^{Z(d,k')}\hat{(f^*)}_{k',\ell'}Y_{k',\ell'}(x)Y_{k,\ell}(x)d\omega(x)=  \hat{\phi}_k^{\frac{\beta-\alpha}2}  \hat{(f^*)}_{k,\ell},
$$
we have from (\ref{kernel-relation}) and (\ref{kernel-relation-1}) that
\begin{eqnarray*}
    \|h^*\|_{\psi}^2
   &=&
    \sum_{k=0}^\infty\hat{\psi}_k^{-1}\sum_{\ell=1}^{Z(d,k)}|\hat{(h^*)}_{k,\ell}|^2 
     = 
    \sum_{k=0}^\infty\hat{\psi}_k^{-1}\sum_{\ell=1}^{Z(d,k)}
    \hat{\phi}_k^{\beta-\alpha} |\hat{(f^*)}_{k,\ell}|^2\\
    &=&
    \sum_{k=0}^\infty \hat{\phi}_k^{-\alpha} \sum_{\ell=1}^{Z(d,k)}
    |\hat{(f^*)}_{k,\ell}|^2=
    \sum_{k=0}^\infty \hat{\varphi}_k^{-1} \sum_{\ell=1}^{Z(d,k)}
    |\hat{(f^*)}_{k,\ell}|^2.
\end{eqnarray*}
Noting further that  $f^*\in\mathcal N_\varphi$ implies
$$
   \|f^*\|_\varphi^2=\sum_{k=0}^\infty\hat{\varphi}_k^{-1}
               \sum_{\ell=1}^{Z(d,k)}|\hat{f}_{k,\ell}|^2<\infty,
$$
we then have 
$
\|h^*\|_{\psi}=\|f^*\|_\varphi<\infty.$ This completes the proof of Lemma \ref{Lemma:source-condition}.
\end{proof}

\subsection{Error decomposition via integral operator}

Our analysis is  motivated by \cite{lu2019analysis}  to 
  divide  the fitting error into approximation error, sample error and computational error, respectively.

\begin{proposition}\label{Proposition:Error-decomposition-out}
Let $f_{D,s^*,\lambda}$ be defined by (\ref{Nystrom}). Under  (\ref{kernel-relation}), if $f^*\in\mathcal N_\varphi$ with (\ref{kernel-relation-1}),  then  
\begin{eqnarray}\label{Error decomposition}
   \|J_{\phi,\psi}f_{D,s^*,\lambda}-f^*\|_\psi
   \leq
  \mathcal A(D,\lambda,m)+\mathcal S(D,\lambda,m)+\mathcal C(D,\lambda,m),
\end{eqnarray}
where
\begin{eqnarray*}
   \mathcal A(D,\lambda,m)&:=& \lambda\|(J_{\phi,\psi}P_{D_m}J_{\phi,\psi}^T+\lambda I)^{-1}(J_{\phi,\psi}P_{D_m}J^T_{\phi,\psi})^{\frac{\alpha-\beta}{2}}h^*\|_\psi, \label{Def-A*} \\
    \mathcal S(D,\lambda,m)&:=&\mathcal S_1(D,\lambda,m)+\mathcal S_2(D,\lambda,m)+\mathcal S_3(D,\lambda,m),\\
   \mathcal S_1(D,\lambda,m)&:=&\|J_{\phi,\psi}(P_{D_m}  L_{\phi,D} P_{D_m}+\lambda I)^{-1}P_{D_m}(J_{\phi,\psi}^Tf^*-S_D^Ty_D)\|_\psi, \label{Def-S1*}\\
   \mathcal S_2(D,\lambda,m) &:=&
   \|J_{\phi,\psi}(P_{D_m}  L_{\phi,D} P_{D_m}+\lambda I)^{-1}
         P_{D_m}(L_{\phi,D}-\mathcal L_\phi)P_{D_m}J_{\phi,\psi}^T
         \label{Def-S2*}\\
         &&
         (J_{\phi,\psi}P_{D_m}J_{\phi,\psi}^T+\lambda I)^{-1} 
          (J_{\phi,\psi}P_{D_m}J^T_{\phi,\psi})^{\frac{\alpha-\beta}{2}}h^*\|_\psi, \nonumber\\
     \mathcal S_3(D,\lambda,m) &:=&
   \|J_{\phi,\psi}(P_{D_m}  L_{\phi,D} P_{D_m}+\lambda I)^{-1}
         P_{D_m}(L_{\phi,D}-\mathcal L_\phi)P_{D_m}J_{\phi,\psi}^T\\ &&(J_{\phi,\psi}P_{D_m}J_{\phi,\psi}^T+\lambda I)^{-1} 
          \left((J_{\phi,\psi}J^T_{\phi,\psi})^{\frac{\alpha-\beta}{2}}-(J_{\phi,\psi}P_{D_m}J^T_{\phi,\psi})^{\frac{\alpha-\beta}{2}}\right)h^*\|_\psi, \nonumber\\ 
   \mathcal C(D,\lambda,m)&:=&\lambda\|(J_{\phi,\psi}P_{D_m}J_{\phi,\psi}^T+\lambda I)^{-1} 
    ((J_{\phi,\psi}J^T_{\phi,\psi})^{\frac{\alpha-\beta}{2}}-(J_{\phi,\psi}P_{D_m}J^T_{\phi,\psi})^{\frac{\alpha-\beta}{2}})h^*\|_\psi, \nonumber
\end{eqnarray*}
and  $h^*$ is given in Lemma \ref{Lemma:source-condition}.
\end{proposition}

\begin{proof} 
Due to the triangle inequality, we get 
\begin{equation}\label{Error-app-sam}
  \|J_{\phi,\psi}f_{D,s^*,\lambda}-f^*\|_\psi\leq\|f^\diamond_{D_m,\lambda}-f^*\|_\psi+\|f^\diamond_{D_m,\lambda}-J_{\phi,\psi}f_{D,s^*,\lambda}\|_\psi.
\end{equation}
It follows from  $P_{D_m}^2=P_{D_m}$ and \eqref{exchange-adjoint}  with $A=P_{D_m}J^T_{\phi,\psi} $ that
$$
    (P_{D_m}J_{\phi,\psi}^TJ_{\phi,\psi} P_{D_m}+\lambda I)^{-1}P_{D_m}J_{\phi,\psi}^T=P_{D_m}J_{\phi,\psi}^T(J_{\phi,\psi} P_{D_m}J_{\phi,\psi}^T+\lambda I)^{-1}.
$$
Then    
  (\ref{spetral definetion}) and (\ref{population-version-1})  yield
\begin{eqnarray*} 
   &&f^\diamond_{D_m,\lambda}-f^*=(J_{\phi,\psi} g_{D_m,\lambda}(\mathcal L_{\phi}) J_{\phi,\psi}^T-I) f^*\\
 & =&
  (J_{\phi,\psi}(P_{D_m}J_{\phi,\psi}^TJ_{\phi,\psi}P_{D_m}+\lambda I)^{-1}P_{D_m}J_{\phi,\psi}^T-I) f^* \nonumber \\
  &=&
  (J_{\phi,\psi}P_{D_m}J_{\phi,\psi}^T(J_{\phi,\psi}P_{D_m}J_{\phi,\psi}^T+\lambda I)^{-1}-I)f^*  \nonumber\\
  &=&
  \lambda(J_{\phi,\psi}P_{D_m}J_{\phi,\psi}^T+\lambda I)^{-1}f^*. \nonumber
\end{eqnarray*}
Therefore, it follows from Lemma \ref{Lemma:source-condition} and $L_{\phi,\psi}=J_{\phi,\psi}J^T_{\phi,\psi}$ that
\begin{eqnarray}\label{app.1.1}
  && \|f^\diamond_{D_m,\lambda}-f^*\|_\psi
  = \lambda\|(J_{\phi,\psi}P_{D_m}J_{\phi,\psi}^T+\lambda I)^{-1}(J_{\phi,\psi}J^T_{\phi,\psi})^{\frac{\alpha-\beta}{2}}h^*\|_\psi  \\
  &\leq&
  \lambda\|(J_{\phi,\psi}P_{D_m}J_{\phi,\psi}^T+\lambda I)^{-1}(J_{\phi,\psi}P_{D_m}J^T_{\phi,\psi})^{\frac{\alpha-\beta}{2}}h^*\|_\psi\nonumber\\
  &+&\lambda\|(J_{\phi,\psi}P_{D_m}J_{\phi,\psi}^T+\lambda I)^{-1}\left((J_{\phi,\psi}J^T_{\phi,\psi})^{\frac{\alpha-\beta}{2}}-(J_{\phi,\psi}P_{D_m}J^T_{\phi,\psi})^{\frac{\alpha-\beta}{2}}\right)h^*\|_\psi. \nonumber
\end{eqnarray}
Noting further Lemma \ref{Lemma:operator-representation} and (\ref{population-version-1}), we obtain
\begin{eqnarray*}
   &&f^\diamond_{D_m,\lambda}-J_{\phi,\psi}f_{D,s^*,\lambda}
   =
   J_{\phi,\psi}\left((g_{D_m,\lambda}(\mathcal L_\phi)J_{\phi,\psi}^Tf^*
   -(g_{D_m,\lambda}(  L_{\phi,D})S_D^Ty_D\right)\\
   &=&
  J_{\phi,\psi}\left( (P_{D_m}\mathcal L_\phi P_{D_m}+\lambda I)^{-1}P_{D_m}J_{\phi,\psi}^Tf^*-  (P_{D_m}  L_{\phi,D} P_{D_m}+\lambda I)^{-1}P_{D_m}S_D^Ty_D
  \right).
\end{eqnarray*}
For positive operators $A$ and $B$, since $A^{-1}-B^{-1}=B^{-1}(B-A)A^{-1}$, the above estimates yield 
\begin{eqnarray*}
   &&f^\diamond_{D_m,\lambda}-J_{\phi,\psi}f_{D,s^*,\lambda}
   =
   J_{\phi,\psi}(P_{D_m}  L_{\phi,D} P_{D_m}+\lambda I)^{-1}P_{D_m}(J_{\phi,\psi}^Tf^*-S_D^Ty_D)\\
   &+&
   J_{\phi,\psi}\left((P_{D_m}\mathcal L_\phi P_{D_m}+\lambda I)^{-1}
      -(P_{D_m}  L_{\phi,D} P_{D_m}+\lambda I)^{-1}\right)P_{D_m}J_{\phi,\psi}^Tf^*\\
      &=&
         J_{\phi,\psi}(P_{D_m}  L_{\phi,D} P_{D_m}+\lambda I)^{-1}P_{D_m}(J_{\phi,\psi}^Tf^*-S_D^Ty_D)\\
         &+&
         J_{\phi,\psi}(P_{D_m}  L_{\phi,D} P_{D_m}+\lambda I)^{-1}
         P_{D_m}(L_{\phi,D}-\mathcal L_\phi)P_{D_m}
          (P_{D_m}\mathcal L_\phi P_{D_m}+\lambda I)^{-1}P_{D_m}J_{\phi,\psi}^Tf^*.
\end{eqnarray*}
Due to Lemma \ref{Lemma:source-condition}, the same step as  that in \eqref{app.1.1} shows
\begin{eqnarray*}
   &&
          (P_{D_m}\mathcal L_\phi P_{D_m}+\lambda I)^{-1}P_{D_m}J_{\phi,\psi}^Tf^*\\
         & =&
         P_{D_m}J_{\phi,\psi}^T (J_{\phi,\psi}P_{D_m}J_{\phi,\psi}^T+\lambda I)^{-1}
         (J_{\phi,\psi}P_{D_m}J^T_{\phi,\psi})^{\frac{\alpha-\beta}{2}}h^*\\
         &+&
         P_{D_m}J_{\phi,\psi}^T (J_{\phi,\psi}P_{D_m}J_{\phi,\psi}^T+\lambda I)^{-1}\left((J_{\phi,\psi}J^T_{\phi,\psi})^{\frac{\alpha-\beta}{2}}-(J_{\phi,\psi}P_{D_m}J^T_{\phi,\psi})^{\frac{\alpha-\beta}{2}}\right)h^*.
\end{eqnarray*}
Combining the above two equations, we then get
\begin{eqnarray*}
  &&\|f^\diamond_{D_m,\lambda}-J_{\phi,\psi}f_{D,s^*,\lambda}\|_\psi
  \leq
   \|J_{\phi,\psi}(P_{D_m}  L_{\phi,D} P_{D_m}+\lambda I)^{-1}P_{D_m}(J_{\phi,\psi}^Tf^*-S_D^Ty_D)\|_\psi\\
   &+&
  \|J_{\phi,\psi}(P_{D_m}  L_{\phi,D} P_{D_m}+\lambda I)^{-1}
         P_{D_m}(L_{\phi,D}-\mathcal L_\phi)P_{D_m}J_{\phi,\psi}^T (J_{\phi,\psi}P_{D_m}J_{\phi,\psi}^T+\lambda I)^{-1}\\
         &&(J_{\phi,\psi}P_{D_m}J^T_{\phi,\psi})^{\frac{\alpha-\beta}{2}}h^*\|_\psi\\
    &+&
    \|J_{\phi,\psi}(P_{D_m}  L_{\phi,D} P_{D_m}+\lambda I)^{-1}
         P_{D_m}(L_{\phi,D}-\mathcal L_\phi)P_{D_m}J_{\phi,\psi}^T (J_{\phi,\psi}P_{D_m}J_{\phi,\psi}^T+\lambda I)^{-1}\\
        &&  \left((J_{\phi,\psi}J^T_{\phi,\psi})^{\frac{\alpha-\beta}{2}}-(J_{\phi,\psi}P_{D_m}J^T_{\phi,\psi})^{\frac{\alpha-\beta}{2}}\right)h^*\|_\psi.
\end{eqnarray*}
This together with (\ref{Error-app-sam}) and (\ref{app.1.1}) completes the proof of Proposition \ref{Proposition:Error-decomposition-out}.
\end{proof}

We call $\mathcal A(D,\lambda,m)$, $\mathcal S(D,\lambda,m)$, and $\mathcal C(D,\lambda,m)$  the approximation error, sample error, and computational error, respectively.  The sample error   reflects the effect of noise and is a quantity to measure the stability of the proposed algorithm; the computational error   aims to quantitatively describe the difference between the identity mapping $I$ and projection operator $P_{D_m}$ and, therefore, is mainly devoted to measuring the quality of our { sketching} strategy; the approximation error focuses on the approximation capability of the { sketching} algorithm on $\mathcal H_{D_m}$.  Based on Proposition \ref{Proposition:Error-decomposition-out}, the estimate of the fitting error can be simplified to bounding  $\mathcal A(D,\lambda,m)$, $\mathcal S(D,\lambda,m)$, and $\mathcal C(D,\lambda,m)$, respectively.

\subsection{Error estimates via operator differences}
The aim of this subsection is to quantify the error via  operator (or function) differences. Define 
\begin{eqnarray}
    \mathcal Q_{D,\lambda}&:=& \|(\mathcal L_\phi+\lambda I) (L_{\phi,D}+\lambda I)^{-1}\|_{\phi\rightarrow\phi},\label{def:Q}\\
       \mathcal P_{D,\lambda}&:=&\|(\mathcal L_{\phi}+\lambda I)^{-1/2}   (S_D^Ty_D-L_{\phi,D}f^*)\|_\phi, \label{def:P}\\
         \mathcal R_{D,\lambda}&:=&\|(\mathcal L_\phi+\lambda I)^{-1/2}
    (L_{\phi,D}-\mathcal L_\phi)\|_{\phi\rightarrow\phi}. \label{def:R}
\end{eqnarray}
The following proposition quantifies the fitting error via the above three quantities.

\begin{proposition}\label{Prop:fitting-via-difference}
If $f^*\in\mathcal N_\varphi$, and (\ref{kernel-relation}) and (\ref{kernel-relation-1}) hold with $0\leq \alpha-\beta\leq 1$, 
then
\begin{eqnarray*} 
   \|f_{D,s^*,\lambda}-f^*\|_\psi
   \leq
  (1 +\mathcal Q_{D,\lambda}\mathcal R_{D,\lambda}\lambda^{-\frac12}) \lambda^{ \frac{\alpha-\beta}{2}}(1+\mathcal Q_{D_m,\lambda}^{\frac{\alpha-\beta}{2}})\|f^*\|_\varphi
  +\mathcal Q_{D,\lambda}\mathcal P_{D,\lambda}. 
\end{eqnarray*}
\end{proposition} 

To prove the above proposition, it suffices to bound the corresponding five terms in Proposition \ref{Proposition:Error-decomposition-out} respectively. For this purpose, we need three auxiliary lemmas. 
\begin{lemma}\label{Lemma:au-1}
Let $F$  be either   a function in $\mathcal N_\phi$ or an operator from $\mathcal N_\phi$ to $\mathcal N_\phi$, then
$$
      \|J_{\phi,\psi}(P_{D_m}  L_{\phi,D} P_{D_m}+\lambda I)^{-1}P_{D_m}F\|_{*}\leq 
      \mathcal Q_{D,\lambda}\|(\mathcal L_K+\lambda I)^{-1/2}F\|_{**},
$$
where $\|F\|_*$ and $\|F\|_{**} $ denote  either $\|F\|_\psi$ and $\|F\|_\phi$  respectively if $F$ is a function
or $\|F\|_{\phi\rightarrow\psi}$ and $\|F\|_{\phi\rightarrow\phi}$ respectively  when $F$ is an operator.
\end{lemma}

\begin{proof}
It follows from (\ref{norm-relation-111}) that
\begin{eqnarray*}
   &&\|J_{\phi,\psi}(P_{D_m}  L_{\phi,D} P_{D_m}+\lambda I)^{-1}P_{D_m}F\|_{*}\\
   &\leq&
   \|J_{\phi,\psi}(\mathcal L_\phi+\lambda I)^{-1/2}\|_{\phi\rightarrow\psi} 
   \|(\mathcal L_\phi+\lambda I)^{1/2}( L_{\phi,D}+\lambda I)^{-1/2}\|_{\phi\rightarrow\phi}\\
   &\times&
   \|( L_{\phi,D}+\lambda I)^{1/2}(P_{D_m}  L_{\phi,D} P_{D_m}+\lambda I)^{-1}P_{D_m}( L_{\phi,D}+\lambda I)^{1/2}\|_{\phi\rightarrow\phi}\\
   &\times& \|( L_{\phi,D}+\lambda I)^{-1/2}(\mathcal  L_{\phi,}+\lambda I)^{1/2}\|_{\phi\rightarrow\phi}
   \|(\mathcal  L_{\phi}+\lambda I)^{-1/2}F\|_{**}.
\end{eqnarray*}  
Then,   (\ref{def:Q}), Lemma \ref{Lemma:property-g} and Lemma \ref{Lemma:Codes-inequality}  yield
\begin{eqnarray*}
     && \|J_{\phi,\psi}(P_{D_m}  L_{\phi,D} P_{D_m}+\lambda I)^{-1}P_{D_m}F\|_{*}\\
     &\leq&
     \|J_{\phi,\psi}(\mathcal L_\phi+\lambda I)^{-1/2}\|_{\phi\rightarrow\psi} \mathcal Q_{D,\lambda} \|(\mathcal L_\phi+\lambda I)^{-1/2}F\|_{**}.
\end{eqnarray*}
  Note further
\begin{eqnarray}\label{norm-relation}
   &&\|J_{\phi,\psi}(\mathcal L_\phi+\lambda I)^{-1/2}\|^2_{\phi\rightarrow\psi}
  =\sup_{\|f\|_\phi\leq 1}\|J_{\phi,\psi}(\mathcal L_\phi+\lambda I)^{-1/2}f\|_\psi^2\\
  &=&
  \sup_{\|f\|_\phi\leq 1}
  \langle J_{\phi,\psi}(\mathcal L_\phi+\lambda I)^{-1/2}f,J_{\phi,\psi}(\mathcal L_\phi+\lambda I)^{-1/2}f\rangle_\psi \nonumber \\
  &=&
  \sup_{\|f\|_\phi\leq 1}
  \langle J_{\phi,\psi}^T J_{\phi,\psi}(\mathcal L_\phi+\lambda I)^{-1/2}f, (\mathcal L_\phi+\lambda I)^{-1/2}f\rangle_\phi \nonumber \\
  &=&
   \sup_{\|f\|_\phi\leq 1}\|\mathcal L_\phi^{1/2}(\mathcal L_\phi+\lambda I)^{-1/2}f\|_\phi^2
   =
   \|\mathcal L_\phi^{1/2}(\mathcal L_\phi+\lambda I)^{-1/2}\|_{\phi\rightarrow\phi}^2\leq 1. \nonumber 
\end{eqnarray}
We have
$$
   \|J_{\phi,\psi}(P_{D_m}  L_{\phi,D} P_{D_m}+\lambda I)^{-1}P_{D_m}F\|_{*}
   \leq
   \mathcal Q_{D,\lambda} \|(\mathcal L_K+\lambda I)^{-1/2}F\|_{**}.
$$
This completes the proof of Lemma \ref{Lemma:au-1}.
\end{proof}

\begin{lemma}\label{Lemma:au-2}
Let $0\leq a\leq 1$. For any $0\leq \alpha-\beta\leq 1$, if $2a\geq \alpha-\beta$, then  
\begin{eqnarray*}
    \|(J_{\phi,\psi}P_{D_m}J_{\phi,\psi}^T+\lambda I)^{-a}(J_{\phi,\psi}P_{D_m}J^T_{\phi,\psi})^{\frac{\alpha-\beta}{2}}h^*\|_\psi 
   \leq  \lambda^{ \frac{\alpha-\beta-2a}{2}}\|h^*\|_\psi.
\end{eqnarray*}
\end{lemma}

\begin{proof}
Since $2a\geq\alpha-\beta$, we have
\begin{eqnarray*}
   &&\| (J_{\phi,\psi}P_{D_m}J_{\phi,\psi}^T+\lambda I)^{-a}(J_{\phi,\psi}P_{D_m}J^T_{\phi,\psi})^{\frac{\alpha-\beta}{2}}h^*\|_\psi\\
   &\leq&
   \| (J_{\phi,\psi}P_{D_m}J_{\phi,\psi}^T+\lambda I)^{-a}(J_{\phi,\psi}P_{D_m}J^T_{\phi,\psi})^{\frac{\alpha-\beta}{2}}\|_{\psi\rightarrow\psi} \|h^*\|_\psi\\
   &\leq&
     \|(J_{\phi,\psi}P_{D_m}J_{\phi,\psi}^T+\lambda I)^{ \frac{\alpha-\beta-2a}{2}}\|_{\psi\rightarrow\psi} \|h^*\|_\psi 
      \leq
      \lambda^{ \frac{\alpha-\beta-2a}{2}}\|h^*\|_\psi.
\end{eqnarray*}
The proof of Lemma \ref{Lemma:au-2} is completed.
\end{proof}

\begin{lemma}\label{Lemma:au-3}
Let $0\leq b\leq 1$. For any   $0\leq \alpha-\beta\leq 1$, there holds
\begin{eqnarray*}
  &&   \|(J_{\phi,\psi}P_{D_m}J_{\phi,\psi}^T+\lambda I)^{-b} 
    ((J_{\phi,\psi}J^T_{\phi,\psi})^{\frac{\alpha-\beta}{2}}-(J_{\phi,\psi}P_{D_m}J^T_{\phi,\psi})^{\frac{\alpha-\beta}{2}})h^*\|_\psi\\
    &\leq&
    \lambda^{\frac{\alpha-\beta-2b}{2}} \mathcal Q_{D_m,\lambda}^{\frac{\alpha-\beta}{2}}\|h^*\|_\psi.
\end{eqnarray*}
\end{lemma}
\begin{proof}
Since $0\leq \alpha-\beta\leq 1$,  we get  from   Lemma \ref{Lemma:operator-mono3}  that
$$
  \|(J_{\phi,\psi}J^T_{\phi,\psi})^{\frac{\alpha-\beta}{2}}-(J_{\phi,\psi}P_{D_m}J^T_{\phi,\psi})^{\frac{\alpha-\beta}{2}}\|_{\psi\rightarrow\psi}
  \leq 
  \|J_{\phi,\psi}(I-P_{D_m})J^T_{\phi,\psi}\|_{\psi\rightarrow\psi}^{\frac{\alpha-\beta}{2}}.
$$
But \eqref{projection-p-1} implies
\begin{eqnarray*}
   &&\|J_{\phi,\psi}(I-P_{D_m})J^T_{\phi,\psi}\|_{\psi\rightarrow\psi}
   =
   \| J_{\phi,\psi} (I-P_{D_m})(J_{\phi,\psi} (I-P_{D_m}))^T \|_{\psi\rightarrow\psi}\\
   &=&
   \|(I-P_{D_m})J_{\phi,\psi}^T\|_{\psi\rightarrow\phi}^2.
\end{eqnarray*}
 Then, it follows from (\ref{norm-relation}), (\ref{norm-relation-222}), Lemma \ref{Lemma:Codes-inequality}   and Lemma \ref{Lemma:Projection general}  with $\mathcal H=\mathcal N_\phi$, $\mathcal K=\mathcal F=\mathcal N_\psi$, and $Z=S_{D_m}$ that
\begin{eqnarray*}
   &&\|(I-P_{D_m})J_{\phi,\psi}^T\|_{\psi\rightarrow\phi}
   \leq
   \lambda^{1/2}\|(  L_{\phi,D_m}+\lambda I)^{-1/2}J^T_{\phi,\psi}\|_{\psi\rightarrow\phi}\\
   &=&
   \lambda^{1/2}\|J_{\phi,\psi}( L_{\phi,D_m}+\lambda I)^{-1/2}\|_{\phi\rightarrow\psi} 
    \leq 
   \lambda^{1/2} \| \mathcal L_\phi^{1/2} ( L_{\phi,D_m}+\lambda I)^{-1/2}\|_{\phi\rightarrow\phi}\\
   &\leq&
   \lambda^{1/2}\mathcal Q_{D_m,\lambda}^{1/2}. 
\end{eqnarray*}
Combining the above three estimates, we get
$$
 \|(J_{\phi,\psi}J^T_{\phi,\psi})^{\frac{\alpha-\beta}{2}}-(J_{\phi,\psi}P_{D_m}J^T_{\phi,\psi})^{\frac{\alpha-\beta}{2}}\|_{\psi\rightarrow\psi}
  \leq 
  \lambda^{\frac{\alpha-\beta}{2}} \mathcal Q_{D_m,\lambda}^{\frac{\alpha-\beta}{2}}.
$$
This together with
  $\|(J_{\phi,\psi}P_{D_m}J_{\phi,\psi}^T+\lambda I)^{-1}\|_{\psi\rightarrow\psi} 
\leq\lambda^{-b}$
yields
\begin{eqnarray*}
  &&   \|(J_{\phi,\psi}P_{D_m}J_{\phi,\psi}^T+\lambda I)^{-b} 
    ((J_{\phi,\psi}J^T_{\phi,\psi})^{\frac{\alpha-\beta}{2}}-(J_{\phi,\psi}P_{D_m}J^T_{\phi,\psi})^{\frac{\alpha-\beta}{2}})h^*\|_\psi\\
    &\leq&
    \|(J_{\phi,\psi}P_{D_m}J_{\phi,\psi}^T+\lambda I)^{-b}\|_{\psi\rightarrow\psi}
    \|(J_{\phi,\psi}J^T_{\phi,\psi})^{\frac{\alpha-\beta}{2}}-(J_{\phi,\psi}P_{D_m}J^T_{\phi,\psi})^{\frac{\alpha-\beta}{2}}\|_{\psi\rightarrow\psi}\|h^*\|_\psi\\
    &\leq&
   \lambda^{\frac{\alpha-\beta-2b}{2}} \mathcal Q_{D_m,\lambda}^{\frac{\alpha-\beta}{2}}\|h^*\|_\psi.
\end{eqnarray*}
The proof of Lemma \ref{Lemma:au-3} is completed.
\end{proof}

We are now in a position to prove Proposition \ref{Prop:fitting-via-difference}.

\begin{proof}[Proof of Proposition \ref{Prop:fitting-via-difference}]
To bound $\mathcal A(D,\lambda,m)$, we notice from   Lemma \ref{Lemma:au-2} with $a=1$ that 
\begin{equation}\label{bound.A}
\mathcal A(D,\lambda,m)= \lambda\|(J_{\phi,\psi}P_{D_m}J_{\phi,\psi}^T+\lambda I)^{-1}(J_{\phi,\psi}P_{D_m}J^T_{\phi,\psi})^{\frac{\alpha-\beta}{2}}h^*\|_\psi \leq
  \lambda^{ \frac{\alpha-\beta}{2}}\|h^*\|_\psi.
\end{equation} 
The definition of $\mathcal C(D,\lambda,m)$ and Lemma \ref{Lemma:au-3} with $b=1$ yield
\begin{eqnarray}\label{bound.C}
   \mathcal C(D,\lambda,m)\leq \lambda^{\frac{\alpha-\beta}{2}} \mathcal Q_{D_m,\lambda}^{\frac{\alpha-\beta}{2}}\|h^*\|_\psi.
\end{eqnarray}
To bound $\mathcal S_1(D,\lambda,m)$, we have from (\ref{def:P}) and Lemma \ref{Lemma:au-1} with $F=J_{\phi,\psi}^Tf^*-S_D^Ty_D$ that
\begin{eqnarray}\label{bound.S1}
    \mathcal S_1(D,\lambda,m)\leq \mathcal Q_{D,\lambda}\mathcal P_{D,\lambda}.
\end{eqnarray}
Noting $P_{D_m}^2=P_{D_m}$, for an arbitrary $h\in\mathcal N_\psi$, we obtain
\begin{eqnarray}\label{norm-trans-111}
   &&
   \|P_{D_m}J^T_{\phi,\psi}h\|_\phi^2
   =\langle P_{D_m}J^T_{\phi,\psi}h,P_{D_m}J^T_{\phi,\psi}h\rangle_\phi
   =\langle J_{\phi,\psi} P_{D_m}J^T_{\phi,\psi}h,P_{D_m}J^T_{\phi,\psi}h\rangle_\psi\\
   &=&
    \|J_{\phi,\psi}P_{D_m}J^T_{\phi,\psi}h\|_\psi^2
    \leq
     \|(J_{\phi,\psi}P_{D_m}J^T_{\phi,\psi}+\lambda I)^{1/2}h\|_\psi^2.\nonumber 
\end{eqnarray}
Then, it follows from the definition of $\mathcal S_2(D,\lambda,m)$, (\ref{def:R}),  Lemma \ref{Lemma:au-1} with $F=L_{\phi,D}-\mathcal L_\phi$, and
Lemma \ref{Lemma:au-2} with $a=1/2$ that
\begin{eqnarray}\label{bound.S2}
 &&\mathcal S_2(D,\lambda,m)\leq  
   \|J_{\phi,\psi}(P_{D_m}  L_{\phi,D} P_{D_m}+\lambda I)^{-1}
         P_{D_m}(L_{\phi,D}-\mathcal L_\phi)\|_{\phi\rightarrow\psi}\\
         &\times& \|P_{D_m}J_{\phi,\psi}^T
         (J_{\phi,\psi}P_{D_m}J_{\phi,\psi}^T+\lambda I)^{-1} 
          (J_{\phi,\psi}P_{D_m}J^T_{\phi,\psi})^{\frac{\alpha-\beta}{2}}h^*\|_\phi   \nonumber\\
          &\leq&
          \mathcal Q_{D,\lambda}\mathcal R_{D,\lambda}
          \|
         (J_{\phi,\psi}P_{D_m}J_{\phi,\psi}^T+\lambda I)^{-1/2} 
          (J_{\phi,\psi}P_{D_m}J^T_{\phi,\psi})^{\frac{\alpha-\beta}{2}}\|_{\phi\rightarrow\phi}\|h^*\|_\psi
                                \nonumber\\
        &\leq&
         \mathcal Q_{D,\lambda}\mathcal R_{D,\lambda} \|h^*\|_\psi \lambda^{ \frac{\alpha-\beta-1}{2}}. \nonumber
\end{eqnarray}
To derive an upper bound of $\mathcal S_3(D,\lambda,m)$, we have from Lemma \ref{Lemma:au-1} with $F=L_{\phi,D}-\mathcal L_\phi$, Lemma \ref{Lemma:au-3} with $b=1/2$, (\ref{norm-trans-111}), and (\ref{def:R}) that
\begin{eqnarray}\label{bound.S3}
  &&\mathcal S_3(D,\lambda,m)  
  \leq 
   \|J_{\phi,\psi}(P_{D_m}  L_{\phi,D} P_{D_m}+\lambda I)^{-1}
         P_{D_m}(L_{\phi,D}-\mathcal L_\phi)\|_{\phi\rightarrow\psi} \\
         &\times&\|   (J_{\phi,\psi}P_{D_m}J_{\phi,\psi}^T+\lambda I)^{-1/2} 
          \left((J_{\phi,\psi}J^T_{\phi,\psi})^{\frac{\alpha-\beta}{2}}-(J_{\phi,\psi}P_{D_m}J^T_{\phi,\psi})^{\frac{\alpha-\beta}{2}}\right)h^*\|_\psi  \nonumber\\ 
          &\leq&
         \mathcal Q_{D,\lambda}\mathcal R_{D,\lambda}  \lambda^{\frac{\alpha-\beta-1}{2}} \mathcal Q_{D_m,\lambda}^{\frac{\alpha-\beta}{2}}\|h^*\|_\psi.\nonumber
\end{eqnarray}
Plugging (\ref{bound.A}), (\ref{bound.C}), (\ref{bound.S1}), (\ref{bound.S2}), and (\ref{bound.S3}) into (\ref{Error decomposition}), we obtain from $\|h^*\|_\psi=\|f^*\|_\varphi$ that
\begin{eqnarray*} 
   \|f_{D,s^*,\lambda}-f^*\|_\psi
   \leq
  (1 +\mathcal Q_{D,\lambda}\mathcal R_{D,\lambda}\lambda^{-\frac12}) \lambda^{ \frac{\alpha-\beta}{2}}(1+\mathcal Q_{D_m,\lambda}^{\frac{\alpha-\beta}{2}})\|f^*\|_\varphi
  +\mathcal Q_{D,\lambda}\mathcal P_{D,\lambda}. 
\end{eqnarray*}
This completes the proof of Proposition \ref{Prop:fitting-via-difference}.
\end{proof}

\subsection{Proof  of Theorem \ref{Theorem:native-out}}
 Due to Proposition \ref{Prop:fitting-via-difference},
it suffices to bound $\mathcal Q_{D,\lambda}$, $\mathcal P_{D,\lambda}$, and $\mathcal Q_{D,\lambda}^*$, which were derived in our recent paper \cite{Feng2021radial}.

\begin{lemma}\label{Lemma:operator-diff}
Let $0<\delta<1$,  $0\leq v< (2\gamma-d)/2\gamma$, and $D$ satisfy  (\ref{Model1:fixed}). If $\mathcal T_t$ and $D_m=\mathcal T_{s^*}$ are spherical  designs  and $\hat \phi_k\sim k^{-2\gamma}$  with $\gamma> d/2$, then
\begin{eqnarray}
 && \quad \mathcal P_{D,\lambda}
    \leq
    \tilde{c}'\lambda^{-\frac{d}{4\gamma}} |D|^{-1/2}\log\frac3\delta, \label{bound-P}\\
 && \quad  \mathcal R_{D,\lambda}
  \leq
    \tilde{c}\lambda^{-1/2}t^{-\gamma}, \label{bound-R}\\
 && \quad   \mathcal Q_{D,\lambda}
   \leq
   \left\{\begin{array}{cc}
    \tilde{c}^2  \lambda^{-2+2v} t^{-2(1-v)\gamma}+\tilde{c}\lambda^{-1+v}t^{-(1-v)\gamma}+1,& \mbox{if}\ v\leq 1/2, \\
    \tilde{c}^2\lambda^{-2}t^{-2\gamma}+\tilde{c} \lambda^{-1+v}t^{-(1-v)\gamma}+1,& \mbox{if}\ v>1/2,
    \end{array}\right. \label{bound-Q}\\
  && \quad   \mathcal Q_{D_m,\lambda}
   \leq
   \left\{\begin{array}{cc}
    \tilde{c}^2  \lambda^{-2+2v} (s^*)^{-2(1-v)\gamma}+\tilde{c}\lambda^{-1+v}(s^*)^{-(1-v)\gamma}+1,& \mbox{if}\ v\leq 1/2, \\
    \tilde{c}^2\lambda^{-2}(s^*)^{-2\gamma}+\tilde{c} \lambda^{-1+v}(s^*)^{-(1-v)\gamma}+1,& \mbox{if}\ v>1/2,
    \end{array} \right. \label{Bound-Q-small}
\end{eqnarray}
where $\tilde{c}$ and $\tilde{c}'$ are constants depending only on $d$ and $M$.
\end{lemma}

We then prove Theorem \ref{Theorem:native-out} via using Lemma \ref{Lemma:operator-diff} and Proposition \ref{Prop:fitting-via-difference}.

\begin{proof}[Proof of Theorem \ref{Theorem:native-out}]
According to (\ref{bound-Q}) and (\ref{bound-R}) with $v=0$, we have from $t\geq s^*\geq \lambda^{-1/\gamma}$ that
$$
  \mathcal Q_{D,\lambda}\mathcal R_{D,\lambda}\lambda^{-\frac12}
  \leq \tilde{c}(\tilde{c}^2+\tilde{c}+1).
$$
Moreover, (\ref{Bound-Q-small}) with $v=0$ together with $s^*\geq \lambda^{-1/\gamma}$ implies
$$
     \mathcal Q_{D_m,\lambda}\leq   (\tilde{c}^2+\tilde{c}+1).
$$
Furthermore, (\ref{bound-Q}) and (\ref{bound-P}) with $v=0$ yield that with confidence at least $1-\delta$, there holds
$$
 \mathcal Q_{D,\lambda}\mathcal P_{D,\lambda}
 \leq
 \tilde{c}' (\tilde{c}^2+\tilde{c}+1)\lambda^{-\frac{d}{4\gamma}} |D|^{-1/2}\log\frac3\delta.
$$
Plugging  the above three estimates into Proposition \ref{Prop:fitting-via-difference}, we know  
\begin{eqnarray*}
      &&\|f_{D,s^*,\lambda}-f^*\|_\psi
        \leq
         (1+ \tilde{c}(\tilde{c}^2+\tilde{c}+1))(1+ (\tilde{c}^2+\tilde{c}+1)^{\frac{\alpha-\beta}{2}})\|f^*\|_\varphi \lambda^{\frac{\alpha-\beta}{2}}\\
         &+&
         \tilde{c}' (\tilde{c}^2+\tilde{c}+1)\lambda^{-\frac{d}{4\gamma}} |D|^{-1/2}\log\frac3\delta 
\end{eqnarray*}
holds with confidence $1-\delta$.
Noting further $\lambda=\bar{c}|D|^{-\frac{2\gamma}{2\gamma(\alpha-\beta)+d}}$ for some absolute constant $\bar{c}>0$, we have
$$
   \|f_{D,s^*,\lambda}-f^*\|_\psi \leq C|D|^{-\frac{\gamma(\alpha-\beta)}{2\gamma(\alpha-\beta)+d}}\log\frac3\delta,
$$
where 
$$
    C:= (1+ \tilde{c}(\tilde{c}^2+\tilde{c}+1))(\bar{c}+ \bar{c}(\tilde{c}^2+\tilde{c}+1)^{\frac{\alpha-\beta}{2}})\|h^*\|_\psi
    +\tilde{c}' (\tilde{c}^2+\tilde{c}+1).
$$
The only thing left is to present the range of $s^*$. Since 
$s^*\geq \lambda^{-1/\gamma}$ and $\lambda=\bar{c}|D|^{-\frac{2\gamma}{2\gamma(\alpha-\beta)+d}}$, 
we have
$$
    s^*\geq (\bar{c})^{-1/\gamma} |D|^{\frac{2}{2\gamma(\alpha-\beta)+d}}.
$$
 This completes the proof of Theorem \ref{Theorem:native-out}.
\end{proof}

\subsection{Construction of hard-to-approximate target functions for lower bounds}
To derive the lower bound, a crucial procedure is to construct  hard-to-approximate target functions  in  $\mathcal N_\varphi$ that cannot be approximated well.  Denote $d^*_n=\sum_{k=n+1}^{2n}Z(d,k)$. 
Given $\tau>0$, define a spherical polynomial  by
\begin{equation}\label{f-bad}
     f_{\omega,\tau,n}:=\left(\frac{\tau}{d^*_n}\right)^{1/2}\sum_{k=n+1}^{2n}\hat{\phi}_{k}^{\beta/2}\sum_{\ell=1}^{Z(d,k)}
     \omega_{k,j}Y_{k,\ell},
\end{equation}
where $\omega=(\omega_{n+1,1},\omega_{n+1,2}\dots,\omega_{2n,Z(d,2n)})\in\{0,1\}^{d_n^*}$ is a binary string of size $d_n^*$.
Then, for any  $n$ satisfying
\begin{equation}\label{Rest.on.n}
    (2\hat{c})^{\alpha-\beta} \tau n^{2\gamma(\alpha-\beta)}\leq U^2,
\end{equation}
  $\hat{c}^{-1}k^{-2\gamma}\leq \hat{\phi}_k\leq\hat{c} k^{-2\gamma}$ together with (\ref{kernel-relation-1}) yields
\begin{eqnarray*} 
    \|f_{\varepsilon,\tau,n}\|^2_\varphi 
    &\leq&
    \frac{\tau}{d^*_n}  \sum_{k=n+1}^{2n}Z(d,k)
    \hat{\phi}_{k}^{\beta-\alpha} \\
    &\leq& \frac{\tau}{d^*_n}  \sum_{k=n+1}^{2n}Z(d,k) \hat{c}^{\alpha-\beta}k^{-2\gamma(\beta-\alpha)}\leq (2\hat{c})^{\alpha-\beta} \tau n^{2\gamma(\alpha-\beta)}
    \leq U^2.    \nonumber
\end{eqnarray*}

This implies $f_{\omega,\tau,n}\in \mathcal N_{\varphi}$ and $\|f_{\omega,\tau,n}\|_\varphi\leq U$.
 For another string  $\omega'=(\omega'_{n+1,1},\dots,$ $\omega'_{2n,Z(d,2n)})\in\{0,1\}^{d_n^*}$,  (\ref{kernel-relation-1}) and
(\ref{kernel-relation}) also yield
\begin{eqnarray}\label{norm-bound-2}
    \|f_{\omega,\tau,n}-f_{\omega',\tau,n}\|_{L^2(\mathbb S^d)}^2
   =
   \frac{\tau}{d^*_n}\sum_{k=n+1}^{2n}\hat{\phi}_{k}^{\beta}\sum_{\ell=1}^{Z(d,k)} (\omega_{k,\ell}-\omega_{k,\ell}')^2 
    \leq 
     \tau \hat{c}^{\beta}n^{-2\gamma\beta}.
\end{eqnarray}

We then aim to derive lower bound of $\|f_{\omega,\tau,n}-f_{\omega',\tau,n}\|_{\psi}$, for which the following Gilbert-Varahamov bound provided in \cite[Lemma 2.9]{tsybakov2009introduction} (see also \cite[Lemma 24]{fischer2020sobolev}) is needed.

\begin{lemma}\label{Lemma:Gilbert-varshamov}
For $S\geq 8$, there exist some $L\geq2^{S/8}$ and some binary strings $\omega^{(0)},\omega^{(1)},\dots,\omega^{(L)}\in\{0,1\}^S$ such that $\omega^{(0)}=0$ and
$$
       \sum_{\jmath=1}^S(\omega^{(\imath)}_\jmath-\omega^{(\imath')}_\jmath)^2\geq S/8
$$
for all $\imath\neq \imath'$, where $\omega^{(\imath)} 
=(\omega^{(\imath)}_1,\dots,\omega^{(\imath)}_S)$.
\end{lemma}

If $S:=d_{n}^*\geq 8$, let
\begin{equation}\label{numeric-set}
    \mathcal E_{n,L}:=\{\omega^{*,\imath}\}_{\imath=1}^L
\end{equation}
be the set of binary strings satisfying the conditions of  Lemma \ref{Lemma:Gilbert-varshamov}. Then  it follows from Lemma \ref{Lemma:Gilbert-varshamov} that for any $\omega^{*,\imath},\omega^{*,\imath'}\in \mathcal E_{n,L}$, there holds
\begin{eqnarray}\label{psi-norm-111}
    \|f_{\omega^{*,\imath},\tau,n}-f_{\omega^{*,\imath'},\tau,n}\|_{\psi}^2
    =
    \frac{\tau}{d^*_n}\sum_{k=n+1}^{2n}\sum_{\ell=1}^{Z(d,k)} 
    (\omega_{k,\ell}-\omega_{k,\ell}')^2
   \geq  \tau/8.
\end{eqnarray}

\subsection{Proof of Theorem \ref{Theorem:lower-bound}}
Set    $f^*=f_{\omega^{*,(\imath)},\tau,n}$ for some $\tau$ and $n$ satisfying (\ref{Rest.on.n}) and  $\omega^{*,(\imath)}\in\mathcal E_{n,L}$ with $1\leq\imath\leq L$. Let the noise $\{\varepsilon_i\}$  be drawn i.i.d. from the normal distribution $\mathcal N(0,M^2)$. We
obtain a probability measure 
$$
      { \rho}_{\omega^{*,(\imath)},\tau,n}= { \rho}_{\mathcal T_t}(dx) \times \mathcal N(f_{\omega^{*,(\imath)},\tau,n},M^2),
$$
where    $\mathcal T_t=\{x_i\}_{i=1}^{|D|}$ is the spherical $t$-design and  ${ \rho}_{\mathcal T_s}(dx)$ denotes the deterministic and discrete  distribution  on the sphere satisfying
$$
    \int_{\mathbb S^d}f(x) { \rho}_{\mathcal T_s}(dx)=\frac1{|D|}\sum_{x_i\in\mathcal T_t}  f(x_i). 
$$
It is easy to check that  the data $D$
  generated by ${ \rho}_{\omega^{*,(\imath)},\tau,n}$ satisfies 
(\ref{Model1:fixed}). 

For any $f_D$ derived from $D$, define
\begin{equation}\label{number-minimal-error}
    \Phi(D):={\arg\min}_{\imath=0,1\dots,L}\|f_D-f_{\omega^{*,\imath},\tau,n}\|_\beta.
\end{equation}
Then for any $\imath\in\{0,\dots,L\}$ with $\imath\neq \Phi(D)$, we have from \eqref{psi-norm-111} that
\begin{eqnarray}\label{important-111}
   &&\sqrt{\tau/8}
   \leq 
    \|f_{\omega^{*,\Phi(D)},\tau,n}-f_{\omega^{*,\imath},\tau,n}\|_{\psi}\\
    &\leq&
    \|f_{\omega^{*,\Phi(D)},\tau,n}-f_D\|_\psi+\|f_D-f_{\omega^{*,\imath},\tau,n}\|_{\psi} 
    \leq 2\|f_D-f_{\omega^{*,\imath},\tau,n}\|_{\psi}. \nonumber
\end{eqnarray}
Therefore, 
\begin{equation}\label{important-222}
    { \mathbf P_{\rho^{|D|}_{\omega^{*,(\imath)}},\tau,n}}\left[D:\|f_D-f_{\omega^{*,\imath},\tau,n}\|^2_{\psi}\geq \tau/32\right]
    \geq
    { \bf \mathbf P_{\rho_{\omega^{*,(\imath)},\tau,n}^{|D|}}}\left[D:\Phi(D)\neq\imath \right],
\end{equation}
where ${ \rho^{|D|}=\overbrace{\rho\times \dots\times \rho}^{|D|}}$ and $\imath=0,1,\dots,L$.

\begin{definition}
Let ${ \rho}_1$ and ${ \rho}_2$ be two probability measures on some common measurable 
space $(\Omega,\mathcal A)$ satisfying that ${ \rho}_1$ is absolutely continuous with respect to ${ \rho}_2$, then the Kullback-Leibler divergence for probability measures ${ \rho}_1$ and ${ \rho}_2$ is defined as
$$
    \mathcal K({ \rho}_1,{ \rho}_2):=\int_{\Omega}\log\left(\frac{d{ \rho}_1}{d{ \rho}_2}\right)d{ \rho}_1.
$$
\end{definition}

To lower bound  $ { \mathbf P_{\rho_{\omega^{*,(\imath)},\tau,n}^{|D|}}}\left[D:\Phi(D)\neq\imath \right]$, we  recall the following lemma derived in \cite[Lemma 20]{fischer2020sobolev}.

\begin{lemma}\label{Lemma:KL-relation}
Let $L\geq 2$ and  $(\Omega,\mathcal A)$ be a measurable space. 
Let $0<c_*<\infty $,
${ \rho}_0,{ \rho}_1,\dots,{ \rho}_L$ be probability measures on $(\Omega,\mathcal A)$, and ${ \rho}_\imath$ be absolutely continuous with respect to ${ \rho}_0$ for $\imath=1,2,\dots,L$ satisfying 
$$
      \frac1L\sum_{\imath=1}^L\mathcal K({ \rho}_\imath,{ \rho}_0)\leq c_*.
$$
Then for all measurable functions $\Phi:\Omega\rightarrow\{0,1,\dots,L\}$, there holds
$$
   \max_{\imath=0,1,\dots,L} { \mathbf P_{\rho_\imath}} [\mathcal D\in \Omega:\Phi(\mathcal D)\neq \imath ]\geq 
   \frac{\sqrt{L}}{1+\sqrt{L}}\left(1-\frac{3c_*}{\log L}-\frac{1}{2\log L}\right).
$$
\end{lemma}

Based on the above foundations, we can prove Theorem \ref{Theorem:lower-bound} as follows.

\begin{proof}[Proof of Theorem \ref{Theorem:lower-bound}]
Let ${ \rho}_1={ \rho}_{\omega^{*,(\imath)},\tau,n}$ and ${ \rho}_2={ \rho}_{\omega^{*,(0)},\tau,n}$ be the same as in Lemma \ref{Lemma:KL-relation}. Due to the definition of the Kullback-Leibler divergence, it is easy to derive \cite[P.998]{blanchard2018optimal}
$$
     \mathcal K({ \rho}^{|D|}_1,{ \rho}_2^{|D|})=|D| \mathcal K({ \rho}_1,{ \rho}_2),
$$
and 
$$
    \mathcal K(\mathcal N(f_{\omega^{*,(\imath)},\tau,n},M^2),\mathcal N(f_{\omega^{*,(\imath)},\tau,n},M^2))
    =\frac{(f_{\omega^{*,(\imath)},\tau,n}-f_{\omega^{*,(\imath)},\tau,n})^2}{2M^2}.
$$
Since $f_{\omega,\tau,n}$ defined by (\ref{f-bad}) is a spherical polynomial of degree at most $2n$, then
  for 
\begin{equation}\label{Restriction-on-n-2}
       n\leq t/2, 
\end{equation}
we have from (\ref{norm-bound-2}) and the definition of the spherical $t$-design 
that
\begin{eqnarray*}
  && \frac{1}{L}\sum_{\imath=1}^L\mathcal K({ \rho}_{\omega^{*,(\imath)},\tau,n}^{|D|},{ \rho}_{\omega^{*,(0)},\tau,n}^{|D|})
    =  
   \frac{|D|}{M^2}\frac{1}{L}\sum_{\imath=1}^L
   \frac{1}{|D|}\sum_{i=1}^{|D|}(f_{\omega^{*,(\imath)},\tau,n}(x_i)-f_{\omega^{*,(0)},\tau,n}(x_i))^2\\
   &=&
  \frac{|D|}{M^2}\frac{1}{L}\sum_{\imath=1}^L\|f_{\omega^{*,(\imath)},\tau,n}-f_{\omega^{*,(0)},\tau,n}\|_{L^2(\mathbb S^d)} 
   \leq
   \frac{\hat{c}^{\beta}|D|\tau}{M^2n^{ 2\gamma\beta}}=:c_*. 
\end{eqnarray*}
Setting $\mathcal D=D$ in Lemma \ref{Lemma:KL-relation}, we   get
$$
    \max_{\imath=0,1,\dots,L}{\mathbf P_{\rho_{\omega^{*,(\imath)},\tau,n}}}\left[D:\Phi(D)\neq\imath\right]\geq \frac{\sqrt{L}}{1+\sqrt{L}}\left(1-\frac{  3\hat{c}^{\beta}|D|\tau    }{M^2n^{2\gamma\beta}\log L}-\frac{1}{2\log L}\right),
$$
where 
 $L\geq 2^{d_n^*/8}$. 
This together with (\ref{important-222}) and $d_n^*\leq Z(d+1,2n) \leq 2^dn^d$ yields that there is a probability { distribution $\rho^*$} and a   function $f^*_{bad}$ satisfying (\ref{Model1:fixed}) such that 
\begin{eqnarray*}
    { \mathbf P_{\rho^*}}\left[D:\|f_D-f^*_{bad}\|^2_{\psi}\geq \tau/32\right]
    \geq
    \frac{1}{2}\left(1-\frac{  12\hat{c}^{\beta}|D|\tau    }{2^dM^2n^{2\gamma\beta+d} }-\frac{4}{2^dn^d}\right).
\end{eqnarray*}
The only thing left is to select suitable  $\tau$ and $n$
to satisfy (\ref{Rest.on.n}) and (\ref{Restriction-on-n-2}).

Since $\hat{c}_1^{-1} |D|^{1/d}\leq s\leq \hat{c}_1 |D|^{1/d}$ for some $\hat{c}_1\geq 1$, 
if $\tau=C_1''|D|^{-\frac{2\gamma(\alpha-\beta)}{2\gamma(\alpha-\beta)+d}}$ and $n=C_1'|D|^{\frac{1}{2\gamma(\alpha-\beta)+d}}$
with $C_1'=\min\left\{\frac{2}{\gamma(\alpha-\beta)}(2\hat{c})^{-\frac1\gamma}(U^2C_1'')^{-\frac1{\gamma(\alpha-\beta)}},\frac{\hat{c}_1}{2}\right\}$ and $C_1''$ small enough such that 
$$
      \max\left\{\frac{12\hat{c}^\beta C_1''}{2^dM^2(C_1')^{2\gamma\beta+d}},\frac{4}{2^d(C_1')^d}\right\}\leq\frac14,
$$
then (\ref{Rest.on.n}) and (\ref{Restriction-on-n-2}) hold,  and 
$$
  { \mathbf P_{\rho^*}}\left[D:\|f_D-f^*_{bad}\|_{\psi}\geq C_1|D|^{-\frac{\gamma(\alpha-\beta)}{2\gamma(\alpha-\beta)+d}} \right]
 \geq \frac12\left(1-\frac12|D|^{-\frac{\min\{2\gamma\beta,d\}}{2\gamma(\alpha+\beta)+d}}\right)\geq \frac14,
$$ 
where $C_1=(C_1''/32)^{1/2}$. This completes the proof of Theorem \ref{Theorem:lower-bound}.
\end{proof}

\bibliographystyle{siamplain}
\bibliography{Nysphere}

\end{document}